\documentclass{article}
\usepackage{graphicx} %
\usepackage{enumerate}
\usepackage{amsmath}
\usepackage{booktabs}
\usepackage{amssymb}
\usepackage{amsfonts}
\usepackage{amsthm}
\usepackage{pkg}
\usepackage{xcolor}
\usepackage{mathtools}
\usepackage{subcaption}
\usepackage[sort,numbers]{natbib}
\usepackage[margin=1in]{geometry}

\newtheorem{Theorem}{Theorem}[section]

\newtheorem{Assumption}[Theorem]{Assumption}
\newtheorem{Lemma}[Theorem]{Lemma}

\newtheorem{Remark}[Theorem]{Remark}

\newcommand{\R}{\mathbb{R}}

\usepackage[colorlinks,
            linkcolor=red,
            anchorcolor=blue,
            citecolor=magenta,
            ]{hyperref}
\usepackage{cleveref}

\usepackage{algorithm,algorithmic}

\title{Stochastic Control for Fine-tuning Diffusion Models:\\ Optimality, Regularity, and Convergence}
\author{Yinbin Han, Meisam Razaviyayn, Renyuan Xu 
}
\author{Yinbin Han
\thanks{Department of Management Science and Engineering, School of Engineering, Stanford University. \textbf{Email:} \{yinbinha,renyuanxu\}@stanford.edu}
\and
Meisam Razaviyayn \thanks{Daniel J. Epstein Department of Industrial and Systems Engineering, University of Southern California. \textbf{Email:} razaviya@usc.edu}
\and
Renyuan Xu 
  \footnotemark[1]
}
\date{First version:~Dec 2024; this version:~Aug 2025}

\begin{document}
\maketitle
\allowdisplaybreaks
\begin{abstract}
    Diffusion models have emerged as powerful tools for generative modeling, demonstrating exceptional capability in capturing target data distributions from large datasets. However, fine-tuning these massive models for specific downstream tasks, constraints, and human preferences remains a critical challenge. While recent advances have leveraged reinforcement learning algorithms to tackle this problem, much of the progress has been empirical, with limited theoretical understanding. To bridge this gap, we propose a stochastic control framework for fine-tuning diffusion models. Building on denoising diffusion probabilistic models as the pre-trained reference dynamics, our approach integrates linear dynamics control with Kullback–Leibler regularization. We establish the well-posedness and regularity of the stochastic control problem and develop a {policy iteration algorithm (PI-FT)} for numerical solution. We show that PI-FT achieves global convergence at a linear rate. Unlike existing work that assumes regularities throughout training, we prove that the control and value sequences generated by the algorithm maintain the regularity. Additionally, we explore extensions of our framework to parametric settings and continuous-time formulations and demonstrate the
practical effectiveness of the proposed PI-FT algorithm through numerical experiments.  Our code is available at \url{https://github.com/yinbinhan/fine-tuning-of-diffusion-models}.
\end{abstract}

\section{Introduction}

The increasing availability of large-scale datasets, combined with advancements in high-performance computing infrastructures, have catalyzed the rise of data-driven approaches across diverse scientific and engineering domains. Traditional data-driven approaches, which involve collecting and fitting models to offline or online data, are highly dependent on the quality and availability of the data. Consequently, data-related limitations can significantly hinder the performance of these methods \cite{doersch2016tutorial,durgadevi2021generative,fetaya2019understanding}. For example, constraints in experimental design or collection processes often result in data scarcity, which can significantly hinder model performance. In addition, obtaining sufficient data to ensure experimental validity can be costly and time-intensive, presenting practical barriers to scalability. Finally and more importantly, data gathered under specific conditions often fails to generalize effectively to other environments or downstream applications, limiting the versatility of the resulting models \cite{bansal2022systematic,gangwal2024current,alzubaidi2023survey}.

Generative modeling offers a flexible solution by creating synthetic data that maintains the properties of collected data while enhancing data diversity. Diffusion models, such as those proposed by \cite{sohl2015deep,song2019generative,song2020score,Ho20}, have emerged as powerful tools in this area, supporting notable advancements such as DALL·E \cite{ramesh2022hierarchical,betker2023improving}, Stable Diffusion \cite{rombach2022high}, and Sora \cite{openai2024sora}. These models excel by learning the {\it score function} from potentially high-dimensional and limited data, extracting critical information for the data generation process. To achieve cost-effective, task-specific data generation and multi-modal integration, techniques like alignment \cite{wallace2024diffusion}, including guidance \cite{dhariwal2021diffusion,ho2022classifier} and fine-tuning \cite{fan2023optimizing,black2023training,fan2024reinforcement}, of pre-trained models are proposed. While diffusion alignment has achieved significant empirical success, the theoretical foundations remain in the early stages---a gap this paper seeks to address.

Aligning a pre-trained diffusion model to a specific task or dataset requires updating its parameters through additional training, typically using a smaller, task-specific dataset \cite{wallace2024diffusion,podell2023sdxl,rombach2022high,dai2023emu}. Among the prevalent alignment techniques, fine-tuning and guidance differ in how the additional criteria are handled—either as a soft constraint or a hard constraint. The soft constraint approach is particularly effective for applications that incorporate human preference or embed human values as a reward signal evaluated on the task-specific dataset \cite{black2023training,fan2024reinforcement,uehara2024bridging,zhao2024adding}, which is the central focus of our paper.

\subsection*{Our work and contributions.}
We introduce a discrete-time stochastic control formulation with linear dynamics and Kullback–Leibler (KL) regularization for fine-tuning diffusion models. Specifically, we establish a novel connection to denoising diffusion probabilistic models (DDPMs) \cite{Ho20} by treating the fine-tuned score as a {\it control to be learned}. In particular, the soft constraint or human preference is modeled by a reward signal evaluated on the generated data output and the KL regularization term penalizes the deviation of the control from the pre-trained score function. We utilize the discrete-time formulation since fine-tuning is {\it inherently} a {\it discrete-time} problem, as it relies on a pre-trained model which is typically implemented in discrete time, with DDPM being the most commonly-used example. By appropriately selecting the regularization parameter, we demonstrate the well-posedness of the control problem and analyze the regularity of the optimal value function. A key insight arises from the concavity of the nested one-step optimization problem in the Bellman equation, a direct result of KL regularization under a properly chosen regularization coefficient. Leveraging this property, we develop a policy iteration algorithm (PI-FT) with guaranteed convergence to the globally optimal solution at a linear rate. A central challenge involves preserving the regularity of the sequence of value functions and output controls generated during algorithm iterations. This is achieved through a novel coupled induction argument and precise estimations of the regularization parameters. Finally, we discuss the algorithm design in the parametric setting and the connection to continuous-time stochastic control problems,  and demonstrate the practical
efficiency and effectiveness of the proposed PI-FT method.

To the best of our knowledge, this is the first work on fine-tuning of diffusion models that presents a {\it convergence guarantee}.
While reinforcement learning (RL) and control-based fine-tuning frameworks have been explored in prior literature, analyzing algorithm convergence remains challenging due to the continuous nature of the state-action space. 
Our approach addresses this challenge by leveraging linear dynamics and the (one-step) concavity introduced by the KL regularization term, which inherently captures the essence of fine-tuning problems. Specifically, we establish the universal regularity of both the optimal value function and the sequence of value functions generated during iterative update. Moreover, the foundational techniques underlying our results extend naturally to parametric algorithm design and continuous-time formulations, offering insights beyond the specific context of diffusion model fine-tuning.

\subsection*{Related literature.} Our work is related to several emerging research directions.

\paragraph{Fine-tuning of diffusion models.}
Our framework is closely related to the recent studies on fine-tuning diffusion models. Motivated by advancements in the (discrete-time) RL literature, \citet{fan2023optimizing} was the first to introduce a reward-based approach for improving pre-trained diffusion models. Building on this idea, two concurrent works \cite{black2023training,fan2024reinforcement} proposed a Markov decision process (MDP) formulation for denoising diffusion processes. To mitigate reward over-optimization \cite{gao2023scaling}, \citet{fan2024reinforcement} examined the impact of incorporating KL regularization as an implicit reward signal. {Very recently, \citet{uehara2024feedback} introduced an online fine-tuning framework where the regret is upper bounded by the accuracy of a statistical error oracle for reward estimation.} For a comprehensive review of this topic, we direct interested readers to \cite{uehara2024understanding}. Beyond RL-based fine-tuning methods, alternative approaches include classifier guidance \cite{dhariwal2021diffusion} and classifier-free guidance \cite{ho2022classifier}, supervised fine-tuning \cite{lee2023aligning}, LoRA and its variants \cite{hu2021lora,ryu2023low}, and  DreamBooth \cite{ruiz2023dreambooth} among others.

All the aforementioned references focused solely on empirical investigations, except for \cite{uehara2024feedback}, which leaves theoretical foundations largely unexplored. While the discrete-time MDP formulation discussed above aligns well with the RL literature, it overlooks the structural properties inherent to diffusion models. In contrast, our approach specifically leverages DDPM as the pre-trained model, thereby fully utilizing the underlying structure of diffusion models.

\paragraph{Continuous-time control/RL for fine-tuning.}
Inspired by the continuous-time nature of diffusion processes, recent work has explored the application of continuous-time control/RL for fine-tuning. For instance, \citet{berner2022optimal} established a connection between fine-tuning and stochastic control by analyzing the Hamilton-Jacobi-Bellman (HJB) equation for the log density. \citet{domingo2024adjoint} introduced an adjoint matching method for the fine-tuning of rectified flow, drawing inspiration from Pontryagin's maximum principle.
A more rigorous treatment of fine-tuning within the framework of entropy-regularized control problems is developed in \cite{tang2024fine}, addressing the well-posedness and regularity of the corresponding HJB equation and extending the analysis to general $f$-divergences, though no specific algorithm was proposed. Following this entropy-regularization perspective, \citet{zhao2024scores} derived the formulas for  continuous-time policy gradient method and continuous-time version of proximal policy optimization (PPO). Furthermore, \citet{gao2024reward} applied a continuous-time q-learning algorithm to simulate data that reflects human preferences directly, without relying on pre-trained models.

Compared with the literature that technically remain at the level of exploring HJB equations, we establish convergence rate for the proposed algorithm in discrete time. Despite the continuous-time nature of diffusion processes, the discrete-time setting is more suitable because fine-tuning builds upon a pre-trained model, which is naturally implemented and provided in discrete time.

\paragraph{RL theory.}
Our theoretical developments are also closely connected to the literature of RL and policy gradient methods in discrete time. For MDPs with {\it finite} state and action spaces, recent advancements providing global convergence guarantees for policy gradient methods and their variants can be found in \cite{bhandari2021linear, bhandari2024global,berner2022optimal,Cen2021FastGC,Ding2020NaturalPG,Fu2021SingleTimescaleAP,Liu2019NeuralPR,Liu2020AnIA,Wang2020NeuralPG,xiao2022convergence,Xu2021DoublyRO,zhang2020variational,zhang2021provably,fatkhullin2023stochastic,zhan2023policy,mondal2024improved}. Beyond MDPs, policy gradient methods have also been applied to Linear Quadratic Regulators (LQRs), a specific class of control problems characterized by linear dynamics and quadratic cost functions \cite{bu2019lqr,fazel2018global,hambly2021policy,malik2019derivative,mohammadi2019global,szpruch2021exploration,szpruch2022optimal,wang2020reinforcement,han2023policy,guo2023fast,zhou2023single}. While the convergence analysis of policy gradient methods is well-established in the above-mentioned settings, the study on control problems with continuous state-action space and general cost functions has been limited. Our work ventures into this broader class of control problems, providing convergence guarantees for policy gradient algorithms under  more general settings.

\subsection*{Notations and Organization.} Take $d\in \mathbb{N}_{+}$, the set of all positive integers. We denote $(\Omega, \mcal{F}\coloneqq\{\mathcal{F}_t\}_{t \ge 0}, \mathbb{P})$ as a usual filtered probability space supporting a random variable $R\in L^1(\mathbb{R}^d,\mathcal{F}_0)$, namely, $R(y)\in \mathcal{F}_0$ has finite mean for all $y\in \mathbb{R}^d$. Denote $\{W_t\}_{t =0}^{T-1}$ a sequence of independent standard $d$-dimensional Gaussian random vectors. We denote by $\mathbb{F}^W$ the natural filtration of $\{W_t\}_{t =0}^{T-1}$. In addition, $\mcal{N}(\mu, \Sigma)$ denotes the normal distribution with mean $\mu \in \R^d$ and covariance $\Sigma \in \R^{d\times d}$. Let $f({}\cdot{}\vert\mu, \Sigma)$ denote the probability density of $\mcal{N}(\mu, \Sigma)$. For two probability distributions $P$ and $Q$ such that $P \ll Q$ with densities $p$ and $q$ supported on $\R^d$, we denote the KL divergence by ${\rm KL}(p \| q) \coloneqq \int_{\R^d}p(y)\log\frac{p(y)}{q(y)}{\rm d}y$. Finally, $I_d \in \R^{d\times d} $ denotes the identity matrix and  $\norm{{}\cdot{}}_2$ denotes the Euclidean norm. 

The paper is organized as follows. Section \ref{sec:set-up} introduces the set-up of the stochastic problem  and establishes the well-posedness and regularity. Section \ref{sec:algorithm} proposes a policy iteration algorithm and develops the linear convergence result. Finally, some discussion on future directions, such as parametric setting and continuous-time formulation, are included in Section \ref{sec:discussion}.

\section{Problem set-up and regularity results}
\label{sec:set-up}

\subsection{Problem set-up}
To address the challenge of fine-tuning, we first introduce the dynamics of denoising diffusion probabilistic models (DDPMs), a widely adopted pre-trained framework in practice. Our focus is on the discrete-time formulation, which aligns with the standard implementation of diffusion models in practice. In addition, DDPMs serve as a foundational scheme for effective fine-tuning.

\paragraph{Denoising diffusion probabilistic models (DDPM).} A well-trained DDPM $\{Y_t^{\rm pre}\}_{t=0}^T$ in discrete time usually follows the following stochastic dynamics with state $Y_t^{\rm pre}\in \mathbb{R}^d$:
\begin{align}
    Y_{t+1}^{\rm pre} & = \frac{1}{\sqrt{\alpha_t}}\big(Y_t^{\rm pre} + (1 - \alpha_t)s^{\rm pre}_t(Y_t^{\rm pre})\big) + \sigma_t W_t,   \;\; 0 \leq t \leq T - 1, \;\; \textit{with}\;\; Y_0^{\rm pre} \sim \mcal{N}(0, I_d),\label{eq:ddpm}
\end{align}
where $\curly{W_t}_{t = 0}^{T - 1}$ are i.i.d.~standard Gaussian random vectors such that $W_t\sim \mathcal{N}(0,I_d)$. The hyper-parameters $\curly{\alpha_t}_{t=0}^{T-1}$ with  $\alpha_t\in (0, 1)$ and $\curly{\sigma_t}_{t=0}^{T-1}$ with $\sigma_t>0$ represent the prescribed denoising rate schedules \cite{Ho20,li2023towards}. They control the variance of noise in data generation\footnote{While DDPM chooses $\sigma_t^2 = 1/\alpha_t - 1$, our analysis in the subsequent sections works for general $\sigma_t$.}.  Here, $s^{\rm pre}_t:\mathbb{R}^d \rightarrow\mathbb{R}^d$ is the score function associated with the pre-trained model. In practice, the score estimator $s_t^{\rm pre}$  is obtained by training a neural network to minimize the score matching loss \cite{hyvarinen2005estimation}. {With a well-trained pre-trained model, we expect the distribution of $Y_T$ to be close to the target distribution, from which the pre-trained model has access to samples and seeks to generate additional samples.} Given $s_t^{\rm pre}$, we denote $ p_{t+1 \vert t}^{\rm pre}(\cdot \vert y_t) $ as the conditional density of $Y_{t+1}^{\rm pre}$ given $Y_t^{\rm pre} = y_t$ induced by the dynamics \eqref{eq:ddpm}, i.e., 
\begin{align*}
    p_{t+1\vert t}^{\rm pre}({}\cdot{}\vert y_{t}) = f\parenthesis{{}\cdot{}\bigg\vert\frac{1}{\sqrt{\alpha_t}}\parenthesis{y_t + (1 - \alpha_t)s^{\rm pre}_t(y_t)}, \sigma_{t}^{2} I_d}.
\end{align*}

\begin{Remark}[Choice of the pre-trained model.]
    DDPMs have become the leading choice for pre-trained models across a wide range of applications, serving as a powerful building block for fine-tuning in diverse tasks \cite{dhariwal2021diffusion,Ho20,watson2023novo,li2023towards}.  By pre-training on large datasets, DDPMs capture complex data distributions, enabling relatively easier fine-tuning for specialized applications such as conditional image synthesis, text-to-image generation, and even time-series forecasting. This versatility has been demonstrated in models such as Stable Diffusion \cite{rombach2022high} and DALL·E \cite{ramesh2022hierarchical}, where fine-tuning on task-specific data enhances performance and customization \cite{fan2024reinforcement}. 
\end{Remark}

\paragraph{Stochastic control formulation.}
When human preference or soft constraint can be modeled by a (stochastic) reward function $R(\cdot)$, we consider the following  stochastic control problem to model fine-tuning tasks of diffusion models:
\begin{eqnarray}
    \sup_{\curly{U_t}_{t = 0}^{T-1}\in \mathcal{U}_0} &  \mathbb{E}\left[R(Y_T) - \sum_{t = 0}^{T-1}\beta_t{\rm KL}\parenthesis{p_{t+1\vert t} ({}\cdot{}|Y_t)\,\Big\|\,p_{t+1\vert t}^{\rm pre}({}\cdot{}|Y_t)}\right], \label{eq:objective-dst} \\ 
    \text{s.t.} & \;\; 
    Y_{t+1} = \frac{1}{\sqrt{\alpha_t}}\parenthesis{Y_t + (1 - \alpha_t)U_t} + \sigma_t W_t,\quad y_{0} \sim \mathcal{N}(0, I_{d}), \label{eq:dynamics-dst}
\end{eqnarray}
   with state dynamics $Y_t\in \mathbb{R}^d$ and control $U_t\in \mathbb{R}^d$. In particular, we work with Markovian policies such that the objective in \eqref{eq:objective-dst} is well-defined. Mathematically, we specify the admissible control set as 
   \begin{align*}
    \mathcal{U}_{t}:=\Big\{\{U_\ell\}_{\ell=t}^{T-1} \,\Big|& U_\ell = {u_\ell(Y_\ell)} \text{ for some measurable function } u_\ell: \mathbb{R}^d\rightarrow \mathbb{R}^d \\
    & \text{ and }\mathbb{E}\Big[R(Y_T) - \sum_{\ell = t}^{T-1}\beta_t{\rm KL}\Big(p_{\ell+1\vert \ell} ({}\cdot{}|Y_\ell)\,\Big\|\,p_{\ell+1\vert \ell}^{\rm pre}({}\cdot{}|Y_\ell)\Big)\Big]<\infty\Big\}, \; 0 \leq t \leq T-1.
   \end{align*}
In terms of the reward, we assume $R\in L^1(\mathbb{R}^d,\mathcal{F}_0)$, namely, $R(y)\in \mathcal{F}_0$ has finite mean for all $y\in \mathbb{R}^d$.
   We assume the reward function is {\it known} in this work.
   The KL-divergence measures the deviation of the fine-tuned model $p_{t+1\vert t}$ from the pre-trained model $p_{t+1\vert t}^{\rm pre}$, where $ p_{t+1 \vert t} $ denotes the conditional distribution of $ Y_{t+1} $ given $ Y_{t} $ induced by the control policy $u_t$. The regularization coefficients $\{\beta_t\}_{t=0}^{T-1}$ control the strength of the regularization term.  

{We have a few remarks in place.}    
\begin{Remark}[Control as score function]
    Our formulation is rooted in both diffusion models and control theory. Comparing \eqref{eq:dynamics-dst} with \eqref{eq:ddpm}, we replace the pre-trained score $s_t^{\rm pre}$ by a control variable $u_t$. Consequently, in the context of diffusion models, the learned control sequence $\{u_t\}_{t=0}^{T-1}$ can be viewed as the {\it new} score function of the fine-tuned model. In other words, solving the control problem \eqref{eq:objective-dst}-\eqref{eq:dynamics-dst} is essentially learning a new score function in response to the reward signal $R(\cdot)$ for fine-tuning. Note that the linear dynamics \eqref{eq:dynamics-dst} with Gaussian noise is preferred in stochastic control due to its tractability in analysis. Specifically, our theoretical analysis in the subsequent sections heavily relies on the linearity and the Gaussian smoothing~effect.
\end{Remark} 

\begin{Remark}[Rationale for the objective function and the known reward assumption]
     The objective function in the control formulation\eqref{eq:objective-dst}-\eqref{eq:dynamics-dst}  consists of two parts:~a terminal reward function $R(\cdot)$ at time $T$ and intermediate KL penalties. The reward $R(\cdot)$ captures human preference on the generated samples. For example, in the text-to-image generation, the reward $R$ represents how the generated data $Y_T$ is aligned with the input prompt \cite{fan2024reinforcement,black2023training}. In addition, the penalty term ensures that the fine-tuned model is not too far away from the pre-trained model, which prevents overfitting. Unlike the Shannon entropy of the (randomized) control policy commonly used in the RL literature to encourage exploration, the KL regularization in \eqref{eq:objective-dst}  is applied between two conditional probability densities. From an optimization perspective, the KL divergence introduces   concavity/convexity in control and consequently leads to a better landscape of the objective. Indeed, we will choose the parameter $\beta_t$ sufficiently large to guarantee the existence and uniqueness of the optimal control and to satisfy certain regularity conditions; see Theorem \ref{thm:well-posed}.

    In practice, fine-tuning is typically performed on a dataset containing reward ratings for each data sample, which is much smaller than the pre-trained dataset. Using this fine-tuning dataset, one can estimate the expected reward function. The assumption of known reward is not overly restrictive, as online learning to acquire new rewards is generally uncommon.
\end{Remark}

    \begin{Remark}[Connection to KL divergence on path-wise measures.]
    Another natural idea is to impose the KL divergence between the path-wise measures, i.e., $ {\rm KL}(p_{0:T} \| p_{0:T}^{\rm pre}) $, where $ p_{0:T} $ is the joint density of $ \curly{Y_{t}}_{t = 0}^{T} $ and $ p_{0:T}^{\rm pre} $ is the joint density of $ \curly{Y_{t}^{\rm pre}}_{t = 0}^{T} $. This choice is relevant to formulation \eqref{eq:dynamics-dst}. In particular, the Markov property of the dynamics and the chain rule of the KL divergence imply
    \begin{align*}
        {\rm KL}(p_{0:T} \| p_{0:T}^{\rm pre}) = \E\bracket{\sum_{t = 0}^{T-1}{\rm KL}\Big(p_{t+1\vert t}({}\cdot{}\vert Y_t) \|p_{t+1\vert t}^{\rm pre}({}\cdot{}\vert Y_t)\Big)},
    \end{align*}
    where the expectation is taken over all random variables $\{Y_t\}_{t = 0}^T$. Here, we define
    \begin{align}
        {\rm KL}(p_{0:T} \| p_{0:T}^{\rm pre}) & \coloneqq \int p_{0:T}(y_0, \dots, y_t)\log\frac{p_{0:T}(y_0, \dots, y_t)}{p_{0:T}^{\rm pre}(y_0, \dots, y_t)}{\rm d}y_0\cdots {\rm d}y_t, \label{eq:kl-path}
    \end{align} 
    and for given $y_t \in \R^d$, denote
    \begin{align}
        {\rm KL}\Big(p_{t+1\vert t}({}\cdot{}\vert y_t) \|p_{t+1\vert t}^{\rm pre}({}\cdot{}\vert y_t)\Big) & \coloneqq \int p_{t+1\vert t}(y_{t+1} \vert y_t) \log \frac{p_{t+1\vert t}(y_{t+1} \vert y_t)}{p^{\rm pre}_{t+1\vert t}(y_{t+1} \vert y_t)}{\rm d}y_{t+1}.  \label{eq:kl-condition}
    \end{align}
Thus, when $\beta_t = \beta$ for all $t$, the objective in \eqref{eq:dynamics-dst} becomes
\begin{align*}
    \mathbb{E}\left[R(Y_T) - \sum_{t = 0}^{T-1}\beta_t{\rm KL}\Big(p_{t+1\vert t}({}\cdot{}\vert Y_t) \|p_{t+1\vert t}^{\rm pre}({}\cdot{}\vert Y_t)\Big)\right] = \mathbb{E}\left[R(Y_T)\right] - \beta  {\rm KL}(p_{0:T} \| p_{0:T}^{\rm pre}),
\end{align*}
which is a common choice in fine-tuning of diffusion models \cite{fan2024reinforcement,black2023training,zhang2001some,gao2024reward}. 
    \end{Remark}
   
  Given two Gaussian densities $p_{t+1\vert t}$ and $p^{\rm pre}_{t+1\vert t}$ in  \eqref{eq:dynamics-dst}, the following lemma simplifies the KL divergence term.

  \begin{Lemma} \label{lemma:kl-l2} 
  Consider the KL divergencd defined as in \eqref{eq:kl-condition}.
  For any $y_t \in \R^d$ and any admissible control policy $u_t$, it holds that
      \begin{align}\label{eq:kl-l2-1}
        {\rm KL}\Big(p_{t+1\vert t}({}\cdot{}\vert y_t) \|p_{t+1\vert t}^{\rm pre}({}\cdot{}\vert y_t)\Big) = \frac{(1 - \alpha_{t})^{2}}{2\alpha_{t}\sigma_{t}^{2}}\norm{u_{t}(y_t) - s^{\rm pre}_{t}(y_{t})}_{2}^{2}.
    \end{align}
  \end{Lemma}
    
  Lemma \ref{lemma:kl-l2} links the KL-divergence between two conditional distributions with the squared loss between the control $u_t$ and the pre-trained score $s_t^{\rm pre}$. As the control can be interpreted as the new score of the fine-tuned model, \eqref{eq:kl-l2-1} enjoys the spirit of the score matching loss in training diffusion models \cite{Ho20,song2020score,han2024neural}. The proof of Lemma \ref{lemma:kl-l2} is based on direct calculations.

    \begin{proof}[Proof of Lemma \ref{lemma:kl-l2}]
        Recall that for any $y_t \in \R^d$,
        \begin{align*}
           {p_{t+1\vert t}({}\cdot{}\vert y_{t})} = f\parenthesis{{}\cdot{}\vert\mu_t(y_t), \sigma_{t}^{2} I_d} \;\; \text{and}\;\; p_{t+1\vert t}^{\rm pre}({}\cdot{}\vert y_{t}) = f\parenthesis{{}\cdot{}\vert\mu_t^{\rm pre}(y_t), \sigma_{t}^{2} I_d},
        \end{align*}
        with 
        \begin{align*}
            \mu_t(y_t) = \frac{1}{\sqrt{\alpha_t}}\parenthesis{y_t + (1 - \alpha_t)u_t(y_t)} \;\; \text{and} \;\; \mu_t^{\rm pre}(y_t) = \frac{1}{\sqrt{\alpha_t}}\parenthesis{y_t + (1 - \alpha_t)s^{\rm pre}_t(y_t)}.
        \end{align*}
        Thus, for any $y_t, y_{t+1} \in \R^d$, we have
        \begin{align*}
            \log \left(\frac{p_{t+1\vert t}(y_{t+1} \vert y_t)}{p^{\rm pre}_{t+1\vert t}(y_{t+1} \vert y_t)} \right)= -\frac{1}{2\sigma_t^2}\norm{y_{t+1} - \mu_t(y_t)}_2^2 + \frac{1}{2\sigma_t^2}\norm{y_{t+1} - \mu_t^{\rm pre}(y_t)}_2^2. 
        \end{align*}
        Denote $\E_{p_{t+1\vert t}}$ as the expectation under the conditional density $p_{t+1\vert t}(\cdot\vert y_t)$ of $Y_{t+1}$ given $Y_t = y_t$. By definition of the KL divergence, we have %
        \begin{align}
             {\rm KL}\Big(p_{t+1\vert t}({}\cdot{}\vert y_t) \|p_{t+1\vert t}^{\rm pre}({}\cdot{}\vert y_t)\Big) & =\E_{p_{t+1\vert t}}\bracket{\log \left(\frac{p_{t+1\vert t}(Y_{t+1} \vert y_t)}{p^{\rm pre}_{t+1\vert t}(Y_{t+1} \vert y_t)}\right)} \nonumber\\ & = -\frac{1}{2\sigma_t^2}\E_{p_{t+1\vert t}}\bracket{\norm{Y_{t+1} - \mu_t(y_t)}_2^2} +  \frac{1}{2\sigma_t^2}\E_{p_{t+1\vert t}}\bracket{\norm{Y_{t+1} - \mu_t^{\rm pre}(y_t)}_2^2}. \label{eq:kl-simple}
        \end{align}
        Note that 
        \begin{align*}
            \E_{p_{t+1\vert t}}\bracket{\norm{Y_{t+1} - \mu_t^{\rm pre}(y_t)}_2^2} & = \E_{p_{t+1\vert t}}\bracket{\norm{Y_{t+1} - \mu_t(y_t)}_2^2} + \E_{p_{t+1\vert t}}\bracket{\norm{\mu_t(y_t) - \mu_t^{\rm pre}(y_t)}_2^2} \\ & \quad + 2\E_{p_{t+1\vert t}}\bracket{(Y_{t+1} - \mu_t(y_t))^\top(\mu_t(y_t) -\mu_t^{\rm pre}(y_t))} \\ & = \E_{p_{t+1\vert t}}\bracket{\norm{Y_{t+1} - \mu_t(y_t)}_2^2} + \norm{\mu_t(y_t) - \mu_t^{\rm pre}(y_t)}_2^2,
        \end{align*}
        where we use the fact that $\E_{t+1\vert t}\bracket{Y_{t+1}} = \mu_t(y_t)$.
        Plugging the above equality into \eqref{eq:kl-simple}, we obtain
        \begin{align*}
            {\rm KL}\Big(p_{t+1\vert t}({}\cdot{}\vert y_t) \|p_{t+1\vert t}^{\rm pre}({}\cdot{}\vert y_t)\Big) = \frac{1}{2\sigma_t^2}\norm{\mu_t(y_t) - \mu_t^{\rm pre}(y_t)}_2^2 =  \frac{(1 - \alpha_{t})^{2}}{2\alpha_{t}\sigma_{t}^{2}}\norm{u_{t}(y_t) - s^{\rm pre}_{t}(y_{t})}_{2}^{2},
        \end{align*}
        which completes the proof. 
    \end{proof}

    With Lemma \ref{lemma:kl-l2}, we define the optimal value function at time $ t $ as
    \begin{align}
    V^{*}_{t}(y) &\coloneqq \sup_{\curly{U_{\ell}}_{\ell \geq t}\in \mcal{U}_t}\mathbb{E}\left[R(Y_T) - \sum_{\ell = t}^{T-1}\beta_\ell{\rm KL}\Big(p_{\ell+1\vert \ell}({}\cdot{}\vert Y_\ell) \|p_{\ell+1\vert \ell}^{\rm pre}({}\cdot{}\vert Y_\ell)\Big)\,\Big\vert \,Y_t=y\right] \nonumber \\
    &=\sup_{\curly{u_{\ell}}_{\ell \geq t}}\mathbb{E}\left[R(Y_T) - \sum_{\ell = t}^{T-1}\beta_\ell\frac{(1 - \alpha_{\ell})^{2}}{2\alpha_{\ell}\sigma_{\ell}^{2}}\norm{u_{\ell}(Y_\ell) - s^{\rm pre}_{\ell}(Y_{\ell})}_{2}^{2}\,\Big\vert \,Y_t=y\right]. \label{eq:opt-value}
    \end{align}
The Dynamic Programming Principle implies that $V_t^*$ satisfies the Bellman equation:
    \begin{align}
        V^{*}_{t}(y) =\sup_{u_{t}} \ \E\bracket{V^{*}_{t+1}\left(\frac{1}{\sqrt{\alpha_t}}\parenthesis{y + (1 - \alpha_t)u_t(y)} + \sigma_t W_t\right) - \beta_t \frac{(1 - \alpha_{t})^{2}}{2\alpha_{t}\sigma_{t}^{2}}\norm{u_{t}(y) - s^{\rm pre}_{t}(y)}_{2}^{2}}.\label{eq:bellman-V*}
    \end{align} 
In the next section, we will discuss the well-posedness of the optimal control problem \eqref{eq:objective-dst}--\eqref{eq:dynamics-dst} and the regularity of the optimal value function $V_t^*$. 
    
\subsection{Regularity and well-posedness}
In this section, we establish the regularity of the optimal value function and the optimal control policy. To start, we assume the following assumptions on the reward and pre-trained score functions.
\begin{Assumption}[Smoothness of the reward]\label{ass:smooth r}
    Assume $r(y) \coloneqq \mathbb{E}[R(y)]$ is $L_{0}^r$-Lipschitz and $L_{1}^r$-gradient Lipschitz in $y \in \R^d$, i.e.,
    \begin{align}
        \abs{r(y_1) - r(y_2) } \leq L_{0}^r\norm{y_1 - y_2}_2, \\ 
        \norm{\nabla r(y_1) - \nabla r(y_2) }_2 \leq L_{1}^r\norm{y_1 - y_2}_2.
    \end{align}
\end{Assumption}

\begin{Assumption}[Smoothness of the pre-trained score function]\label{ass:smooth s}
    Assume $s_t^{\rm pre}$ is $L_{0}^{s}$-Lipschitz and $L_{1}^{s}$-gradient Lipschitz in $y \in \R^d$, i.e.,
    \begin{align}
        \norm{s_t^{\rm pre}(y_1) - s_t^{\rm pre}(y_2) }_2 \leq L_{0,t}^{s}\norm{y_1 - y_2}_2, \\ 
        \norm{\nabla s_t^{\rm pre}(y_1) - \nabla s_t^{\rm pre}(y_2) }_2 \leq L_{1,t}^{s}\norm{y_1 - y_2}_2.
    \end{align}
\end{Assumption}

While Assumptions \ref{ass:smooth r} and \ref{ass:smooth s} guarantee the smoothness of both the expected reward $r$ and the pre-trained score $\curly{s_t^{\rm pre}}_{t=0}^{T-1}$, no convexity assumptions are imposed, and thus the control problem \eqref{eq:objective-dst}--\eqref{eq:dynamics-dst} is in general {\it non-concave} with respect to $\curly{u_t}_{t=0}^{T-1}$.

The next theorem establishes the existence and uniqueness of the optimal control, as well as the regularities of the solution to problem \eqref{eq:objective-dst}--\eqref{eq:dynamics-dst}. For ease of exposition, we define  a series of constants recursively. Let $0 < \lambda_t < 1$ for each $ t < T$. Set $L_{0, T}^{V^{*}} = L_0^r$ and $L_{1, T}^{V^{*}} = L_1^r$. For $ t < T$, define
\begin{eqnarray}
    L_{0,t}^{V^{*}} &=& \frac{1}{\sqrt{\alpha_t}}(1 + (1 - \alpha_t)L_{0, t}^s)L_{0, t+1}^{V^{*}}, \label{eq:LV0star}\\ 
    L_{1, t}^{V^{*}} &=& \frac{1}{\alpha_{t}}(1 + (1 - \alpha_{t})L_{0, t}^s)(1 + (1 - \alpha_{t})L_{0, t}^{u^{*}})L_{1,t+1}^{V^{*}} + \frac{1 - \alpha_{t}}{\sqrt{\alpha_{t}}}L_{1, t}^s L_{0,t+1}^{V^{*}},\label{eq:LV1star} \\ 
    L_{0, t}^{u^*} & =& \lambda_{t}^{-1}\parenthesis{L_{0, t}^s + \frac{1 - \lambda_t}{1- \alpha_t}}, \label{eq:Lu0star}\\ 
    L_{1, t}^{u^*}  &= & \lambda_{t}^{-1}\parenthesis{L_{1, t}^{s} +  \frac{\E\bracket{\norm{W_{t}}_2}(1 - \lambda_t)}{(1 - \alpha_t)\sqrt{\alpha_{t}}\sigma_{t}}\parenthesis{1 + (1 - \alpha_{t})L_{0,t}^{u^{*}}}^{2}}.\label{eq:Lu1star}
\end{eqnarray}
Here, $\E\bracket{\norm{W_t}_2} = \frac{\sqrt{2}\Gamma((d+1)/2)}{\Gamma(d/2)} < \infty $ is a constant with $\Gamma(\cdot)$ being the Gamma function. Note that the above constants only depend on the system parameters $\alpha_t$, $L_0^r$, $L_1^r$, $L_0^s$ and $L_1^s$ as well as the hyper-parameter $\lambda_t$. In the next theorem, we show that the above constants are indeed the Lipschitz coefficients for $\curly{V_t^{*}}_{t = 0}^{T}$ and $\curly{u_t^{*}}_{t=0}^{T-1}$.

\begin{Theorem}[Regularity and well-posedness]\label{thm:well-posed}
Suppose Assumptions \ref{ass:smooth r} and \ref{ass:smooth s} hold and let the constants be defined according to \eqref{eq:LV0star}--\eqref{eq:Lu1star}. Choose $\beta_t$ such that $ 1 - \frac{\sigma_t^2}{\beta_t}L_{1, t+1}^{V^*} \geq \lambda_{t} > 0  $. Then,
\begin{enumerate}[(i)]
\item The optimal value function $V_t^*(y)$ defined in \eqref{eq:opt-value} is $L_{0,t}^{V^{*}}$-Lipschitz, differentiable and $L_{1,t}^{V^{*}}$-gradient Lipschitz in $y \in \R^d$.

    \item There is a unique optimal control $u_t^*: \R^d \to \R^d$ of problem \eqref{eq:objective-dst}--\eqref{eq:dynamics-dst} satisfying
\begin{eqnarray}\label{eq:u*(y)}
     u_{t}^{*}(y)  = s_{t}^{\rm pre}(y) + \frac{\sqrt{\alpha_{t}}\sigma_{t}^{2}}{(1 - \alpha_{t})\beta_t}\E\bracket{\nabla V_{t+1}^{*}\left(\frac{1}{\sqrt{\alpha_t}}\parenthesis{y + (1 - \alpha_t)u_t^*(y)} + \sigma_t W_t\right)}.
\end{eqnarray}
Moreover, the optimal value function $V_t^*$ is the unique $\mcal{C}^1$ solution to the Bellman equation \eqref{eq:bellman-V*}. 
\item \label{item:lip-u*} The optimal control $u_t^*(y)$ is $L_{0, t}^{u^{*}}$-Lipschitz, differentiable, and $ L_{1,t}^{u^{*}} $-gradient Lipschitz in $y \in \R^d$.

\end{enumerate}
\end{Theorem}

When the regularization coefficient $\beta_t$ is sufficiently large, Theorem \ref{thm:well-posed} states that the optimal value function $V_t^*(\cdot)$ is Lipschitz and gradient Lipschitz. The choice of $\beta_t$ also ensures the right hand side of \eqref{eq:bellman-V*} is strongly concave in $u_t(\cdot)$, guaranteeing the existence and uniqueness of the optimal control $u_t^*$. Furthermore, Theorem \ref{thm:well-posed} shows the optimal control $u_t^*$ is also Lipschitz and gradient Lipschitz. The regularities of $V_t^*$ and $u_t^*$ will serve as the foundation of the algorithm design and convergence analysis in the subsequent sections. \\

{We outline a brief proof sketch to highlight the ideas before providing the detailed proof, which is based on backward induction. Assuming $V_{t+1}^*$ is Lipschitz and gradient Lipschitz, the first-order optimality condition  implies that \eqref{eq:u*(y)} holds, given the  optimization problem is unconstrained. Direct calculation leads to the Lipchitz condition of $u_t^*$ and the smoothing effect of the Gaussian random variable is utilized to obtain differentiability. In addition, the Lipschitz property of $\nabla u_t^*$ is a consequence of the integration by parts formula in \eqref{eq:magic-1}; hence \eqref{item:lip-u*} holds. Finally, the regularities of $V_t^*$ follow from the Lipschitz conditions of $u_t^*$ and $\nabla u_t^*$, which completes the induction.}

\begin{proof}
We prove Theorem \ref{thm:well-posed} by backward induction. At time $t = T$, Assumption \ref{ass:smooth r} implies $V_T^*(y) = R(y)$ is $L_{0, T}^{V^*}$-Lipschitz and $L_{1, T}^{V^*}$-gradient Lipschitz with $L_{0, T}^{V^*} = L_0^r$ and $L_{1, T}^{V^*} = L_1^r$.

\underline{Step 1: Lipschitz condition of $ u_{t}^{*} $.} Assume that $V^*_{t+1}$ is $L_{0, t+1}^{V^*}$-Lipschitz and $L_{1, t+1}^{V^*}$-gradient Lipschitz in~$y\in \R^d$. The choice of $\beta_t$ implies that the mapping
\begin{eqnarray}\label{eq:obj-concave}
    u \mapsto \E\bracket{V^{*}_{t+1}\left(\frac{1}{\sqrt{\alpha_t}}\parenthesis{y + (1 - \alpha_t)u} + \sigma_t W_t\right) - \beta_t \frac{(1 - \alpha_{t})^{2}}{2\alpha_{t}\sigma_{t}^{2}}\norm{u - s^{\rm pre}_{t}(y)}_{2}^{2}}
\end{eqnarray}
is $\gamma_t$-strongly concave with $ \gamma_{t} = \frac{(1 - \alpha_{t})^{2}}{\alpha_{t}}\parenthesis{\frac{\beta_{t}}{\sigma_{t}^{2}} - L_{1, t+1}^{V^{*}}} > 0 $. Hence, there is a unique optimal control $u_t^*$ satisfying
\begin{eqnarray}\label{eq:u*(y)-re}
     u_{t}^{*}(y)  = s_{t}^{\rm pre}(y) + \frac{\sqrt{\alpha_{t}}\sigma_{t}^{2}}{(1 - \alpha_{t})\beta_t}\E\bracket{\nabla V_{t+1}^{*}\left(\frac{1}{\sqrt{\alpha_t}}\parenthesis{y + (1 - \alpha_t)u_t^*(y)} + \sigma_t W_t\right)},
\end{eqnarray}
which is obtained by setting the gradient of mapping \eqref{eq:obj-concave} as zero for each $y \in \R^d$.
Here, we apply the Lipschitz condition of $V_{t+1}^*$ and the dominated convergence theorem to interchange the operators. 

We next prove the Lipschitz condition of $u_t^*$ in $y\in \R^d$. Note that for any $y_1$ and $y_2$ in $\R^d$, Eq.~\eqref{eq:u*(y)-re} implies
\begin{align*}
    u_t^*(y_1) - u_t^*(y_2) & = s_{t}^{\rm pre}(y_1) - s_{t}^{\rm pre}(y_2)  + \frac{\sqrt{\alpha_{t}}\sigma_{t}^{2}}{(1 - \alpha_{t})\beta_t}\E\bracket{\nabla V_{t+1}^{*}\parenthesis{\frac{1}{\sqrt{\alpha_t}}\parenthesis{y_1 + (1 - \alpha_t)u_t^*(y_1)} + \sigma_t W_t)}} \\ & \quad - \frac{\sqrt{\alpha_{t}}\sigma_{t}^{2}}{(1 - \alpha_{t})\beta_t}\E\bracket{\nabla V_{t+1}^{*}\parenthesis{\frac{1}{\sqrt{\alpha_t}}\parenthesis{y_2 + (1 - \alpha_t)u_t^*(y_2)} + \sigma_t W_t)}}.
\end{align*}
Utilizing the Lipschitz condition of $s_t^{\rm pre}$ and $\nabla V_{t+1}^{*}$, we obtain
\begin{align*}
    \norm{u_t^*(y_1) - u_t^*(y_2)}_2 & \leq \norm{s_{t}^{\rm pre}(y_1) - s_{t}^{\rm pre}(y_2) }_2  + \frac{\sqrt{\alpha_{t}}\sigma_{t}^{2}}{(1 - \alpha_{t})\beta_t}\E\Big[\Big\|\nabla V_{t+1}^{*}\parenthesis{\frac{1}{\sqrt{\alpha_t}}\parenthesis{y_1 + (1 - \alpha_t)u_t^*(y_1)} + \sigma_t W_t)} \\ & \qquad\qquad\qquad - \nabla V_{t+1}^{*}\parenthesis{\frac{1}{\sqrt{\alpha_t}}\parenthesis{y_2 + (1 - \alpha_t)u_t^*(y_2)} + \sigma_t W_t)}\Big\|_2\Big] \\ & \leq L_{0, t}^s\norm{y_1 - y_2}_2 \\ & \quad + \frac{\sqrt{\alpha_{t}}\sigma_{t}^{2}}{(1 - \alpha_{t})\beta_t}L_{1, t+1}^{V^*}\norm{\frac{1}{\sqrt{\alpha_t}}\parenthesis{y_1 + (1 - \alpha_t)u_t^*(y_1)} - \frac{1}{\sqrt{\alpha_t}}\parenthesis{y_2 + (1 - \alpha_t)u_t^*(y_2)}}_2 \\ & \leq L_{0, t}^s\norm{y_1 - y_2}_2  + \frac{\sqrt{\alpha_{t}}\sigma_{t}^{2}}{(1 - \alpha_{t})\beta_t}L_{1, t+1}^{V^*}\parenthesis{\frac{1}{\sqrt{\alpha_t}}\norm{y_1 - y_2}_2 + \frac{1 - \alpha_t}{\sqrt{\alpha_t}}\norm{u_t^*(y_1) - u_t^*(y_2)}_2}.
\end{align*}
Equivalently, we have 
\begin{align*}
      \parenthesis{1 - \frac{\sigma_t^2}{\beta_t}L^{V^*}_{1, t+1}}\norm{u_t^*(y_1) - u_t^*(y_2)}_2  & \leq \parenthesis{L_{0, t}^s + \frac{\sigma_t^2 L^{V^*}_{1, t+1}}{(1- \alpha_t)\beta_t}}\norm{y_1 - y_2}_2.
\end{align*}
Since $1 - \frac{\sigma_t^2}{\beta_t}L^{V^*}_{1, t+1}\geq \lambda_t > 0$, we established the Lipschitz condition of the optimal control $u_t^*$:
\begin{align*}
    \norm{u_t^*(y_1) - u_t^*(y_2)}_2 \leq \lambda_{t}^{-1}\parenthesis{L_{0, t}^s + \frac{1 - \lambda_t}{1- \alpha_t}}\norm{y_1 - y_2}_2 = L_{0, t}^{u^*}\norm{y_1 - y_2}_2.,
\end{align*}
where we recall $L_{0, t}^{u^*}$ defined in \eqref{eq:Lu0star}.

\underline{Step 2: Differentiability of $ u_{t}^{*} $.} We now argue that $ u_t^*(y)$ is differentiable for all $y\in \mathbb{R}^d$. Let $h \in \R^d$ be an arbitrary non-zero vector and let $y \in \R^d$ be fixed. Since $s_t^{\rm pre}$ is differentiable, we have
\begin{align}
    s_t^{\rm pre}(y + h) - s_t^{\rm pre}(y) = \nabla s_t^{\rm pre}(y) h + o(\norm{h}_2). \label{eq:s-diff}
\end{align}
Moreover, by the inductive hypothesis, $\nabla V_{t+1}^*(y)$ is Lipschitz and thus $\nabla^2 V_{t+1}^*(y)$ exists for almost all $y$ and is bounded. \cite[Theorem 2.27]{folland1999real} implies the mapping $z \mapsto \E\bracket{\nabla V^{*}_{t+1}(z + \sigma_t W_t)}$ is differentiable everywhere and its derivative is given by
\begin{align}\label{eq:derivative nabla V}
    \frac{\partial}{\partial z}\E\bracket{\nabla V^{*}_{t+1}(z + \sigma_t W_t)} = \E\bracket{\nabla^2 V^{*}_{t+1}(z + \sigma_t W_t)}. 
\end{align}
Let $z = \frac{1}{\sqrt{\alpha_t}}\parenthesis{y + (1 - \alpha_t)u_t^*(y)} $ and $k = \frac{1}{\sqrt{\alpha_t}}(h + (1 - \alpha_t)(u_t^{*}(y+h) - u_t^{*}(y))$. Eq.~\eqref{eq:derivative nabla V} implies
\begin{align}
    \E\bracket{\nabla V^{*}_{t+1}(z + k + \sigma_t W_t)} - \E\bracket{\nabla V^{*}_{t+1}(z + \sigma_t W_t)} = \E\bracket{\nabla^2 V_{t+1}^*(z + \sigma_t W_t)}k + o(\norm{k}_2). \label{eq:nabla-v*-diff}
\end{align}
Since $u_t^*$ is Lipschitz, we have $k = \mcal{O}(\norm{h}_2)$ as $\norm{h}_2 \to 0$. Combining \eqref{eq:s-diff} and \eqref{eq:nabla-v*-diff}, we obtain
\begin{align*}
    u_t^*(y+h) - u_t^*(y) & = s_{t}^{\rm pre}(y+h) - s_{t}^{\rm pre}(y)  + \frac{\sqrt{\alpha_{t}}\sigma_{t}^{2}}{(1 - \alpha_{t})\beta_t}\E\bracket{\nabla V_{t+1}^{*}\parenthesis{\frac{1}{\sqrt{\alpha_t}}\parenthesis{(y+h) + (1 - \alpha_t)u_t^*(y+h)} + \sigma_t W_t)}} \\ & \quad - \frac{\sqrt{\alpha_{t}}\sigma_{t}^{2}}{(1 - \alpha_{t})\beta_t}\E\bracket{\nabla V_{t+1}^{*}\parenthesis{\frac{1}{\sqrt{\alpha_t}}\parenthesis{y + (1 - \alpha_t)u_t^*(y)} + \sigma_t W_t)}} \\ & = \nabla s_t^{\rm pre}(y) h  + \frac{\sqrt{\alpha_{t}}\sigma_{t}^{2}}{(1 - \alpha_{t})\beta_t}\mcal{H}_{t+1}(y)\parenthesis{\frac{1}{\sqrt{\alpha_t}}(h + (1 - \alpha_t)(u_t^{*}(y+h) - u_t^{*}(y))} + o(\norm{h}_2),
\end{align*}
where $\mcal{H}_{t+1}(y) \coloneqq \E\bracket{\nabla^2 V_{t+1}^{*}\left(\frac{1}{\sqrt{\alpha_t}}\parenthesis{y + (1 - \alpha_t)u_t^*(y)} + \sigma_t W_t\right)}  $.
Re-arranging the terms leads to
\begin{align*}
    u_t^*(y+h) - u_t^*(y) = \parenthesis{I_d - \frac{\sigma_t^2}{\beta_t}\mcal{H}_{t+1}(y)}^{-1}\parenthesis{\nabla s_t^{\rm pre}(y) + \frac{\sigma_t^2}{(1 - \alpha_t)\beta_t}\mcal{H}_{t+1}(y)}h + o(\norm{h}_2),
\end{align*}
which proves that $u_t^*(y)$ is differentiable for any $y\in \R^d$ and its derivative is given by 
\begin{align}\label{eq:grad u*-re}
    \nabla u_t^*(y) = \parenthesis{I_d - \frac{\sigma_t^2}{\beta_t}\mcal{H}_{t+1}(y)}^{-1}\parenthesis{\nabla s_t^{\rm pre}(y) + \frac{\sigma_t^2}{(1 - \alpha_t)\beta_t}\mcal{H}_{t+1}(y)}.
\end{align}

\underline{Step 3: Lipschitz condition of $  \nabla u_{t}^{*}$.} Next, we show that $\nabla u_t^*$ is Lipschitz. Taking the derivative of both sides of~\eqref{eq:u*(y)-re},
\begin{align}
    \nabla u_{t}^{*}(y) = \nabla s_{t}^{\rm pre}(y) + \frac{\sqrt{\alpha_{t}}\sigma_{t}^{2}}{(1 - \alpha_{t})\beta_t}\cdot\frac{\partial}{\partial y} \E\bracket{\nabla {V}^*_{t+1}\left(\frac{1}{\sqrt{\alpha_t}}\parenthesis{y + (1 - \alpha_t)u_t^{*}(y)} + \sigma_t W_t\right)}. \label{eq:u-diff}
\end{align}
For any $ z \in \R^{d} $, the integration by parts formula implies that
\begin{align}\label{eq:magic-1}
    \frac{\partial}{\partial z}\E\bracket{\nabla {V}^*_{t+1}\left(z + \sigma_t W_t\right)} = \E\bracket{\nabla^{2} {V}^*_{t+1}\left(z + \sigma_t W_t\right)} = \E\bracket{\nabla {V}^*_{t+1}\left(z + \sigma_t W_t\right)\frac{W_t\top}{\sigma_t}}.
\end{align}
Substituting $ z = \frac{1}{\sqrt{\alpha_t}}\parenthesis{y + (1 - \alpha_t)u_t^{*}(y)} $ and applying the chain rule, we obtain
\begin{align*}
     \frac{\partial}{\partial y}\E\bracket{\nabla {V}^*_{t+1}\left(\frac{1}{\sqrt{\alpha_t}}\parenthesis{y + (1 - \alpha_t)u_t^{*}(y)} + \sigma_t W_t\right)}  = \mcal{W}_{t+1}(y)\mcal{U}_t(y),
\end{align*}
where $\mcal{W}_{t+1}(y)  \coloneqq  \E\bracket{\nabla {V}^*_{t+1}\left(\frac{1}{\sqrt{\alpha_t}}\parenthesis{y + (1 - \alpha_t)u_t^{*}(y)} + \sigma_t W_t\right)\frac{W_t^\top}{\sigma_t}}$ and $\mathcal{U}_t(y) \coloneqq \frac{1}{\sqrt{\alpha_t}}\parenthesis{I_d + (1 - \alpha_t) \nabla u_t^*(y)}$. It follows from \eqref{eq:u-diff} that
\begin{align}
    \nabla u_{t}^{*}(y)  & = \nabla s_{t}^{\rm pre}(y)  + \frac{\sqrt{\alpha_{t}}\sigma_{t}^{2}}{(1 - \alpha_{t})\beta_t} \mcal{W}_{t+1}(y)\mcal{U}_t(y). \label{eq:u-diff-1}
\end{align}
For any $y_1$ and $y_2$ in $\R^d$, Eq.~\eqref{eq:u-diff-1} implies that %
\begin{align}
\nabla u_t^{*} (y_1) - \nabla u_t^{*} (y_2)  & = \nabla s_{t}^{\rm pre}(y_{1}) - \nabla s_{t}^{\rm pre}(y_{2})  +  \frac{\sqrt{\alpha_{t}}\sigma_{t}^{2}}{(1 - \alpha_{t})\beta_t} \big(\mcal{W}_{t+1}(y_1)\mcal{U}_t(y_1)  - \mcal{W}_{t+1}(y_2)\mcal{U}_t(y_2)\big). \label{eq:u-lip-1}
\end{align}
Note that for any $ y \in \R^d$ we have
\begin{align}
    & \norm{\mcal{W}_{t+1}(y)}_{2} = \norm{\mcal{H}_{t+1}(y)}_2 \leq L_{1, t+1}^{V^*}, \label{eq:upper-bound-magic-1} \\ 
    & \norm{\mcal{U}_t(y)}_2 \leq \frac{1}{\sqrt{\alpha_{t}}}\parenthesis{1 + (1 - \alpha_{t})L_{0,t}^{u^{*}}}, \label{eq:upper-bound-nabla-u-n-1}
\end{align}
where \eqref{eq:upper-bound-magic-1} holds by applying the identity \eqref{eq:magic-1}, and \eqref{eq:upper-bound-nabla-u-n-1} is a consequence of the Lipschitz condition of $ u_{t}^{*} $.
Moreover, for any $ y_{1} $ and $ y_{2} $ in $\R^d$, the Lipschitz conditions of $ \nabla {V}^*_{t+1} $ and~$ u_{t}^{*} $ imply that
\begin{align}
     \norm{\mcal{W}_{t+1}(y_1) - \mcal{W}_{t+1}(y_2)}_2 \nonumber &  \leq \frac{\E\bracket{\norm{W_{t}}_2}}{\sigma_{t}}L_{1, t+1}^{{V}^*}\frac{1}{\sqrt{\alpha_{t}}}\parenthesis{\norm{y_{1} - y_{2}}_2 + (1 - \alpha_{t})L_{0, t}^{u^{*}}\norm{y_{1} - y_{2}}_2} \nonumber \\ & = \frac{\E\bracket{\norm{W_{t}}_2}}{\sqrt{\alpha_{t}}\sigma_{t}}L_{1, t+1}^{V^{*}}\parenthesis{1 + (1 - \alpha_{t})L_{0, t}^{u^{*}}}\norm{y_{1} - y_{2}}_2. \label{eq:Lip-nabla-V-w-1}
\end{align}
Furthermore, the definition of $\mcal{U}_t$ implies
\begin{align}
    \norm{\mcal{U}_t(y_1) - \mcal{U}_t(y_2)}_2 & = \frac{1 - \alpha_{t}}{\sqrt{\alpha_{t}}}\norm{\nabla u_{t}^{*}(y_{1}) - \nabla u_{t}^{*}(y_{2})}_2 \label{eq:Lip-nabla-u-n-1}.
\end{align}
Combining \eqref{eq:u-lip-1}--\eqref{eq:Lip-nabla-u-n-1} together, we have
\begin{align*}
    \norm{\nabla u_{t}^{*}(y_{1}) - \nabla u_{t}^{*}(y_{2})}_2 & \leq \norm{\nabla s_{t}^{\rm pre}(y_{1}) - \nabla s_{t}^{\rm pre}(y_{2})}_2  +  \frac{\sqrt{\alpha_{t}}\sigma_{t}^{2}}{(1 - \alpha_{t})\beta_t} \norm{\mcal{W}_{t+1}(y_1)  - \mcal{W}_{t+1}(y_2)}_2\norm{\mcal{U}_t(y_1)}_2 \\ & \qquad + \frac{\sqrt{\alpha_{t}}\sigma_{t}^{2}}{(1 - \alpha_{t})\beta_t} \norm{\mcal{W}_{t+1}(y_2)}_2\norm{\mcal{U}_t(y_1) - \mcal{U}_t(y_2)}_2 
    \\ & \leq L_{1, t}^{s}\norm{y_{1} - y_{2}}_2 + \frac{\sqrt{\alpha_{t}}\sigma_{t}^{2}}{(1 - \alpha_{t})\beta_t}\bigg(L_{1, t+1}^{V^{*}}\frac{1 - \alpha_{t}}{\sqrt{\alpha_{t}}} \norm{\nabla u_{t}^{*}(y_{1}) - \nabla u_{t}^{*}(y_{2})}_2 \\ & \qquad + \frac{1}{\alpha_{t}}\parenthesis{1 + (1 - \alpha_{t})L_{0,t}^{u^{*}}}^{2}\frac{\E\bracket{\norm{W_{t}}_2}}{\sigma_{t}}L_{1, t+1}^{V^{*}}\norm{y_{1} - y_{2}}_2\bigg).
\end{align*}
Since $ \frac{\sigma_{t}^{2}}{\beta_{t}}L_{1, t+1}^{V^{*}} \leq 1 - \lambda_{t} $, we deduce that
\begin{align*}
    \norm{\nabla u_{t}^{*}(y_{1}) - \nabla u_{t}^{*}(y_{2})}_2 & \leq \lambda_{t}^{-1}\parenthesis{L_{1, t}^{s} +  \frac{\E\bracket{\norm{W_{t}}_2}(1 - \lambda_t)}{(1 - \alpha_t)\sqrt{\alpha_{t}}\sigma_{t}}\parenthesis{1 + (1 - \alpha_{t})L_{0,t}^{u^{*}}}^{2}}\norm{y_1 - y_2}_2 \\ &  = L_{1, t}^{u^*}\norm{y_1 - y_2}_2,
\end{align*}
where we recall $L_{1, t}^{u^*}$ defined in \eqref{eq:Lu1star}.

\underline{Step 4: Lipschitz conditions of $ V_{t}^{*} $ and $ \nabla V_{t}^{*}$.} Finally, we turn to prove the Lipschitz and gradient Lipschitz conditions of $V_t^*$. Plugging $u_t^*$ into the Bellman equation \eqref{eq:bellman-V*}, we have
\begin{align}\label{eq:bellman-u*}
 V_t^*(y) & = \E\bracket{V^{*}_{t+1}\left(\frac{1}{\sqrt{\alpha_t}}\parenthesis{y + (1 - \alpha_t)u_t^*(y)} + \sigma_t W_t\right)} - \beta_t \frac{(1 - \alpha_{t})^{2}}{2\alpha_{t}\sigma_{t}^{2}}\norm{u_t^*(y) - s^{\rm pre}_{t}(y)}_{2}^{2}.
\end{align}
Since $V_{t+1}^*$ and $s_t^{\rm pre}$ are differentiable with Lipschitz gradients and $u_t^*$ is differentiable, we know that $V_{t}^{*}$ is differentiable and 
\begin{align*}
 \nabla V_t^*(y) &= \frac{\partial}{\partial y} \E\Bigg[V^{*}_{t+1}\left(\frac{1}{\sqrt{\alpha_t}}\parenthesis{y + (1 - \alpha_t)u_t^*(y)} + \sigma_t W_t\right)- \beta_t \frac{(1 - \alpha_{t})^{2}}{2\alpha_{t}\sigma_{t}^{2}}\norm{u_t^*(y) - s^{\rm pre}_{t}(y)}_{2}^{2}\Bigg]  \\
 &= \frac{1}{\sqrt{\alpha_t}}\parenthesis{I_d 
 + (1 - \alpha_t) \nabla u_t^*(y)}^\top\E\bracket{\nabla V_{t+1}^{*}\left(\frac{1}{\sqrt{\alpha_t}}\parenthesis{y + (1 - \alpha_t)u_t^*(y)} + \sigma_t W_t\right)}\\
 & \quad - \beta_t \frac{(1 - \alpha_{t})^{2}}{\alpha_{t}\sigma_{t}^{2}}(\nabla u_t^*(y) - \nabla s^{\rm pre}_{t}(y))^\top(u_t^* - s^{\rm pre}_{t}(y)).
\end{align*}
Define $\mathcal{G}_{t+1}(y) \coloneqq \E\bracket{\nabla V_{t+1}^{*}\left(\frac{1}{\sqrt{\alpha_t}}\parenthesis{y + (1 - \alpha_t)u_t^*(y)} + \sigma_t W_t\right)} $. It follows that 
\begin{eqnarray} 
 \nabla V_t^*(y) 
 &=&\frac{1}{\sqrt{\alpha_t}}\parenthesis{I_d 
 + (1 - \alpha_t) \nabla u_t^*(y)}^\top\mathcal{G}_{t+1}(y)- \beta_t \frac{(1 - \alpha_{t})^{2}}{\alpha_{t}\sigma_{t}^{2}}(\nabla u_t^*(y) - \nabla s^{\rm pre}_{t}(y))^\top\left(  \frac{\sqrt{\alpha_{t}}\sigma_{t}^{2}}{(1 - \alpha_{t})\beta_t}\mathcal{G}_{t+1}(y)\right)\nonumber\\
 &=&\frac{1}{\sqrt{\alpha_t}}\parenthesis{I_d 
 + (1 - \alpha_t) \nabla s_t^{\rm pre}(y)}^\top\mathcal{G}_{t+1}(y) \label{eq:grad-V*},
\end{eqnarray}
where \eqref{eq:u*(y)-re} is applied to obtain the first equality. Next, we use \eqref{eq:grad-V*} to establish the Lipschitz conditions of $ V_{t}^{*} $ and $\nabla V_t^{*}$. The Lipschitz condition of $\nabla V_{t+1}^*$ implies $\norm{\mcal{G}_{t+1}(y)}_2 \leq L_{0, t+1}^{V^*}$ for all $y \in \R^d$. Based on \eqref{eq:grad-V*}, we have
\begin{align*}
\norm{\nabla V_{t}^{*}(y)}_{2} \leq \frac{1}{\sqrt{\alpha_t}}(1 + (1 - \alpha_t)L_{0, t}^{s})
L_{0, t+1}^{V^*} = L_{0, t}^{V^*},
\end{align*}
which proves the Lipschitz condition of $ V_t^* $. Next, we prove the gradient Lipschitz condition of $ V_t^* $. For ease of exposition, we denote $\mathcal{S}_t(y) \coloneqq \frac{1}{\sqrt{\alpha_t}}\parenthesis{I_d 
+ (1 - \alpha_t) \nabla s_t^{\rm pre}(y)}$. We now use \eqref{eq:grad-V*} to show the Lipschitz condition of $\nabla V_t^*$. We first note that $ \nabla \mcal{G}_{t+1}(y) = \mcal{H}_{t+1}(y)\mcal{U}_t(y) $ is well-defined at every $y \in \R^d$, and thus
\begin{align*}
\norm{\nabla \mcal{G}_{t+1}(y)}_{2} \leq \norm{\mcal{H}_{t+1}(y)}_{2}\norm{\mcal{U}_t(y)}_{2} \leq \frac{1}{\sqrt{\alpha_{t}}}(1 + (1 - \alpha_{t})L_{0, t}^{u^*})L_{1, t+1}^{V^*}.
\end{align*}
 With the Lipschitz condition of $ \mcal{G}_{t+1} $ in hand,  for any $y_1$ and $y_2$ in $\R^d$, we have that
\begin{align*}
\norm{\nabla V_{t}^{*}(y_{1}) - \nabla V_{t}^{*}(y_{2})}_{2} & \leq \norm{\mcal{S}_t(y_{1})^{\top}\mcal{G}_{t+1}(y_{1}) - \mcal{S}_t(y_{2})^{\top}\mcal{G}_{t+1}(y_{2})}_2 \\ & \leq \norm{\mcal{S}_t(y_{1})^{\top}\parenthesis{\mcal{G}_{t+1}(y_{1}) - \mcal{G}_{t+1}(y_{2})}}_2 + \norm{\parenthesis{\mcal{S}_t(y_{1}) - \mcal{S}_t(y_{2})}^{\top}\mcal{G}_{t+1}(y_{2})}_{2} \\ & \leq \norm{\mcal{S}_t(y_{1})}_{2}\norm{\mcal{G}_{t+1}(y_{1}) - \mcal{G}_{t+1}(y_{2})}_{2} + \norm{\mcal{S}_t(y_{1}) - \mcal{S}_t(y_{2})}_{2}\norm{\mcal{G}_{t+1}(y_{2})}_{2} \\ & \leq \frac{1}{\alpha_{t}}(1 + (1 - \alpha_{t})L_{0, t}^s)(1 + (1 - \alpha_{t})L_{0, t}^{u^*})L_{1,t+1}^{V^*}\norm{y_{1} - y_{2}}_{2} \\ & \quad + \frac{1 - \alpha_{t}}{\sqrt{\alpha_{t}}}L_{1,t}^s L_{0, t+1}^{V^*}\norm{y_{1} - y_{2}}_{2},
\end{align*}
where the last equation uses the Lipschitz conditions of $s_t^{\rm pre}$, $\nabla s_t^{\rm pre}$, $V_{t+1}^*$ and $\nabla V_{t+1}^*$. Consequently, we have
\begin{align*}
    \norm{\nabla V_{t}^{*}(y_{1}) - \nabla V_{t}^{*}(y_{2})}_{2} \leq L_{1, t}^{V^*}\norm{y_1 - y_2}_2,
\end{align*}
where we recall $L_{1, t}^{V^*}$ defined in \eqref{eq:LV1star}. In other words, $L_{1, t}^{V^*}$ is indeed the Lipschitz constant of $\nabla V_t^*$. This completes the proof.

\end{proof}

\section{Algorithm development and convergence analysis}
\label{sec:algorithm}
In this section, we propose an iterative algorithm for fine-tuning and provide non-asymptotic analysis of its  convergence. Assuming that the expected reward function  $r(\cdot)$ is known, we develop an iterative algorithm to approximate the optimal policy in high-dimensional settings, suitable for practical implementation. In practice,  $r(\cdot)$ is  approximated using a small sample set with human feedback, which then serves as the basis for fine-tuning. The unknown  $r(\cdot)$ case is left as a topic for future investigation.

\subsection{Proposed algorithm}
Recall that the optimal control $u_t^*$ satisfies \eqref{eq:u*(y)}. Motivated by this observation, we propose the following iterative algorithm for fine-tuning and computing the optimal control $u_t^*$. 
\begin{algorithm}[H]
    \caption{Policy Iteration for Fine-Tuning (PI-FT)}
    \begin{algorithmic}[1] \label{algo:update-u}
    \STATE {\bf {Input:}} Expected reward function $r(\cdot)$, pre-trained model $\curly{s_t^{\rm pre}}_{t = 0}^{T} $, and number of iterations $m_t$ at each timestamp $t$ $(0\leq t \leq T-1)$.
    \STATE Set $V_T^{(m_T)}(y) = r(y)$ for all $y \in \R^d$. 
    \FOR{$t = T - 1, \dots, 0$}
    \STATE Set $u_{t}^{(0)}(y) = s_t^{\rm pre}(y)$.
    \FOR{$m = 1, \dots, m_t - 1$} 
    \STATE  Update the control using
    \begin{align}
        u_{t}^{(m+1)}(y) = s_{t}^{\rm pre}(y) + \frac{\sqrt{\alpha_{t}}\sigma_{t}^{2}}{(1 - \alpha_{t})\beta_t}\E\bracket{\nabla {V}_{t+1}^{(m_{t+1})}\left(\frac{1}{\sqrt{\alpha_t}}\parenthesis{y + (1 - \alpha_t)u_t^{(m)}(y)} + \sigma_t W_t\right)}. \label{eq:update-u_t}
    \end{align}
    \ENDFOR 
    \STATE Compute the value function $V_t^{(m_t)}$ using 
    \begin{align}
        V_{t}^{(m_t)}(y) = \mathbb{E}\left[{V}^{(m_{t+1})}_{t+1}\left(\frac{1}{\sqrt{\alpha_t}}\parenthesis{y + (1 - \alpha_t)u_t^{(m_t)}(y)} + \sigma_t W_t\right) -\beta_t \frac{(1-\alpha_t)^2}{2\alpha_t \sigma_t^2}\|u_t^{(m_t)}(y)-s_t^{\rm pre}(y)\|^2\right]. \label{eq:eval-v_t}
    \end{align}
    \ENDFOR 
    \RETURN $ \curly{u_t^{(m_t)}}_{t = 0}^{T-1} $ and $ \curly{V_t^{(m_t)}}_{t = 0}^{T} $.
    \end{algorithmic}
\end{algorithm}

{Algorithm \ref{algo:update-u} consists of two nested loops. The inner loop updates the control at each iteration $m$ and the outer loop computes the value function at time $t$.} Given the value function $V_{t+1}^{(m_{t+1})}$ at time step $t+1$, we update $u_t^{(m)}$ using \eqref{eq:update-u_t} and then evaluate the associated value function using \eqref{eq:eval-v_t}. Note that \eqref{eq:update-u_t} can be viewed as an approximation of the following update rule:
\begin{align}
        u_{t}^{(m+1)}(y) = s_{t}^{\rm pre}(y) + \frac{\sqrt{\alpha_{t}}\sigma_{t}^{2}}{(1 - \alpha_{t})\beta_t}\E\bracket{\nabla {V}_{t+1}^{*}\left(\frac{1}{\sqrt{\alpha_t}}\parenthesis{y + (1 - \alpha_t)u_t^{(m)}(y)} + \sigma_t W_t\right)}, \label{eq:update-u_t-*}
\end{align}
and \eqref{eq:eval-v_t} can be seen as an approximation of the Bellman equation \eqref{eq:bellman-V*}, where we replace the unknown $V_{t+1}^*$ with $V_{t+1}^{(m_{t+1})}$.
Direct calculation shows that the fixed-point update rule \eqref{eq:update-u_t-*} guarantees the convergence of $\curly{u_t^{(m)}}_{m\geq 0}$ to $u_t^*$; see \eqref{eq:contraction}. However, in practice we cannot access $V_{t+1}^*$. Hence, Eq.~\eqref{eq:update-u_t} can be viewed as a proxy to \eqref{eq:update-u_t-*} as long as we can control the error $\norm{\nabla V_t^{(m_t)}(y) - \nabla V_t^*(y)}_2$ for all $t$. With $u_t^{(m_t)}$ and $V_{t+1}^{(m_{t+1})}$, we use \eqref{eq:eval-v_t} to  calculate $V_{t}^{(m_t)}$, an estimate of the optimal value function $V_t^*$.

\subsection{Convergence analysis}
In this section, we analyze the convergence of Algorithm \ref{algo:update-u}. For a fixed $\lambda_t \in (0, 1)$, define a series of constants recursively as follows. Let $L_{0, T}^{\bar{V}} = L_{0}^r$ and $L_{1, T}^{\bar{V}} = L_{1}^r$. For each $t < T$, set 
\begin{align}
        L_{0,t}^{\bar{V}} &= \frac{1}{\sqrt{\alpha_t}}(1 + (1 - \alpha_t)L_{0, t}^{s})
L_{0, t+1}^{\bar{V}} + \frac{L_{0, t+1}^{\bar{V}}}{\sqrt{\alpha_{t}}}\parenthesis{1 + (1 - \alpha_{t})L_{0,t}^{u^{*}}}(1- \lambda_t), \label{eq:L0t-v-bar}  \\ \text{and} \;\; L_{1,t}^{\bar{V}} & = \frac{1}{\alpha_{t}}(1 + (1 - \alpha_{t})L_{0, t}^s)(1 + (1 - \alpha_{t})L_{0, t}^{u^*})L_{1,t+1}^{\bar{V}} + \frac{1 - \alpha_{t}}{\sqrt{\alpha_{t}}}L_{1,t}^s L_{0, t+1}^{\bar{V}} \nonumber \\ & \quad + \max_{m \geq 1}\curly{\bigg(\frac{1 - \alpha_{t}}{\sqrt{\alpha_{t}}} L_{1, t}^{u^{*}}  + (m+2)\parenthesis{1 + (1 - \alpha_{t})L_{0,t}^{u^{*}}}^{2}\frac{\E\bracket{\norm{W_{t}}_2}}{\alpha_t\sigma_{t}}\bigg)L_{0, t+1}^{\bar{V}}( 1- \lambda_t)^{m+1}}. \label{eq:L1t-v-bar}
\end{align}
Here, the constant $L_{1,t}^{\bar{V}}$ is well-defined as the maximum must be achieved for some $m < \infty$ due to the exponential decay of $(1 - \lambda_t)^{m+1}$. 
Similar to the constants defined in \eqref{eq:LV0star}--\eqref{eq:Lu1star}, the above constants only depend on the system parameters $\alpha_t$, $L_0^r$, $L_1^r$, $L_0^s$ and $L_1^s$ as well as the hyper-parameter $\lambda_t$. Later, we will show that the constants defined in \eqref{eq:L0t-v-bar} and \eqref{eq:L1t-v-bar} indeed serve as the Lipschitz coefficients of $\curly{V_t^{(m)}}_{t=0}^T$ universal over all $m \geq 0$. 
The main result of this section is stated as follows. 
\begin{Theorem}\label{thm:convergence}
Suppose Assumptions \ref{ass:smooth r} and \ref{ass:smooth s} hold. 
    For each $t < T$, choose $\beta_t$ such that $ 1 - \frac{\sigma_t^2}{\beta_t}L_{1, t+1}^{\bar{V}} \geq \lambda_{t} > 0  $. Let $\curly{u_t^{(m_t)}}_{t=0}^{T-1}$ and $\curly{V_t^{(m_t)}}_{t=0}^T$ be the output of Algorithm \ref{algo:update-u}. Then it holds that
\begin{eqnarray}
     \norm{u_t^{(m_t)}(y) - u_t^*(y)}_2  \leq & \parenthesis{(1 - \lambda_t)^{m_t}  L_{0,t+1}^{V^*} + \lambda_t^{-1}E_{t+1}}  \frac{\sqrt{\alpha_{t}}(1 - \lambda_t)}{(1 - \alpha_{t})L_{1, t+1}^{V^*}}, 
\end{eqnarray}
where
\begin{eqnarray}
        E_{t} \coloneqq \norm{\nabla V_t^{(m_t)}(y) - \nabla V_t^{*}(y)}_2  &\leq & \sum_{k = t}^{T - 1} \parenthesis{\prod_{\ell = t}^{k - 1}C_{1, \ell}} C_{2, k} (1 - \lambda_k)^{m_k + 1}, \label{eq:Et}
\end{eqnarray}
and 
\begin{align*}
    C_{1,t} =  \frac{1 + (1 - \alpha_t)L_{0, t}^{s}}{\lambda_t\sqrt{\alpha_t}} , \;\; \text{and} \;\; C_{2, t} = \frac{\parenthesis{1 + (1 - \alpha_t)L_{0, t}^{u^*}}}{\sqrt{\alpha_t}}L_{0, t+1}^{\bar{V}} +  \frac{\parenthesis{1 + (1 - \alpha_t)L_{0, t}^{s}}}{\sqrt{\alpha_t}}L_{0, t+1}^{V^*}.
\end{align*}
\end{Theorem}
   
   To obtain a simplified convergence rate from Theorem \ref{thm:convergence}, we set  $\lambda_t \equiv \lambda \coloneqq \min_{0\leq t\leq T-1} \lambda_t$ and $m_t \equiv M$ for all $t$. In this case, we have $E_t = \mcal{O}((1 - \lambda)^{M+1})$ and consequently $\norm{u_t^{(M)}(y) - u_t^*(y)}_2 = \mcal{O}((1 - \lambda)^{M})$. In other words, the control sequence $\curly{u_t^{(m)}}_{t=0}^{T-1}$ converges to the optimal control at a linear rate. Note that the parameter $\lambda$ determines the convergence rate. In particular, setting a larger 
 $ \lambda$ results in faster convergence. However, when 
$\lambda$ is large, Theorem \ref{thm:convergence} requires that the regularization parameter 
$\beta_t$ must also be chosen to be sufficiently large. Recall that $\beta_t$ controls the distance between the fine-tuned model and the pre-trained model, while a bigger $\lambda$ may result in an optimal control far away from $s_t^{\rm pre}$. Therefore, choosing an appropriate $\lambda$ (or equivalently $\beta_t$) is crucial to trade-off computational efficiency and closeness to the pre-trained model.

To prove Theorem \ref{thm:convergence}, we need a few intermediate results for one-step update rule. We begin with a lemma that characterizes the regularity of control and value functions after one update. Given a function $\widehat{V}_{t+1}:\R^d \to \R$, define the following update rule:
    \begin{align}
    u_{t}^{(m+1)}(y) = s_{t}^{\rm pre}(y) + \frac{\sqrt{\alpha_{t}}\sigma_{t}^{2}}{(1 - \alpha_{t})\beta_t}\E\bracket{\nabla \widehat{V}_{t+1}\left(\frac{1}{\sqrt{\alpha_t}}\parenthesis{y + (1 - \alpha_t)u_t^{(m)}(y)} + \sigma_t W_t\right)}, \label{eq:update-u-hat}
    \end{align} 
    with initialization $u_{t}^{(0)}(y) = s_{t}^{\rm pre}(y)$. Furthermore, let $ V_{t}^{(m)} $ be the value function induced by $ u^{(m)}_{t} $, i.e., 
    \begin{align}
        V_{t}^{(m)}(y) = \mathbb{E}\left[\widehat{V}_{t+1}\left(\frac{1}{\sqrt{\alpha_t}}\parenthesis{y + (1 - \alpha_t)u_t^{(m)}(y)} + \sigma_t W_t\right) -\beta_t \frac{(1-\alpha_t)^2}{2\alpha_t \sigma_t^2}\|u_t^{(m)}(y)-s_t^{\rm pre}(y)\|^2\right]. \label{eq:Vn-hat}
    \end{align}
    Later, we will choose $\widehat{V}_{t+1} = V_{t+1}^{(m_{t+1})}$ to analyze the convergence of Algorithm \ref{algo:update-u}. With \eqref{eq:update-u-hat} and \eqref{eq:Vn-hat}, we state the lemma as follows. 

\begin{Lemma}[One-step regularity and universal upper bound]\label{lemma:regularity} Suppose Assumptions \ref{ass:smooth r} and \ref{ass:smooth s} hold.
    Let $\widehat{V}_{t+1}$ be a function that is $ L_{0,t+1}^{\widehat{V}} $-Lipschitz and $ L_{1,t+1}^{\widehat{V}} $-gradient Lipschitz. Consider $\curly{u_t^{(m)}}_{m=0}^{m_t}$ and $\curly{V_t^{(m)}}_{m=0}^{m_t}$ defined in \eqref{eq:update-u-hat} and \eqref{eq:Vn-hat}. For $t < T$, choose $\beta_t$ such that $ 1 - \frac{\sigma_t^2}{\beta_t}L_{1, t+1}^{\widehat{V}} \geq \lambda_{t} > 0  $. Then it holds for every $m = 0, \dots, m_t$ that

\begin{enumerate}[(i)]
    \item    
         $u_{t}^{(m)}$  is  $L_{0, t}^{u^{(m)}}$-Lipschitz and   $L_{1, t}^{u^{(m)}}$-gradient Lipschitz, with coefficients $L_{0,t}^{u^{(m)}} $ and $L_{1,t}^{u^{(m)}}$ satisfying
    \begin{eqnarray}\label{eq:one_step_upp1}
        L_{0,t}^{u^{(m)}} \leq L_{0,t}^{u^{*}} \;\; \text{and}\;\; L_{1,t}^{u^{(m)}} \leq L_{1, t}^{u^*}.
    \end{eqnarray}
    \item  
        $V_{t}^{(m)}$  is  $L_{0, t}^{V^{(m)}}$-Lipschitz and  $L_{1, t}^{V^{(m)}}$-gradient Lipschitz, with coefficients $L_{0, t}^{V^{(m)}}$ and $L_{1, t}^{V^{(m)}}$ satisfying
    \begin{eqnarray}
        L_{0,t}^{V^{(m)}} &\leq& \frac{1}{\sqrt{\alpha_t}}(1 + (1 - \alpha_t)L_{0, t}^{s})
L_{0, t+1}^{\widehat{V}} + \frac{L_{0, t+1}^{\widehat{V}}}{\sqrt{\alpha_{t}}}\parenthesis{1 + (1 - \alpha_{t})L_{0,t}^{u^{*}}}(1- \lambda_t)^{m+1}, \label{eq:lip-0-vm}   \\ 
L_{1,t}^{V^{(m)}} & \leq& \frac{1}{\alpha_{t}}(1 + (1 - \alpha_{t})L_{0, t}^s)(1 + (1 - \alpha_{t})L_{0, t}^{u^*})L_{1,t+1}^{\widehat{V}} + \frac{1 - \alpha_{t}}{\sqrt{\alpha_{t}}}L_{1,t}^s L_{0, t+1}^{\widehat{V}} \nonumber \\ 
& \quad& + \bigg(\frac{1 - \alpha_{t}}{\sqrt{\alpha_{t}}} L_{1, t}^{u^{*}}  + (m+2)\parenthesis{1 + (1 - \alpha_{t})L_{0,t}^{u^{*}}}^{2}\frac{\E\bracket{\norm{W_{t}}_2}}{\alpha_t\sigma_{t}}\bigg)L_{0, t+1}^{\widehat{V}}( 1- \lambda_t)^{m+1}. \label{eq:lip-1-vm}
    \end{eqnarray}

\end{enumerate}

\end{Lemma}

Lemma \ref{lemma:regularity} provides  regularity properties of the sequence of controls and value functions generated by \eqref{eq:update-u-hat} and \eqref{eq:Vn-hat}  throughout the optimization process. Specifically, the Lipschitz and gradient Lipschitz constants of $u_{t}^{(m)}$ are bounded by those of the optimal control $u_t^*$; see \eqref{eq:one_step_upp1}. Furthermore, when $\widehat{V}_{t+1} = V_{t+1}^*$,  Eqns.~\eqref{eq:lip-0-vm} and \eqref{eq:lip-1-vm} lead to 
\begin{align*}
    L_{0, t}^{V^{(m)}} = L_{0, t}^{V^*} + \mcal{O}((1 - \lambda_t)^{m+1}) \;\; \text{and}\;\; L_{1, t}^{V^{(m)}} = L_{1, t}^{V^*} + \mcal{O}((m+2)(1 - \lambda_t)^{m+1}).
\end{align*}
The residual terms diminish to zero as $m \to \infty$ and thus providing the convergence of Lipschitz constants.
Consequently, selecting $\widehat{V}_{t+1}$ as ${V}_{t+1}^{(m_{t+1})}$, when feasible, directly ensures the regularity of the control and value functions defined in Algorithm \ref{algo:update-u}. %

We outline a brief proof sketch to highlight the ideas before providing the detailed proof. The Lipschitz condition of $u_t^{(m)}$ is derived through direct calculation, while the gradient Lipschitz condition leverages the integration by parts formula \eqref{eq:magic}. Additionally, establishing the regularity of $V_t^{(m)}$ requires a careful analysis of the gradient expression, along with a tightly controlled upper bound; see \eqref{eq:grad-v-m}.

\begin{proof}
\underline{Step 1: Lipschitz conditions of $ u_{t}^{(m)} $ and $ \nabla u_{t}^{(m)}$.}
We begin with the Lipschitz condition of $ u_{t}^{(m)} $.  For any $y_1$ and $y_2$ in $\R^d$, the update rule \eqref{eq:update-u-hat} implies
\begin{eqnarray*}
\norm{u_{t}^{(m+1)}(y_{1}) - u_{t}^{(m+1)}(y_{2})}_{2} &\leq& L_{0, t}^{s}\norm{y_{1} - y_{2}}_{2}  \\
&& \quad +\frac{\sqrt{\alpha_{t}}\sigma_{t}^{2}}{(1 - \alpha_{t})\beta_{t}}L_{1, t+1}^{\widehat{V}}\parenthesis{\frac{1}{\sqrt{\alpha_{t}}}\norm{y_{1} - y_{2}}_{2} + \frac{1 - \alpha_{t}}{\sqrt{\alpha_{t}}}\norm{u_{t}^{(m)}(y_{1}) - u_{t}^{(m)}(y_{2})}_2}.
\end{eqnarray*}
Since $ 1 - \frac{\sigma_{t}^{2}}{\beta_{t}}L_{1, t+1}^{\widehat{V}} \geq \lambda_{t} $, unrolling the recursion leads to 
\begin{align*}
    L_{0, t}^{u^{(m+1)}} & \leq L_{0,t}^{s} + \frac{\sigma_{t}^{2}}{(1 - \alpha_{t})\beta_{t}}L_{1,t+1}^{\widehat{V}} + \frac{\sigma^{2}_{t}}{\beta_{t}}L_{1,t+1}^{\widehat{V}}L_{0, t}^{u^{(m)}} \\ & \leq \lambda_{t}^{-1}\parenthesis{L_{0,t}^{s} + \frac{1 - \lambda_{t}}{1 - \alpha_{t}}} = L_{0, t}^{u^{*}}.
\end{align*}
Here, we also use the fact that $ L_{0, t}^{u^{(0)}} = L_{0, t}^{s} $.  Note that the condition $ 1 - \frac{\sigma_{t}^{2}}{\beta_{t}}L_{1, t+1}^{\widehat{V}} \geq \lambda_{t} $ is important as it decouples the upper bound, in the sense that the last line does not depend on $L^{\widehat V}_{1,t+1}$ directly.

Next, since $u_t^{(0)} = s_t^{\rm pre}$ is differentiable, a simple inductive argument shows that $\nabla u_t^{(m)}$ is well-defined. Furthermore, we show that $ \nabla u_{t}^{(m)} $ is Lipschitz.  Differentiate both sides of the update rule \eqref{eq:update-u-hat},
\begin{align}
    \nabla u_{t}^{(m+1)}(y) = \nabla s_{t}^{\rm pre}(y) + \frac{\sqrt{\alpha_{t}}\sigma_{t}^{2}}{(1 - \alpha_{t})\beta_t}\cdot\frac{\partial}{\partial y} \E\bracket{\nabla \widehat{V}_{t+1}\left(\frac{1}{\sqrt{\alpha_t}}\parenthesis{y + (1 - \alpha_t)u_t^{(m)}(y)} + \sigma_t W_t\right)}. \label{eq:grad-um}
\end{align}
For any $ z \in \R^{d} $, the integration by parts formula implies that
\begin{align}\label{eq:magic}
    \frac{\partial}{\partial z}\E\bracket{\nabla \widehat{V}_{t+1}\left(z + \sigma_t W_t\right)} = \E\bracket{\nabla^{2} \widehat{V}_{t+1}\left(z + \sigma_t W_t\right)} = \E\bracket{\nabla \widehat{V}_{t+1}\left(z + \sigma_t W_t\right)\frac{W_t\top}{\sigma_t}}.
\end{align}
Substituting $ z = \frac{1}{\sqrt{\alpha_t}}\parenthesis{y + (1 - \alpha_t)u_t^{(m)}(y)} $ and applying the chain rule, we have
\begin{align}
    & \frac{\partial}{\partial y}\E\bracket{\nabla \widehat{V}_{t+1}\left(\frac{1}{\sqrt{\alpha_t}}\parenthesis{y + (1 - \alpha_t)u_t^{(m)}(y)} + \sigma_t W_t\right)} \nonumber \\ & = \E\bracket{\nabla \widehat{V}_{t+1}\left(\frac{1}{\sqrt{\alpha_t}}\parenthesis{y + (1 - \alpha_t)u_t^{(m)}(y)} + \sigma_t W_t\right)\frac{W_t\top}{\sigma_t}}\parenthesis{\frac{1}{\sqrt{\alpha_t}}\parenthesis{I_d + (1 - \alpha_t)\nabla u_t^{(m)}(y)}}. \label{eq:grad-E-grad-v}
\end{align}
Define 
\begin{align*}
    \widehat{\mathcal{G}}_{t+1}^{(m)}(y)  & \coloneqq \E\bracket{\nabla \widehat{V}_{t+1}\left(\frac{1}{\sqrt{\alpha_t}}\parenthesis{y + (1 - \alpha_t)u_t^{(m)}(y)} + \sigma_t W_t\right)}, \\
     \widehat{\mcal{W}}_{t+1}^{{(m)}}(y)  & \coloneqq  \E\bracket{\nabla \widehat{V}_{t+1}\left(\frac{1}{\sqrt{\alpha_t}}\parenthesis{y + (1 - \alpha_t)u_t^{(m)}(y)} + \sigma_t W_t\right)\frac{W_t\top}{\sigma_t}}, \\ 
    \mathcal{U}^{(m)}_t(y) & \coloneqq \frac{1}{\sqrt{\alpha_t}}\parenthesis{I_d + (1 - \alpha_t) \nabla u_t^{(m)}(y)}.     
\end{align*}
It follows from \eqref{eq:grad-E-grad-v} that $\frac{\partial}{\partial y} \widehat{\mathcal{G}}_{t+1}^{(m)}(y) =\widehat{\mcal{W}}_{t+1}^{{(m)}}(y) 
\mathcal{U}^{(m)}_t(y) $. Consequently, Eq.~\eqref{eq:grad-um} becomes
\begin{align}
    \nabla u_{t}^{(m+1)}(y)  & = \nabla s_{t}^{\rm pre}(y)  + \frac{\sqrt{\alpha_{t}}\sigma_{t}^{2}}{(1 - \alpha_{t})\beta_t} \widehat{\mcal{W}}_{t+1}^{{(m)}}(y) \mathcal{U}^{(m)}_t(y). \label{eq:grad-u-m-w}
\end{align}
Since $\nabla \widehat{V}_{t+1}$ and $u_t^{(m)}$ are both Lipschitz, we 
know that $ \nabla u_{t}^{(m+1)} $ is Lipschitz as long as $ \nabla u_{t}^{(m)} $ is Lipschitz. Specifically, for any $y_1$ and $y_2$ in $\R^d$, we have
\begin{align*}
\nabla u_t^{(m+1)} (y_1) - \nabla u_t^{(m+1)} (y_2)  & = \nabla s_{t}^{\rm pre}(y_{1}) - \nabla s_{t}^{\rm pre}(y_{2})  +  \frac{\sqrt{\alpha_{t}}\sigma_{t}^{2}}{(1 - \alpha_{t})\beta_t} \parenthesis{\widehat{\mcal{W}}_{t+1}^{{(m)}}(y_1) \mathcal{U}^{(m)}_t(y_1) - \widehat{\mcal{W}}_{t+1}^{{(m)}}(y_2) \mathcal{U}^{(m)}_t(y_2)}.
\end{align*}
Note that for any $ y \in \R^d $ we have
\begin{align}
    & \norm{\widehat{\mcal{W}}_{t+1}^{{(m)}}(y)}_{2} \leq L_{1, t+1}^{\widehat{V}}, \label{eq:upper-bound-magic} \\ 
    & \norm{ \mathcal{U}^{(m)}_t(y)}_2 \leq \frac{1}{\sqrt{\alpha_{t}}}\parenthesis{1 + (1 - \alpha_{t})L_{0,t}^{u^{*}}}, \label{eq:upper-bound-nabla-u-n}
\end{align}
where the \eqref{eq:upper-bound-magic} holds by applying the identity \eqref{eq:magic}, and \eqref{eq:upper-bound-nabla-u-n} is a consequence of the Lipschitz condition of $ u_{t}^{(m)} $.
Moreover, for any $ y_{1} $ and $ y_{2} $ in $\R^d$, the Lipschitz conditions of $ \nabla \widehat{V}_{t+1} $ and $ u_{t}^{(m)} $ imply that
\begin{align}
    \norm{\widehat{\mcal{W}}_{t+1}^{{(m)}}(y_1) - \widehat{\mcal{W}}_{t+1}^{{(m)}}(y_2)}_2 \nonumber & \leq \frac{\E\bracket{\norm{W_{t}}_2}}{\sigma_{t}}L_{1, t+1}^{\widehat{V}}\frac{1}{\sqrt{\alpha_{t}}}\parenthesis{\norm{y_{1} - y_{2}}_2 + (1 - \alpha_{t})L_{0, t}^{u^{(m)}}\norm{y_{1} - y_{2}}_2} \nonumber \\ & = \frac{\E\bracket{\norm{W_{t}}_2}}{\sqrt{\alpha_{t}}\sigma_{t}}L_{1, t+1}^{\widehat{V}}\parenthesis{1 + (1 - \alpha_{t})L_{0, t}^{u^{*}}}\norm{y_{1} - y_{2}}_2, \label{eq:Lip-nabla-V-w}
\end{align}
and the Lipschitz condition of $\nabla u_t^{(m)}$ leads to
\begin{align}
     \norm{\mathcal{U}^{(m)}_t(y_1) - \mathcal{U}^{(m)}_t(y_2)}_2 & = \frac{1 - \alpha_{t}}{\sqrt{\alpha_{t}}}\norm{\nabla u_{t}^{(m)}(y_{1}) - \nabla u_{t}^{(m)}(y_{2})}_2 \nonumber \\ & \leq \frac{1 - \alpha_{t}}{\sqrt{\alpha_{t}}} L_{1, t}^{u^{(m)}}\norm{y_{1} - y_{2}}_2. \label{eq:Lip-nabla-u-n}
\end{align}
Combine \eqref{eq:grad-u-m-w}--\eqref{eq:Lip-nabla-u-n} together to have
\begin{align*}
    \norm{\nabla u_{t}^{(m+1)}(y_{1}) - \nabla u_{t}^{(m+1)}(y_{2})}_2 & \leq L_{1, t}^{s}\norm{y_{1} - y_{2}}_2 + \frac{\sqrt{\alpha_{t}}\sigma_{t}^{2}}{(1 - \alpha_{t})\beta_t}\bigg(L_{1, t+1}^{\widehat{V}}\frac{1 - \alpha_{t}}{\sqrt{\alpha_{t}}} L_{1, t}^{u^{(m)}} \\ & \qquad + \frac{1}{\alpha_{t}}\parenthesis{1 + (1 - \alpha_{t})L_{0,t}^{u^{*}}}^{2}\frac{\E\bracket{\norm{W_{t}}_2}}{\sigma_{t}}L_{1, t+1}^{\widehat{V}}\bigg)\norm{y_{1} - y_{2}}_2,
\end{align*}
Since $ \frac{\sigma_{t}^{2}}{\beta_{t}}L_{1, t+1}^{\widehat{V}} \leq 1 - \lambda_{t} $, we deduce that
\begin{align*}
L_{1, t}^{u^{(m+1)}} & \leq L_{1, t}^{s} + \bigg( L_{1, t}^{u^{(m)}} + \frac{\E\bracket{\norm{W_{t}}_2}}{(1 - \alpha_t)\sqrt{\alpha_{t}}\sigma_{t}}\parenthesis{1 + (1 - \alpha_{t})L_{0,t}^{u^{*}}}^{2}\bigg)(1 - \lambda_t).
\end{align*}
Equivalently, we have
\begin{align*}
    L_{1, t}^{u^{(m)}} \leq \lambda_{t}^{-1}\parenthesis{L_{1, t}^{s} +  \frac{\E\bracket{\norm{W_{t}}_2}(1 - \lambda_t)}{(1 - \alpha_t)\sqrt{\alpha_{t}}\sigma_{t}}\parenthesis{1 + (1 - \alpha_{t})L_{0,t}^{u^{*}}}^{2}} = L_{1, t}^{u^*}.
\end{align*}

\noindent\underline{Step 2: Lipschitz conditions of $V_t^{(m)}$ and $\nabla V_t^{(m)}$.} With the Lipschitz conditions of $ u_{t}^{(m)} $ and $ \nabla u_{t}^{(m)} $, we now turn to establish the regularity of $ V_{t}^{(m)} $. Note that the expressions in \eqref{eq:update-u-hat}, \eqref{eq:Vn-hat} and \eqref{eq:grad-u-m-w} imply that
\begin{align}
    & \nabla V_{t}^{(m)}(y) \nonumber \\ &=  \frac{1}{\sqrt{\alpha_t}}\parenthesis{I_d 
  + (1 - \alpha_t) \nabla u_t^{(m)}(y)}^\top\widehat{\mathcal{G}}_{t+1}^{(m)}(y)- \beta_t \frac{(1 - \alpha_{t})^{2}}{\alpha_{t}\sigma_{t}^{2}}(\nabla u_t^{(m)}(y) - \nabla s^{\rm pre}_{t}(y))^\top\left(  \frac{\sqrt{\alpha_{t}}\sigma_{t}^{2}}{(1 - \alpha_{t})\beta_t}\widehat{\mathcal{G}}_{t+1}^{(m-1)}(y)\right) \nonumber \\
  &=\frac{1}{\sqrt{\alpha_t}}\parenthesis{I_d
  + (1 - \alpha_t) \nabla s_t^{\rm pre}(y)}^\top\widehat{\mathcal{G}}_{t+1}^{(m)}(y) + \frac{(1 - \alpha_{t})}{\sqrt{\alpha_{t}}} \parenthesis{\nabla u_t^{(m)}(y) - \nabla s_t^{\rm pre}(y)}^\top \Big(\widehat{\mathcal{G}}_{t+1}^{(m)}(y)-\widehat{\mathcal{G}}_{t+1}^{(m-1)}(y)\Big).
  \label{eq:grad-v-m}
\end{align}
We will use the above expression \eqref{eq:grad-v-m} to derive the  Lipschitz and gradient Lipschitz conditions of $ V_{t}^{(m)} $. To proceed, we first get a few useful estimates.
Recall the notation $\mathcal{S}_t(y) = \frac{1}{\sqrt{\alpha_t}}\parenthesis{I_d 
+ (1 - \alpha_t) \nabla s_t^{\rm pre}(y)}$. We first show that $\mcal{S}_t(y)^\top\widehat{\mathcal{G}}_{t+1}^{(m)}(y)$ is bounded and Lipschitz. Note that for any $y\in \R^d$,
\begin{align}
\norm{\mcal{S}_t(y)^\top\widehat{\mathcal{G}}_{t+1}^{(m)}(y)}_{2} \leq \norm{\mcal{S}_t(y)}_2\norm{\widehat{\mathcal{G}}_{t+1}^{(m)}(y)}_{2}  \leq  \frac{1}{\sqrt{\alpha_t}}(1 + (1 - \alpha_t)L_{0, t}^{s})
L_{0, t+1}^{\widehat{V}}. \label{eq:sg-bound}
\end{align}
Also, for any $y_1$ and $y_2$ in $\R^d$, we have
\begin{align}
\norm{\mcal{S}_t(y_{1})^{\top}\widehat{\mathcal{G}}_{t+1}^{(m)}(y_{1}) - \mcal{S}_t(y_{2})^{\top}\widehat{\mathcal{G}}_{t+1}^{(m)}(y_{2})}_2  & \leq \norm{\mcal{S}_t(y_{1})^{\top}\parenthesis{\widehat{\mathcal{G}}_{t+1}^{(m)}(y_{1}) - \widehat{\mathcal{G}}_{t+1}^{(m)}(y_{2})}}_2 + \norm{\parenthesis{\mcal{S}_t(y_{1}) - \mcal{S}_t(y_{2})}^{\top}\widehat{\mathcal{G}}_{t+1}^{(m)}(y_{2})}_{2}\nonumber \\ & \leq \norm{\mcal{S}_t(y_{1})}_{2}\norm{\widehat{\mathcal{G}}_{t+1}^{(m)}(y_{1}) - \widehat{\mathcal{G}}_{t+1}^{(m)}(y_{2})}_{2} + \norm{\mcal{S}_t(y_{1}) - \mcal{S}_t(y_{2})}_{2}\norm{\widehat{\mathcal{G}}_{t+1}^{(m)}(y_{2})}_{2} \nonumber \\ & \leq \frac{1}{\alpha_{t}}(1 + (1 - \alpha_{t})L_{0, t}^s)(1 + (1 - \alpha_{t})L_{0, t}^{u^*})L_{1,t+1}^{\widehat{V}}\norm{y_{1} - y_{2}}_{2} \nonumber \\ & \quad + \frac{1 - \alpha_{t}}{\sqrt{\alpha_{t}}}L_{1,t}^s L_{0, t+1}^{\widehat{V}}\norm{y_{1} - y_{2}}_{2}, \label{eq:lip-sg-m}
\end{align}
where we apply the Lipschitz conditions of $\widehat{V}_{t+1}$, $ \nabla \widehat{V}_{t+1}$, and $s_t^{\rm pre}$ in the assumptions and the Lipschitz condition of $u_t^{(m)}$ in Step 1. 
It remains to show the second term in \eqref{eq:grad-v-m} is bounded and Lipschitz uniformly over $m \geq 0$. Since $\nabla \widehat{V}_{t+1}$ is $L_{1, t+1}^{\widehat{V}}$-Lipschitz, we deduce that 
\begin{align}
\norm{\widehat{\mathcal{G}}_{t+1}^{(m)}(y)-\widehat{\mathcal{G}}_{t+1}^{(m-1)}(y)}_2 & \leq L_{1, t+1}^{\widehat{V}} \frac{1 - \alpha_{t}}{\sqrt{\alpha_{t}}}\norm{u_{t}^{(m)}(y) - u_{t}^{(m-1)}(y)}_2 \nonumber \\ & \leq L_{1, t+1}^{\widehat{V}}\frac{\sigma_{t}^{2}}{\beta_{t}}\norm{\widehat{\mathcal{G}}_{t+1}^{(m-1)}(y)-\widehat{\mathcal{G}}_{t+1}^{(m-2)}(y)}_2,\label{eq:G-hat-n-iter}
\end{align}
where we apply the update rule \eqref{eq:update-u-hat} in the last equation.
As $\frac{\sigma_{t}^{2}}{\beta_{t}}\,L_{1, t+1}^{\widehat{V}} \leq 1 - \lambda_t$, unrolling the recursion \eqref{eq:G-hat-n-iter} leads to
\begin{align}
   \norm{\widehat{\mathcal{G}}_{t+1}^{(m)}(y)-\widehat{\mathcal{G}}_{t+1}^{(m-1)}(y)}_2 & \leq ( 1- \lambda_t)^{m-1} \norm{\widehat{\mathcal{G}}_{t+1}^{(1)}(y)-\widehat{\mathcal{G}}_{t+1}^{(0)}(y)}_2 \nonumber \\ & \leq ( 1- \lambda_t)^{m-1}  L_{1, t+1}^{\widehat{V}} \frac{1 - \alpha_{t}}{\sqrt{\alpha_{t}}}\norm{u_{t}^{(1)} - u_{t}^{(0)}}_2 \nonumber \\ & \leq ( 1- \lambda_t)^{m-1}L_{1, t+1}^{\widehat{V}} \frac{\sigma_t^2}{\beta_t}L_{0, t+1}^{\widehat{V}} \leq ( 1- \lambda_t)^{m}L_{0, t+1}^{\widehat{V}}.\label{eq:G-hat-G-n-bound}
\end{align}
Moreover, for any $ y \in \R^{d} $, we have
\begin{align}
    \norm{\frac{\partial}{\partial y} \parenthesis{\widehat{\mathcal{G}}_{t+1}^{(m)}(y)-\widehat{\mathcal{G}}_{t+1}^{(m-1)}(y)}}_2 & = \norm{\widehat{\mcal{W}}_{t+1}^{{(m)}}(y)\mathcal{U}^{(m)}_t(y) - \widehat{\mcal{W}}_{t+1}^{{(m-1)}}(y)\mathcal{U}^{(m-1)}_t(y)}_2 \nonumber \\ & \leq \norm{\widehat{\mcal{W}}_{t+1}^{{(m)}}(y) - \widehat{\mcal{W}}_{t+1}^{{(m-1)}}(y)}_2\norm{\mathcal{U}^{(m)}_t(y)}_2 + \norm{\widehat{\mcal{W}}_{t+1}^{{(m-1)}}(y)}_2\norm{\mathcal{U}^{(m)}_t(y) - \mathcal{U}^{(m-1)}_t(y)}_2. \label{eq:grad-Gm-diff}
\end{align}
The gradient expression \eqref{eq:grad-u-m-w} implies
\begin{align}
    \norm{\mathcal{U}^{(m)}_t(y) - \mathcal{U}^{(m-1)}_t(y)}_2 \leq \frac{1 - \alpha_t}{\sqrt{\alpha_t}}\norm{\nabla u_t^{(m)}(y) -\nabla u_t^{(m-1)}(y)}_2 \leq \frac{\sigma_t^2}{\beta_t}\norm{\frac{\partial}{\partial y} \parenthesis{\widehat{\mathcal{G}}_{t+1}^{(m-1)}(y)-\widehat{\mathcal{G}}_{t+1}^{(m-2)}(y)}}_2, \label{eq:Um-diff}
\end{align}
and \eqref{eq:G-hat-n-iter}--\eqref{eq:G-hat-G-n-bound} lead to the fact that
\begin{align}
    \norm{\widehat{\mcal{W}}_{t+1}^{{(m)}}(y) - \widehat{\mcal{W}}_{t+1}^{{(m-1)}}(y)}_2 & \leq \frac{\E\bracket{\norm{W_{t}}_2}}{\sigma_{t}}L_{1, t+1}^{\widehat{V}}\frac{1 - \alpha_{t}}{\sqrt{\alpha_{t}}}\norm{u_t^{(m)}(y) - u_t^{(m-1)}(y)}_2  \leq \frac{\E\bracket{\norm{W_{t}}_2}}{\sigma_{t}}L_{0, t+1}^{\widehat{V}}(1 - \lambda_t)^{m}. \label{eq:Wm-diff}
\end{align}
With \eqref{eq:upper-bound-magic}, \eqref{eq:upper-bound-nabla-u-n}, \eqref{eq:Um-diff} and \eqref{eq:Wm-diff}, we can bound \eqref{eq:grad-Gm-diff} with
\begin{align}
    \norm{\frac{\partial}{\partial y} \parenthesis{\widehat{\mathcal{G}}_{t+1}^{(m)}(y)-\widehat{\mathcal{G}}_{t+1}^{(m-1)}(y)}}_2 & \leq \frac{L_{0, t+1}^{\widehat{V}}\E\bracket{\norm{W_{t}}_2}}{\sqrt{\alpha_{t}}\sigma_{t}}\parenthesis{1 + (1 - \alpha_{t})L_{0,t}^{u^{*}}}(1 - \lambda_t)^m \nonumber \\ & \qquad + \frac{\sigma_t^2}{\beta_t}L_{1, t+1}^{\widehat{V}} \norm{\frac{\partial}{\partial y} \parenthesis{\widehat{\mathcal{G}}_{t+1}^{(m-1)}(y)-\widehat{\mathcal{G}}_{t+1}^{(m-2)}(y)}}_2. \label{eq:grad-Gm-diff-recursion}
\end{align}
Since $\frac{\sigma_{t}^{2}}{\beta_{t}}\,L_{1, t+1}^{\widehat{V}} \leq 1 - \lambda_t$, unrolling the recursion \eqref{eq:grad-Gm-diff-recursion} to have
\begin{align}
    \norm{\frac{\partial}{\partial y} \parenthesis{\widehat{\mathcal{G}}_{t+1}^{(m)}(y)-\widehat{\mathcal{G}}_{t+1}^{(m-1)}(y)}}_2 & \leq \frac{L_{0, t+1}^{\widehat{V}}\E\bracket{\norm{W_{t}}_2}}{\sqrt{\alpha_{t}}\sigma_{t}}\parenthesis{1 + (1 - \alpha_{t})L_{0,t}^{u^{*}}}m(1 - \lambda_t)^m \nonumber \\ & \qquad + (1 - \lambda_t)^{(m)}\norm{\frac{\partial}{\partial y} \widehat{\mathcal{G}}_{t+1}^{(0)}(y)}_2 \nonumber \\ & \leq \frac{L_{0, t+1}^{\widehat{V}}\E\bracket{\norm{W_{t}}_2}}{\sqrt{\alpha_{t}}\sigma_{t}}\parenthesis{1 + (1 - \alpha_{t})L_{0,t}^{u^{*}}}(m+1)(1 - \lambda_t)^m, \label{eq:grad-g-diff-m}
\end{align}
where we utilize the fact that
\begin{align*}
\norm{\frac{\partial}{\partial y} \widehat{\mathcal{G}}_{t+1}^{(0)}(y)}_2  \leq \norm{\widehat{\mcal{W}}_{t+1}^{{(0)}}(y)}_2\norm{\mathcal{U}^{(0)}_t(y)}_2 \leq \frac{L_{0, t+1}^{\widehat{V}}\E\bracket{\norm{W_{t}}_2}}{\sqrt{\alpha_{t}}\sigma_{t}}\parenthesis{1 + (1 - \alpha_{t})L_{0,t}^{u^{*}}}.
\end{align*}
Hence, for any $y_1$ and $y_2$ in $\R^d$, \eqref{eq:upper-bound-magic}--\eqref{eq:Lip-nabla-u-n}, \eqref{eq:G-hat-G-n-bound} and \eqref{eq:grad-g-diff-m} imply that
\begin{align}
    & \norm{ \parenthesis{\widehat{\mcal{W}}_{t+1}^{{(m)}}(y_1) \mathcal{U}^{(m)}_t(y_1)}^\top \Big(\widehat{\mathcal{G}}_{t+1}^{(m)}(y_1)-\widehat{\mathcal{G}}_{t+1}^{(m-1)}(y_1)\Big) - \parenthesis{\widehat{\mcal{W}}_{t+1}^{{(m)}}(y_2) \mathcal{U}^{(m)}_t(y_2)}^\top \Big(\widehat{\mathcal{G}}_{t+1}^{(m)}(y_2)-\widehat{\mathcal{G}}_{t+1}^{(m-1)}(y_2)\Big)}_2 \nonumber \\ & \leq \norm{ \widehat{\mcal{W}}_{t+1}^{{(m)}}(y_1) \mathcal{U}^{(m)}_t(y_1)  - \widehat{\mcal{W}}_{t+1}^{{(m)}}(y_2) \mathcal{U}^{(m)}_t(y_2)}_2\norm{\widehat{\mathcal{G}}_{t+1}^{(m)}(y_1)-\widehat{\mathcal{G}}_{t+1}^{(m-1)}(y_1)}_2 \nonumber \\ & \quad + \norm{\widehat{\mcal{W}}_{t+1}^{{(m)}}(y_2) \mathcal{U}^{(m)}_t(y_2)}_2\sup_{y \in \R^d} \norm{\frac{\partial}{\partial y} \parenthesis{\widehat{\mathcal{G}}_{t+1}^{(m)}(y)-\widehat{\mathcal{G}}_{t+1}^{(m-1)}(y)}}_2 \norm{y_1 - y_2}_2 \nonumber \\ & \leq \bigg(\frac{1 - \alpha_{t}}{\sqrt{\alpha_{t}}} L_{1, t}^{u^{(m)}}  + \parenthesis{1 + (1 - \alpha_{t})L_{0,t}^{u^{*}}}^{2}\frac{\E\bracket{\norm{W_{t}}_2}}{\alpha_t\sigma_{t}}\bigg)L_{1, t+1}^{\widehat{V}}L_{0, t+1}^{\widehat{V}}( 1- \lambda_t)^{m}\norm{y_1 - y_2}_2 \nonumber \\ & \quad + \parenthesis{1 + (1 - \alpha_{t})L_{0,t}^{u^{*}}}^2\frac{\E\bracket{\norm{W_{t}}_2}}{\alpha_{t}\sigma_{t}}L_{1, t+1}^{\widehat{V}}L_{0, t+1}^{\widehat{V}}(m+1)(1 - \lambda_t)^m \norm{y_1 - y_2}_2 \nonumber \\ & = \bigg(\frac{1 - \alpha_{t}}{\sqrt{\alpha_{t}}} L_{1, t}^{u^{*}}  + (m+2)\parenthesis{1 + (1 - \alpha_{t})L_{0,t}^{u^{*}}}^{2}\frac{\E\bracket{\norm{W_{t}}_2}}{\alpha_t\sigma_{t}}\bigg)L_{1, t+1}^{\widehat{V}}L_{0, t+1}^{\widehat{V}}( 1- \lambda_t)^{m}\norm{y_1 - y_2}_2. \label{eq:lip-wu-g-diff-m}
\end{align}

With there estimates in hand, we are ready to prove the Lipschitz and gradient Lipschitz conditions of $V_t^{(m)}$. Note that
\begin{align*}
    \frac{(1 - \alpha_{t})}{\sqrt{\alpha_{t}}} \parenthesis{\nabla u_t^{(m)}(y) - \nabla s_t^{\rm pre}(y)} = \frac{\sigma_t^2}{\beta_t} \widehat{\mcal{W}}_{t+1}^{{(m)}}(y) \mathcal{U}^{(m)}_t(y).
\end{align*}
Consequently, for any $y \in \R^d$, we deduce from  \eqref{eq:upper-bound-magic}, \eqref{eq:upper-bound-nabla-u-n}, \eqref{eq:sg-bound} and \eqref{eq:G-hat-G-n-bound} that
\begin{align*}
    \norm{\nabla V_t^{(m)}(y)}_2 & \leq \norm{\mcal{S}_t(y)^\top\widehat{\mathcal{G}}_{t+1}^{(m)}(y)}_{2} + \frac{\sigma_t^2}{\beta_t} \norm{\widehat{\mcal{W}}_{t+1}^{{(m)}}(y)}_2\norm{ \mathcal{U}^{(m)}_t(y)}_2 \norm{\widehat{\mathcal{G}}_{t+1}^{(m)}(y)-\widehat{\mathcal{G}}_{t+1}^{(m-1)}(y)}_2 \\ & \leq \frac{1}{\sqrt{\alpha_t}}(1 + (1 - \alpha_t)L_{0, t}^{s})
L_{0, t+1}^{\widehat{V}} + \frac{\sigma_t^2}{\beta_t}L_{1, t+1}^{\widehat{V}}\parenthesis{\frac{1}{\sqrt{\alpha_{t}}}\parenthesis{1 + (1 - \alpha_{t})L_{0,t}^{u^{*}}}}(1- \lambda_t)^{m}L_{0, t+1}^{\widehat{V}} \\ & \leq \frac{1}{\sqrt{\alpha_t}}(1 + (1 - \alpha_t)L_{0, t}^{s})
L_{0, t+1}^{\widehat{V}} + \frac{L_{0, t+1}^{\widehat{V}}}{\sqrt{\alpha_{t}}}\parenthesis{1 + (1 - \alpha_{t})L_{0,t}^{u^{*}}}(1- \lambda_t)^{m+1}. 
\end{align*}
Finally, for any $y_1$ and $y_2$ in $\R^d$, it follows from \eqref{eq:lip-sg-m} and \eqref{eq:lip-wu-g-diff-m} that
\begin{align*}
     & \norm{\nabla V_t^{(m)}(y_1) - \nabla V_t^{(m)} (y_2)}_2 \\ & \leq \norm{\mcal{S}_t(y_{1})^{\top}\widehat{\mathcal{G}}_{t+1}^{(m)}(y_{1}) - \mcal{S}_t(y_{2})^{\top}\widehat{\mathcal{G}}_{t+1}^{(m)}(y_{2})}_2 \\ & \quad + \frac{\sigma_t^2}{\beta_t}\Big\| \parenthesis{\widehat{\mcal{W}}_{t+1}^{{(m)}}(y_1) \mathcal{U}^{(m)}_t(y_1)}^\top \Big(\widehat{\mathcal{G}}_{t+1}^{(m)}(y_1)-\widehat{\mathcal{G}}_{t+1}^{(m-1)}(y_1)\Big) \\ & \qquad\qquad\qquad\qquad - \parenthesis{\widehat{\mcal{W}}_{t+1}^{{(m)}}(y_2) \mathcal{U}^{(m)}_t(y_2)}^\top \Big(\widehat{\mathcal{G}}_{t+1}^{(m)}(y_2)-\widehat{\mathcal{G}}_{t+1}^{(m-1)}(y_2)\Big)\Big\|_2 \\ & \leq \parenthesis{\frac{1}{\alpha_{t}}(1 + (1 - \alpha_{t})L_{0, t}^s)(1 + (1 - \alpha_{t})L_{0, t}^{u^*})L_{1,t+1}^{\widehat{V}} + \frac{1 - \alpha_{t}}{\sqrt{\alpha_{t}}}L_{1,t}^s L_{0, t+1}^{\widehat{V}}}\norm{y_{1} - y_{2}}_{2} \\ & \quad + \frac{\sigma_t^2}{\beta_t}L_{1, t+1}^{\widehat{V}}\bigg(\frac{1 - \alpha_{t}}{\sqrt{\alpha_{t}}} L_{1, t}^{u^{*}}  + (m+2)\parenthesis{1 + (1 - \alpha_{t})L_{0,t}^{u^{*}}}^{2}\frac{\E\bracket{\norm{W_{t}}_2}}{\alpha_t\sigma_{t}}\bigg)L_{0, t+1}^{\widehat{V}}( 1- \lambda_t)^{m}\norm{y_1 - y_2}_2 \\ & \leq \bigg(\frac{1}{\alpha_{t}}(1 + (1 - \alpha_{t})L_{0, t}^s)(1 + (1 - \alpha_{t})L_{0, t}^{u^*})L_{1,t+1}^{\widehat{V}} + \frac{1 - \alpha_{t}}{\sqrt{\alpha_{t}}}L_{1,t}^s L_{0, t+1}^{\widehat{V}} \\ & \quad + \bigg(\frac{1 - \alpha_{t}}{\sqrt{\alpha_{t}}} L_{1, t}^{u^{*}}  + (m+2)\parenthesis{1 + (1 - \alpha_{t})L_{0,t}^{u^{*}}}^{2}\frac{\E\bracket{\norm{W_{t}}_2}}{\alpha_t\sigma_{t}}\bigg)L_{0, t+1}^{\widehat{V}}( 1- \lambda_t)^{m+1}\bigg)\norm{y_1 - y_2}_2, 
\end{align*}
which finishes the proof. 
\end{proof}

Our next result characterizes the error of the control sequence obtained from the one-step update rule \eqref{eq:update-u-hat}.
\begin{Lemma}[One-step error analysis on controls]\label{lemma:un-u*-error}
    Assume the same assumptions as in Lemma \ref{lemma:regularity}.
    Moreover, suppose there is a constant $E_{t+1}$ such that for all $y$ that
    \begin{align*}
        \norm{\parenthesis{\nabla \widehat{V}_{t+1} - \nabla V^{*}_{t+1}}(y)}_2 \leq E_{t+1}.
    \end{align*}
    Let $u_t^{(m)}$ and $V_t^{(m)}$ be defined as in \eqref{eq:update-u-hat} and \eqref{eq:Vn-hat}, respectively. Then it holds that for all $y\in \R^d$ that
    \begin{align*}
    \norm{u_{t}^{(m)}(y) - u_{t}^{*}(y)}_{2} &  \leq \parenthesis{(1 - \lambda_t)^n  L_{0,t+1}^{V^*} + \lambda_t^{-1}E_{t+1}} \frac{\sqrt{\alpha_{t}}(1 - \lambda_t)}{(1 - \alpha_{t})L_{1, t+1}^{V^*}}.
\end{align*}
\end{Lemma}

\begin{proof}
Define the Bellman optimality operator $\mcal{T}_t$ at time $t$ as
\begin{align*}
    (\mcal{T}_t u )(y) \coloneqq s_{t}^{\rm pre}(y) + \frac{\sqrt{\alpha_{t}}\sigma_{t}^{2}}{(1 - \alpha_{t})\beta_t}\E\bracket{\nabla V_{t+1}^{*}\left(\frac{1}{\sqrt{\alpha_t}}\parenthesis{y + (1 - \alpha_t)u(y)} + \sigma_t W_t\right)}.
\end{align*}
Moreover, define the approximate Bellman operator $\widehat{\mcal{T}}_t$ as
\begin{align*}
    (\widehat{\mcal{T}}_t u )(y) \coloneqq s_{t}^{\rm pre}(y) + \frac{\sqrt{\alpha_{t}}\sigma_{t}^{2}}{(1 - \alpha_{t})\beta_t}\E\bracket{\nabla \widehat{V}_{t+1}\left(\frac{1}{\sqrt{\alpha_t}}\parenthesis{y + (1 - \alpha_t)u(y)} + \sigma_t W_t\right)}.
\end{align*}
The update rule \eqref{eq:update-u-hat} implies $ u_{t}^{(m+1)} = \widehat{\mcal{T}}_t u_{t}^{(m)} $. Together with the optimality condition $u_t^* = \mcal{T}_t u_t^*$, we obtain that
\begin{align}
\norm{u_{t}^{(m+1)}(y) - u_{t}^{*}(y)}_{2} & = \norm{(\widehat{\mcal{T}}_t u_{t}^{(m)})(y) - (\mcal{T}_{t}u_{t}^{*})(y)}_{2} \nonumber \\ & \leq \norm{(\widehat{\mcal{T}}_t u_{t}^{(m)})(y) - (\mcal{T}_{t}u_{t}^{(m)})(y)}_{2} + \norm{(\mcal{T}_t u_{t}^{(m)})(y) - (\mcal{T}_{t}u_{t}^{*})(y)}_{2}.\label{eq:u^n-u*}
\end{align}
Note that $\mcal{T}_t$ is a contraction operator. Indeed for any $u, v:\mathbb{R}^d\rightarrow\mathbb{R}^d$ we have 
\begin{align}
    \norm{(\mcal{T}_t u)(y) - (\mcal{T}_t v)(y)}_2 & \leq  \frac{\sqrt{\alpha_{t}}\sigma_{t}^{2}}{(1 - \alpha_{t})\beta_t} \E\Bigg[ \Bigg\|\nabla V_{t+1}^{*}\left(\frac{1}{\sqrt{\alpha_t}}\parenthesis{y + (1 - \alpha_t)u(y)} + \sigma_t W_t\right) \nonumber \\ & \quad \qquad \qquad \qquad  - \nabla V_{t+1}^{*}\left(\frac{1}{\sqrt{\alpha_t}}\parenthesis{y + (1 - \alpha_t)v(y)} + \sigma_t W_t\right)\Bigg\|_2 \Bigg] \nonumber\\ & \leq \frac{\sigma_t^2}{\beta_t}L_{1, t+1}^{V^*}\norm{u(y) - v(y)}_2 \nonumber \\ & \leq (1 - \lambda_t)\norm{u(y) - v(y)}_2, \label{eq:contraction}
\end{align}
where we apply the Lipschitz condition of $\nabla V^*_{t+1}$ and the choice of $\beta_t$. Consequently, the second term in \eqref{eq:u^n-u*} is bounded with
\begin{align}
    \norm{(\mcal{T}_t u_{t}^{(m)})(y) - (\mcal{T}_{t}u_{t}^{*})(y)}_{2} \leq (1 - \lambda_{t})\norm{u_{t}^{(m)}(y) - u_{t}^{*}(y)}_{2}.\label{eq:Tun}
\end{align}

Next, define
\begin{align*}
    \mathcal{G}_{t+1}^{(m)}(y) = \E\bracket{\nabla V_{t+1}^{*}\left(\frac{1}{\sqrt{\alpha_t}}\parenthesis{y + (1 - \alpha_t)u_t^{(m)}(y)} + \sigma_t W_t\right)}.
\end{align*}
and recall that
\begin{align*}
    \widehat{\mathcal{G}}_{t+1}^{(m)}(y)  \coloneqq \E\bracket{\nabla \widehat{V}_{t+1}\left(\frac{1}{\sqrt{\alpha_t}}\parenthesis{y + (1 - \alpha_t)u_t^{(m)}(y)} + \sigma_t W_t\right)}.
\end{align*}
For the first term in \eqref{eq:u^n-u*}, note that 
\begin{align}
    \norm{(\widehat{\mcal{T}}_t u_t^{(m)})(y) - (\mcal{T}_{t}u_t^{(m)})(y)}_{2} & = \frac{\sqrt{\alpha_{t}}\sigma_{t}^{2}}{(1 - \alpha_{t})\beta_t}\norm{\widehat{\mathcal{G}}_{t+1}^{(m)}(y) - \mathcal{G}_{t+1}^{(m)}(y)}_{2} \nonumber \\ & \leq \frac{\sqrt{\alpha_{t}}\sigma_{t}^{2}}{(1 - \alpha_{t})\beta_t}\E\bracket{\norm{\parenthesis{\nabla \widehat{V}_{t+1} - \nabla V^{*}_{t+1}}(y')}_{2}} \leq \frac{\sqrt{\alpha_{t}}\sigma_{t}^{2}}{(1 - \alpha_{t})\beta_t}E_{t+1},  \label{eq:grad-v-vhat}
\end{align}
where $ y' \sim \mcal{N}\parenthesis{\frac{1}{\sqrt{\alpha_t}}\parenthesis{y + (1 - \alpha_t)u_t^{(m)})(y)}, \sigma_t^2 I_d} $. 

With \eqref{eq:Tun} and \eqref{eq:grad-v-vhat}, we can further bound \eqref{eq:u^n-u*} with
\begin{align}
    \norm{u_{t}^{(m+1)}(y) - u_{t}^{*}(y)}_{2} & \leq \frac{\sqrt{\alpha_{t}}\sigma_{t}^{2}}{(1 - \alpha_{t})\beta_t}E_{t+1} + (1 - \lambda_{t})\norm{u_{t}^{(m)}(y) - u_{t}^{*}(y)}_{2}. \label{eq:recursion-um-u*}
\end{align}
Unrolling the recursion \eqref{eq:recursion-um-u*}, we have
\begin{align*}
    \norm{u_{t}^{(m)}(y) - u_{t}^{*}(y)}_{2} & \leq (1 - \lambda_t)^m\norm{u_{t}^{(0)}(y) - u_{t}^{*}(y)}_{2} + \lambda_t^{-1}\frac{\sqrt{\alpha_{t}}\sigma_{t}^{2}}{(1 - \alpha_{t})\beta_t}E_{t+1} \\ & \leq \parenthesis{(1 - \lambda_t)^m L_{0,t+1}^{V^*} + \lambda_t^{-1}E_{t+1}} \frac{\sqrt{\alpha_{t}}\sigma_{t}^{2}}{(1 - \alpha_{t})\beta_t} \\ & \leq \parenthesis{(1 - \lambda_t)^m L_{0,t+1}^{V^*} + \lambda_t^{-1}E_{t+1}} \frac{\sqrt{\alpha_{t}}(1 - \lambda_t)}{(1 - \alpha_{t})L_{1, t+1}^{V^*}},
\end{align*}
where we apply the condition that $1 - \frac{\sigma_t^2}{\beta_t}L_{1, t+1}^{V^*} \geq \lambda_t$. This finishes the proof.
\end{proof}

The last lemma in this section is on the error analysis of $\nabla V_t^{(m)}$. 

\begin{Lemma}[One-step error analysis on gradient of the value function]\label{lemma:nabla-Vn-V*-error}
    Assume the same assumptions as in Lemma \ref{lemma:un-u*-error}. Then it holds that for all $y\in \R^d$,
    \begin{align}
    \norm{\nabla V_{t}^{(m)}(y) -  \nabla V_{t}^{*}(y)}_2  \leq C_{1, t}E_{t+1} + C_{2, t}(1 - \lambda_t)^{m+1}, \label{eq:error-nabla-v}
\end{align}
where 
\begin{align*}
    C_{1,t} \coloneqq  \frac{1 + (1 - \alpha_t)L_{0, t}^{s}}{\lambda_t\sqrt{\alpha_t}} , \;\; \text{and} \;\; C_{2, t} \coloneqq  \frac{\parenthesis{1 + (1 - \alpha_t)L_{0, t}^{u^*}}}{\sqrt{\alpha_t}}L_{0, t+1}^{\widehat{V}} +  \frac{\parenthesis{1 + (1 - \alpha_t)L_{0, t}^{s}}}{\sqrt{\alpha_t}}L_{0, t+1}^{V^*}.
\end{align*}
\end{Lemma}

\begin{proof}
    Recall the definition that $ \widehat{\mathcal{G}}_{t+1}^{(m)}(y)  = \E\bracket{\nabla \widehat{V}_{t+1}\left(\frac{1}{\sqrt{\alpha_t}}\parenthesis{y + (1 - \alpha_t)u_t^{(m)}(y)} + \sigma_t W_t\right)}$. It follows from \eqref{eq:grad-V*} and \eqref{eq:grad-v-m} that
\begin{align}
     \nabla V_{t}^{(m)}(y) -  \nabla V_{t}^{*}(y)  & =\frac{1}{\sqrt{\alpha_t}}\parenthesis{I_d
  + (1 - \alpha_t) \nabla s_t^{\rm pre}(y)}^\top\parenthesis{\widehat{\mathcal{G}}_{t+1}^{(m)}(y) - \mathcal{G}_{t+1}(y)} \nonumber \\ & \quad + \frac{(1 - \alpha_{t})}{\sqrt{\alpha_{t}}}\parenthesis{\nabla u_t^{(m)}(y) - \nabla s_t^{\rm pre}(y)}^\top \Big(\widehat{\mathcal{G}}_{t+1}^{(m)}(y)-\widehat{\mathcal{G}}_{t+1}^{(m-1)}(y)\Big).\label{eq:nable-V-n-V-*}
 \end{align}
Observe that for any $m \geq 0$, it holds that
\begin{align}
    \norm{\widehat{\mathcal{G}}_{t+1}^{(m)}(y)-\mathcal{G}_{t+1}(y)}_2 \leq \norm{\widehat{\mathcal{G}}_{t+1}^{(m)}(y)-\mathcal{G}_{t+1}^{(m)}(y)}_2 + \norm{\mathcal{G}_{t+1}^{(m)}(y)-\mathcal{G}_{t+1}(y)}_2, \label{eq:G-hat-G-*}
\end{align}
where we recall $\mathcal{G}_{t+1}^{(m)}(y)  \coloneqq \E\bracket{\nabla {V}^*_{t+1}\left(\frac{1}{\sqrt{\alpha_t}}\parenthesis{y + (1 - \alpha_t)u_t^{(m)}(y)} + \sigma_t W_t\right)}$.
For the first term on the RHS of \eqref{eq:G-hat-G-*}, we notice the following inequality holds:
\begin{align}
\norm{\widehat{\mathcal{G}}_{t+1}^{(m)}(y) - \mathcal{G}_{t+1}^{(m)}(y)}_{2}   \leq \E\bracket{\norm{\parenthesis{\nabla \widehat{V}_{t+1} - \nabla V^{*}_{t+1}}(y')}_{2}} \leq E_{t+1}, \label{eq:G-hat-G-n}
\end{align}
where $ y' \sim \mcal{N}\parenthesis{\frac{1}{\sqrt{\alpha_t}}\parenthesis{y + (1 - \alpha_t)u_t^{(m)})(y)}, \sigma_t^2 I_d} $. For the second term in \eqref{eq:G-hat-G-*}, the Lipschitz continuity of $\nabla V_{t+1}^*$ implies
\begin{align}
    \norm{\mathcal{G}^{(m)}_{t+1}(y)-\mathcal{G}_{t+1}(y)}_{2} & \leq L_{1, t+1}^{V^{*}} \frac{1 - \alpha_{t}}{\sqrt{\alpha_{t}}}\norm{u_{t}^{(m)} - u_{t}^{*}}_2. \label{eq:G*-n}
\end{align}
Consequently, Eq.~\eqref{eq:G-hat-G-*} becomes
\begin{align}
    \norm{\widehat{\mathcal{G}}_{t+1}^{(m)}(y)-\mathcal{G}_{t+1}(y)}_2 \leq E_{t+1} + L_{1, t+1}^{V^{*}} \frac{1 - \alpha_{t}}{\sqrt{\alpha_{t}}}\norm{u_{t}^{(m)} - u_{t}^{*}}. \label{eq:G-hat-G-*-new}
\end{align}

With \eqref{eq:upper-bound-magic}, \eqref{eq:upper-bound-nabla-u-n}, \eqref{eq:G-hat-G-n-bound} and \eqref{eq:G-hat-G-*-new}, we further bound \eqref{eq:nable-V-n-V-*} with
\begin{align*}
    \norm{\nabla V_{t}^{(m)}(y) -  \nabla V_{t}^{*}(y)}_2 & \leq C_{1, t}'\parenthesis{E_{t+1} + L_{1, t+1}^{V^{*}} \frac{1 - \alpha_{t}}{\sqrt{\alpha_{t}}}\norm{u_{t}^{(m)} - u_{t}^{*}}_2} + C_{2, t}'(1 - \lambda_t)^{m+1}, 
\end{align*}
where the coefficients are given by 
\begin{align*}
    C_{1,t}' = \frac{1}{\sqrt{\alpha_t}}(1 + (1 - \alpha_t)L_{0, t}^{s}), \;\; \text{and} \;\; C_{2, t}' = \frac{1 + (1 - \alpha_{t})L_{0,t}^{u^{*}}}{\sqrt{\alpha_t}}L_{0, t+1}^{\widehat{V}}.
\end{align*}
Finally, we apply Lemma \ref{lemma:un-u*-error} to obtain
\begin{align*}
    \norm{\nabla V_{t}^{(m)}(y) -  \nabla V_{t}^{*}(y)}_2 & \leq C_{1, t}'\parenthesis{E_{t+1} + L_{1, t+1}^{V^*}\frac{\sigma_t^2}{\beta_t}\parenthesis{(1 - \lambda_t)^m  L_{0,t+1}^{V^*} + \lambda_t^{-1}E_{t+1}} } + C_{2, t}'(1 - \lambda_t)^{m+1} \\ & \leq C_{1, t}'\parenthesis{E_{t+1} + (1 - \lambda_t)\parenthesis{(1 - \lambda_t)^m  L_{0,t+1}^{V^*} + \lambda_t^{-1}E_{t+1}} } + C_{2, t}'(1 - \lambda_t)^{m+1}  \\ & \leq C_{1, t}E_{t+1} + C_{2, t}(1 - \lambda_t)^{m+1},
\end{align*}
where 
\begin{align*}
    C_{1,t} = \lambda_t^{-1} C_{1,t}' , \;\; \text{and} \;\; C_{2, t} = C_{2, t}' + C_{1, t}' L_{0, t+1}^{V^*}.
\end{align*}
Therefore, we finish the proof. 
\end{proof}

Now we are ready to prove Theorem \ref{thm:convergence} with the results in Lemmas \ref{lemma:regularity}, \ref{lemma:un-u*-error} and \ref{lemma:nabla-Vn-V*-error}.
\begin{proof}[Proof of Theorem \ref{thm:convergence}]
It is straightforward to check that $L_{0, t}^{\bar{V}} \geq L_{0, t}^{V^*}$ for all $t \leq T$. Thus, the choice of $\beta_t$ guarantees that Theorem \ref{thm:well-posed} holds. To prove the theorem, we begin with the error $\norm{\nabla V_t^{m_t}(y) - \nabla V_t^{*}(y)}_2$. Let $t < T$ be fixed. To apply Lemma \ref{lemma:nabla-Vn-V*-error}, we choose $\widehat{V}_{t'} = V_{t'}^{(m_{t'})}$ for all $t' > t$. Unrolling recursion \eqref{eq:error-nabla-v} and use the fact that $V_T^{(m_T)}(y) = R(y) = V_T^*(y)$, we obtain that
\begin{align*}
    \norm{\nabla V_t^{(m_t)}(y) - \nabla V_t^{*}(y)}_2 & \leq \prod_{k = t}^{T}C_{1, k}\norm{\nabla V_T^{(m_T)}(y) - V_T^{*}(y)}_2 + \sum_{k = t}^{T - 1} \parenthesis{\prod_{\ell = t}^{k - 1}C_{1, \ell}} C_{2, k} (1 - \lambda_k)^{m_k + 1} \\ & = \sum_{k = t}^{T - 1} \parenthesis{\prod_{\ell = t}^{k - 1}C_{1, \ell}} C_{2, k} (1 - \lambda_k)^{m_k + 1}.
\end{align*}
Next, we apply Lemma \ref{lemma:un-u*-error} to have
\begin{align*}
    \norm{u_t^{(m_t)}(y) - u_t^*(y)}_2 & \leq  \parenthesis{(1 - \lambda_t)^{m_t}  L_{0,t+1}^{V^*} + \lambda_t^{-1}\norm{\nabla V_{t+1}^{m_{t+1}}(y) - \nabla V_{t+1}^{*}(y)}_2}  \frac{\sqrt{\alpha_{t}}(1 - \lambda_t)}{(1 - \alpha_{t})L_{1, t+1}^{V^*}}  \\ & \leq \parenthesis{(1 - \lambda_t)^{m_t}  L_{0,t+1}^{V^*} + \lambda_t^{-1}\sum_{k = t+1}^{T - 1} \parenthesis{\prod_{\ell = t+1}^{k - 1}C_{1, \ell}} C_{2, k} (1 - \lambda_k)^{m_k + 1}}  \frac{\sqrt{\alpha_{t}}(1 - \lambda_t)}{(1 - \alpha_{t})L_{1, t+1}^{V^*}},
\end{align*}
which finishes the proof.

\end{proof}

\section{Discussions}\label{sec:discussion}
Here, we provide some discussions on two future directions that we want to further explore.
\subsection{Parameterized formulation}
{\color{black}In this section, we adopt a linear parameterization to enable practical and  efficient update of Algorithm \ref{algo:update-u}:
\begin{align*}
    u_t(y) = K_t\phi(y), \;\; t = 0, \dots, T - 1,
\end{align*}
where $\phi(y) = (\phi_1(y), \dots, \phi_p(y))^\top$ is the given basis function and ${\bf K} = \curly{K_t}_{t = 0}^{T - 1}$ are (unknown) parameters to learn. Despite its simplicity, this parameterization is flexible, expressive, and tractable, and is widely used in control and reinforcement learning \citep{Agarwal2021OnTT,jin2020provably}. With an appropriate choice of basis, it can capture a broad class of score approximators, including random features \citep{rahimi2007random}, kernel methods \citep{steinwart2008support}, and overparameterized neural networks in the NTK regime \citep{jacot2018neural}.}

We provide a roadmap to show how a linear convergence rate can be achieved, with details similar to earlier sections omitted. We make the following realizability assumption. 

\begin{Assumption}[Realizability]
Assume that at each timestamp $t$, the optimal control satisfies $u_t^*(y) = K_t^*\phi(y)$ for some matrix $K_t^* \in \R^{p \times d}$. 
\end{Assumption}

Define the expected value function associated with the policy ${\bf K}$ as
\begin{align*}
    J({\bf K}) \coloneqq \E_{Y_0\sim \mathcal{D}}\bracket{V_0^{\bf K}(Y_0)},
\end{align*}
where $\mathcal{D}$ is the  distribution of the initial state and we can take $\mathcal{D}=\mathcal{N}(0,I_d)$, which is the distribution of $Y_0^{\rm pre}$. Here the value function is defined as
\begin{align*}
    V_t^{\bf K}(y) \coloneqq \E\bracket{\left. R(Y_T) - \sum_{\ell = t}^{T - 1}\beta_\ell \frac{(1 - \alpha_\ell)^2}{2\alpha_\ell \sigma_\ell^2}\norm{u_\ell(Y_{\ell}) - s_\ell^{\rm pre}(Y_{\ell})}_{2}^{2}\,\,\right|\,\, Y_t = y},
\end{align*}
with terminal condition
\begin{align*}
    V_T^{\bf K}(y) = r(y) = \E\bracket{R(y)}.
\end{align*}
Similarly, define the Q-function:
\begin{align*}
    Q_t^{\bf K}(y, u) = \E\bracket{V_{t+1}^{\bf K}\parenthesis{\frac{1}{\sqrt{\alpha_t}}(y + (1- \alpha_t)u) + \sigma_t W_t}-\beta_t\frac{(1 - \alpha_t)^2}{2\alpha_t\sigma_t^2}\norm{u - s_t^{\rm }(y)}_2^2},
\end{align*}
where the expecation is taken over $W_t$. 
We remark that the policy $ \bf K $ in the superscript refers to $ \curly{K_{t'}}_{t' > t} $ at each $t$. 
Recall the choice of $\beta_t$ implies the map $u \mapsto Q_t^{\bf K^*}(y, u)$ is $\gamma_t$-strongly concave uniformly over all $y$; see Theorem \ref{thm:well-posed}. If the feature mapping $\phi$ is well-behaved, the map $K_t \mapsto Q_t^{\bf K^{*}}(y, K_t\phi(y))$ is also strongly concave. 

Next, we calculate the policy gradient with respect to $ \mathbf{K} $. Note that
\begin{align}
    \partial_{t}J({\bf K}) \coloneqq \frac{\partial J({\bf K})}{\partial K_t} & = \frac{\partial}{\partial K_{t}} \E\bracket{-\sum_{\ell = 0}^{t-1}\beta_{\ell}\frac{(1 - \alpha_\ell)^{2}}{2\alpha_{\ell}\sigma_{\ell}^{2}}\norm{K_{\ell}\phi(Y_{\ell}) - s_{\ell}^{\rm pre}(Y_{\ell})}_{2}^{2} + Q_{t}^{\bf K}(Y_{t}, K_{t}\phi(Y_{t}))} \nonumber \\ & = \frac{\partial}{\partial K_{t}} \E\bracket{Q_{t}^{\bf K}(Y_{t}, K_{t}\phi(Y_{t}))}, \label{eq:policy-gradient}
\end{align}
where the expectation is taken over the initial state $ Y_{0}\sim\mathcal{D}$ and noise $ \curly{W_{t}}_{t=0}^{T-1} $. We next show that the optimal policy $ \bf K^{*} $ is the unique stationary point of $ J(\bf K) $. Let $ \overline{\bf K} $ be a stationary point, i.e., $ \partial_{t}J({\bf K}) \Big\vert_{\bf K = \overline{\bf K}} = 0 $ for all $ t  $. In particular, we have 
\begin{align*}
    \frac{\partial}{\partial K_{T-1}} \E\bracket{Q_{T-1}^{\bf K}(Y_{T-1}, K_{T-1}\phi(Y_{T-1}))} \Big\vert_{\bf K = \overline{\bf K}}= 0.
\end{align*}
The strong concavity in $ K_{T-1} $ implies that $ K_{T - 1}^{*} = \overline{K}_{T - 1}$. Similarly, we have
\begin{align*}
    \frac{\partial}{\partial K_{T-2}} \E\bracket{Q_{T-2}^{\bf K}(Y_{T-2}, K_{T-2}\phi(Y_{T-2}))} \Big\vert_{\bf K = \overline{\bf K}}= \frac{\partial}{\partial K_{T-2}} \E\bracket{Q_{T-2}^{\bf K^{*}}(Y_{T-2}, K_{T-2}\phi(Y_{T-2}))} \Big\vert_{\bf K = \overline{\bf K}}= 0.
\end{align*}
By strong concavity again, we have $  K_{T - 2}^{*} =\overline{K}_{T - 2} $. Repeating the argument to conclude $  {\bf K}^{*}=\overline{\bf K}  $ is the unique stationary point of the objective function $J(\bf K)$. 
The above analysis replies on the two~facts:
\begin{enumerate}[(i)]
    \item At $ T - 1 $, the map $ K_{T-1} \mapsto \E\bracket{Q_{T-1}^{\bf K}(Y_{T-1}, K_{T-1}\phi(Y_{T-1}))} $ has a unique stationary point. 
    \item For any $ t < T - 1 $, the map $ K_{t} \mapsto \E\bracket{Q_{t}^{\bf K^{*}}(Y_{t}, K_{t}\phi(Y_{t}))} $ has a unique stationary point. 
\end{enumerate}
We remark that the above two conditions (i)--(ii) {\it relax}  \cite[Assumption 2.A]{bhandari2024global}, which assumes that $ K_{t} \mapsto \E\bracket{Q_{t}^{\bf K}(Y_{t}, K_{t}\phi(Y_{t}))} $ has no sub-optimal stationary points under {\it any} $ \bf K $. While \cite{bhandari2024global} primarily focuses on global convergence guarantees for infinite-horizon problems, it also introduces a condition ({\it cf.} Condition 4) for finite-horizon MDPs that aligns in spirit with our analysis. 

{\color{black}
Notably, \citet{bhandari2024global} {\it did not} analyze the convergence rate of policy gradient methods for finite-horizon problems.
Since $ {\bf K}^{*} $ is the unique stationary point, the policy gradient method converges to the globally optimal solution as long as the standard smoothness condition is satisfied. Moreover, we conjecture a {\it linear convergence rate} due to the presence of strong concavity. Consider the policy gradient update rule:
\begin{align*}
    K_t^{(m+1)} = K_t^{(m)} + \eta \partial_t J({\bf K}^{(m)}),  \;\; 0 \leq t \leq T - 1, m \geq 0.
\end{align*}
For each $t < T$, we have
\begin{align*}
    \norm{K_t^{(m+1)} - K_t^*}_F  & = \norm{K_t^{(m)} + \eta \partial_t J({\bf K}^{(m)})  - K_t^*}_F \\ &  \leq \norm{K_t^{(m)} + \eta \partial_t J({\bf K}_{\leq t}^{(m)}, {\bf K}_{> t}^{*})  - K_t^*}_F  + \eta\norm{\partial_t J({\bf K}_{\leq t}^{(m)}, {\bf K}_{> t}^{*}) - \partial_t J({\bf K}^{(m)})}_F,
\end{align*}
where ${\bf K}_{\leq t} \coloneqq \curly{K_\ell}_{\ell=0}^t$ and ${\bf K}_{> t} \coloneqq \curly{K_\ell}_{\ell >t}$.
Eq. \eqref{eq:policy-gradient} implies $\partial_t J({\bf K}) = \partial_t J({\bf K}_{\geq t})$ is independent of $\curly{K_\ell}_{\ell < t}$. The (one-step) strong concavity and smoothness leads to 
\begin{align*}
  \norm{K_t^{(m)} + \eta \partial_t J({\bf K}_{\leq t}^{(m)}, {\bf K}_{> t}^{*})  - K_t^*}_F \leq c \norm{K_t^{(m)} - K_t^*}_F,
\end{align*}
for some constant $c<1$. 
If $\partial_t J(\bf K)$ is further $L_{1,t}^J$-Lipschitz in ${\bf K} \in \R^{p \times d}$, then we have
\begin{align*}
   \norm{\partial_t J({\bf K}_{\leq t}^{(m)}, {\bf K}_{> t}^{*}) - \partial_t J({\bf K}^{(m)})}_F 
   \leq L_{1, t}^J \norm{{\bf K}_{> t}^{(m)} - {\bf K}_{> t}^{*}}_F,
\end{align*}
and consequently we prove the linear convergence rate of the policy gradient method.}

\subsection{Continuous-time limit}

While we find the discrete-time formulation more natural for the fine-tuning problem, the continuous-time limit is inherently intriguing from a mathematical perspective. The potential benefit of continuous-time formulation includes robustness and stability in training \citep{zhao2024scores}.

In the continuous time limit, we have
\begin{eqnarray*}
 V(t,y)  =  \sup_{U:=\curly{U_s}_{t\leq s\leq T}\in \mathcal{U}} \;\; \mathbb{E}\left[R(Y_T) -  \int_t^T \beta_s \frac{(1-\alpha_s)^2}{2\alpha_s\sigma_s^2} \|s_s^{\rm pre}(Y_s)-U_s\|_2^2 {\rm d} s\,\,\Bigg|\,\,Y_t=y\right]
\end{eqnarray*}
subject to
\begin{align*}
    {\rm d} Y_t= \parenthesis{\frac{1-\sqrt{\alpha_t}}{\sqrt{\alpha_t}}Y_t + \frac{(1 - \alpha_t)}{\sqrt{\alpha_t}}U_t} {\rm d} t + \sigma_t {\rm d} W_t,
\end{align*}
where $\mathcal{U}$ contains all Markovian policies that adapts to $\{\mathcal{F}^W_t\}_{t\ge 0}$.
The HJB follows:
\begin{eqnarray*}
 0=   \partial_t V(t,y) + \sup_{u} \left\{\parenthesis{\frac{1-\sqrt{\alpha_t}}{\sqrt{\alpha_t}}y + \frac{(1 - \alpha_t)}{\sqrt{\alpha_t}}u} \partial_y V(t,y)- \beta_t \frac{(1-\alpha_t)^2}{2\alpha_t\sigma_t^2} \|s_t^{\rm pre}(y)-u\|_2^2\right\} + \frac{\sigma_t^2}{2}\partial_{yy}V(t,y),
\end{eqnarray*}
with terminal condition $V(T,y) = r(y) = \mathbb{E}[R(y)]$. Hence the optimal control follows:
\begin{eqnarray}\label{eq:continuous_time_control}
    u^*_t(y) = s^{\rm pre}_t(y) + \frac{\sqrt{\alpha_t}\sigma_t^2}{\beta_t (1-\alpha_t)}\partial_y V(t,y),
\end{eqnarray}
and the HJB can be rewritten as
\begin{eqnarray}
 0 &=&    \partial_t V(t,y) +  \frac{1-\sqrt{\alpha_t}}{\sqrt{\alpha_t}}y\,\,\partial_y V(t,y) + \frac{1-\alpha_t}{\sqrt{\alpha_t}} s^{\rm pre}_t(y)  \partial_y V(t,y)+ \frac{\sigma_t^2}{2\beta_t} (\partial_y V(t,y))^2+ \frac{\sigma_t^2}{2}\partial_{yy}V(t,y)\nonumber\\
 &=&    \partial_t V(t,y) +  \left(\frac{1-\sqrt{\alpha_t}}{\sqrt{\alpha_t}}y+  \frac{1-\alpha_t}{\sqrt{\alpha_t}} s^{\rm pre}_t(y)\right)\,\,  \partial_y V(t,y)+ \frac{\sigma_t^2}{2\beta_t} (\partial_y V(t,y))^2+ \frac{\sigma_t^2}{2}\partial_{yy}V(t,y)\label{eq:continuous_time_hjb}
\end{eqnarray}
with terminal condition $V(T,y) = r(y) = \mathbb{E}[R(y)]$. 

Note that both the optimal control \eqref{eq:continuous_time_control} and the HJB \eqref{eq:continuous_time_hjb} share the same form as the discrete-time system, respectively, \eqref{eq:u*(y)} and \eqref{eq:bellman-u*}.

There are a few interesting aspects to study for the continuous-time system \eqref{eq:continuous_time_control}--\eqref{eq:continuous_time_hjb}:

\textbf{Regularity of the value function:} Under conditions \ref{ass:smooth r}--\ref{ass:smooth s}, we believe the regulatory of the optimal value function, in terms of $V\in \mathcal{C}^{1,2}([0,T]\times\mathbb{R}^d)$, can be achieved by adapting the analysis in \cite[Proposition 5.3]{szpruch2021exploration} and \cite[Theorem 4.3.1]{zhang2017backward} via forward-backward stochastic differential equations (FBSDEs) representation. However, the regularity of the value function and iterated control policies {\it along training} remains less clear and non-trivial in the continuous-time regime \cite[Assumption 2]{zhou2023policy}.  Furthermore, the integration-by-parts result \eqref{eq:magic} aligns with the stochastic representation of the gradient in the continuous-time analysis, as detailed in \cite[Eqns 2.12–2.15]{ma2024convergence} and \cite[Chapter 2]{zhang2001some}. We believe this approach will be instrumental in establishing regularity results throughout the training process.

\textbf{Algorithm design and convergence for continuous-time systems:} 
    The algorithm design and convergence analysis is challenging in continuous-time systems. In the literature, \citet{zhou2023policy} studied the convergence of policy gradient flow and then \cite{zhou2024solving} extended the analysis to actor-critic algorithms. The recent work \cite{mou2024bellman} proposed an efficient algorithm to compute the value function of a continuous-time system given discrete trajectory data. However, all these work only focused on the algorithm design in the functional space while the analysis in parametric settings is overlooked. We believe our framework would provide insights of global convergence guarantees for parametric algorithms in both continuous- and discrete-time problems. 

{\color{black}
\section{Numerical Experiments}\label{sec:experiment}
In this section, we evaluate the performance of the PI-FT algorithm from Section \ref{sec:algorithm} via numerical experiments, focusing on the following questions:

\begin{itemize}

    \item In practice, how fast does the PI-FT algorithm converge to the optimal solution?

    \item How does the choice of $\beta$ affect the convergence rate and the quality of the fine-tuned models?
\end{itemize}

As shown in this section, the PI-FT algorithm converges efficiently to the global optimum; increasing $ \beta $ accelerates convergence and yields a model closer to the pre-trained one, aligning with our theoretical analysis in Section \ref{sec:algorithm}.

\paragraph{Model Setup.}
We fine-tune the Stable Diffusion v1.5 \citep{rombach2022high} for text-to-image generation, using LoRA \citep{hu2021lora} and ImageReward \citep{xu2023imagereward}. Following  \citep{fan2024reinforcement}, we use four prompts—``A green colored rabbit,” ``A cat and a dog,” ``Four wolves in the park,” and ``A dog on the moon”—to evaluate the model’s ability to generate correct color, composition, counting, and location, respectively. During training, we generate 10 trajectories, each consisting of 50 transitions, to calculate the gradient with 1000 gradient steps. By default, we use the AdamW optimizer with a learning rate of $ 3 \times 10^{-4} $ , and set the KL regularization coefficient to a fixed value as $ \beta = 0.01 $. 

\paragraph{Evaluation.}
We first compare ImageReward scores for images generated by the pre-trained model, DPOK \citep{fan2024reinforcement}, and our proposed PI-FT. For a fair comparison, we configure DPOK to perform 10 gradient steps per sampling step, using a learning rate of $ 1 \times 10^{-5} $. Each gradient step is computed using 50 randomly sampled transitions from a replay buffer. As a result, 1000 sampling steps and a total of 10,000 gradient steps in DPOK yield a computational cost comparable to that of PI-FT. As shown in Figure \ref{fig:bar_plot}, PI-FT consistently outperforms both baselines across all four prompts. Figure \ref{fig:prompt_comparision} further shows that PI-FT more accurately captures object counts and placements (e.g., four wolves and the dog on the moon) and avoids errors like miscoloring the rabbit. It also produces more natural textures compared to the baselines.

\begin{figure}[h]
    \centering
    \includegraphics[width=0.7\linewidth]{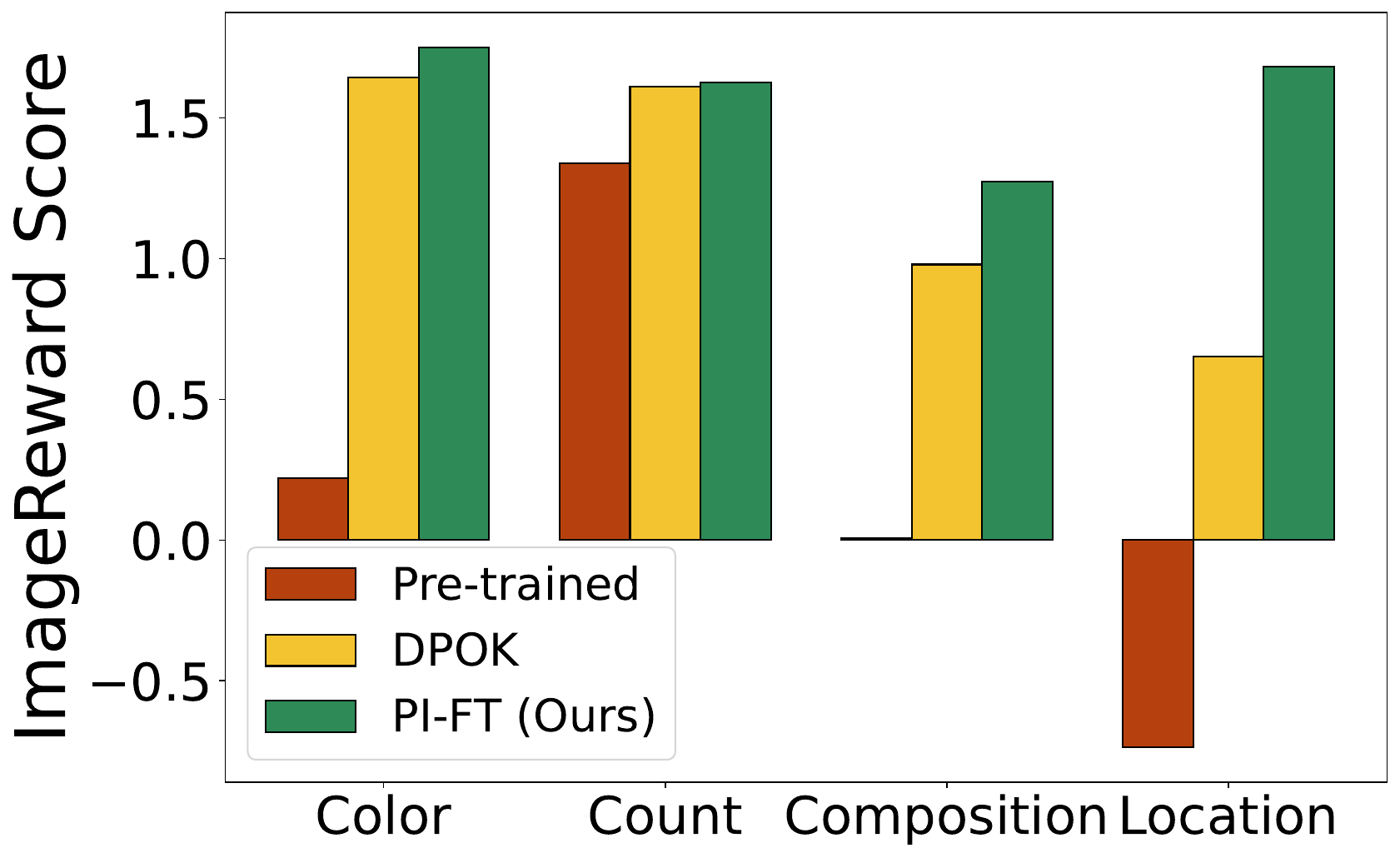}
    \caption{Comparison of the ImageReward Score among three models: pre-trained Stable Diffusion (red), DPOK (yellow), and PI-FT (green). The ImageReward scores are averaged over 50 samples from each model. Prompts from left to right: ``A green colored rabbit'' (color), ``Four wolves in the park'' (count), ``A cat and a dog'' (composition), and ``A dog on the moon'' (location).}
    \label{fig:bar_plot}
\end{figure}

\begin{figure}[h]
    \centering
    \includegraphics[width=0.9\linewidth]{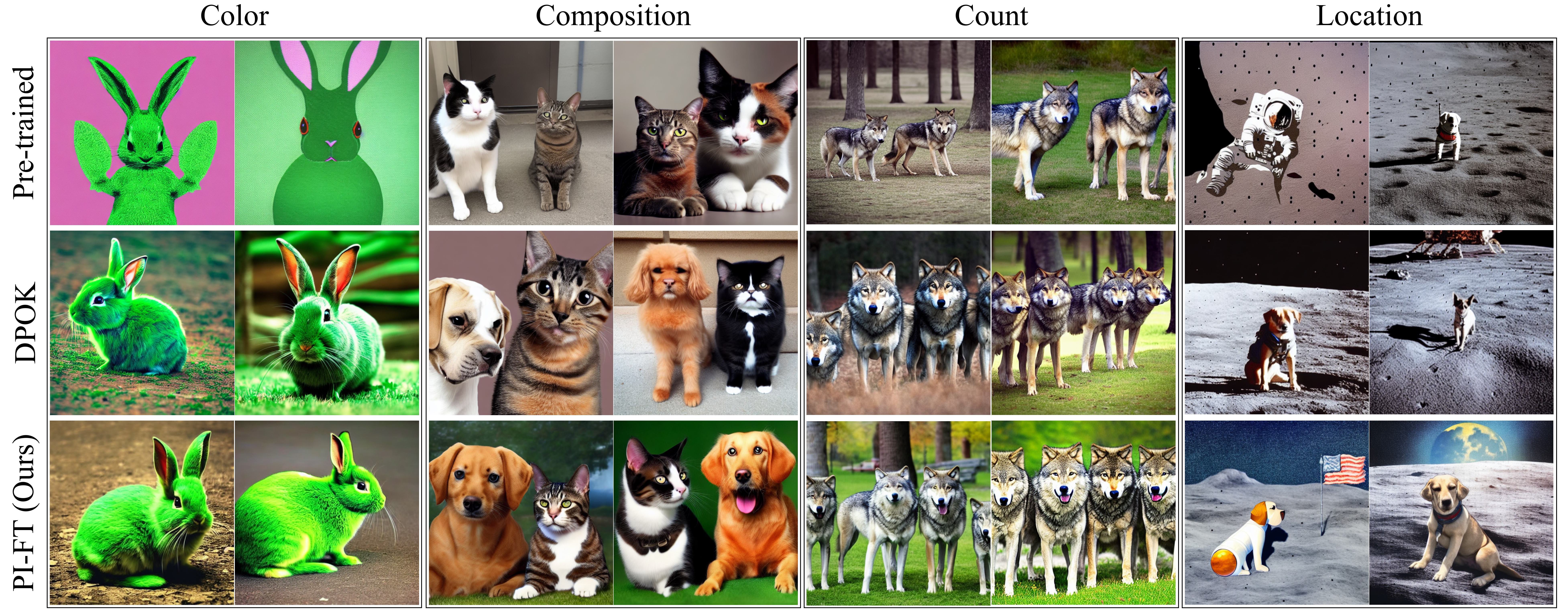}
    \caption{Visual comparison of images generated by the original Stable Diffusion model (pre-trained), DPOK model, and PI-FT model (ours). Prompts from left to right: ``A green colored rabbit'' (color), ``A cat and a dog'' (composition), ``Four wolves in the park'' (count), and ``A dog on the moon'' (location).}
    \label{fig:prompt_comparision}
\end{figure}

\paragraph{Effect of KL regularization.}
KL regularization is known to enhance fine-tuning. We study its effect in PI-FT using the prompt “Four wolves in the park,” varying $\beta \in \{0.01, 0.1, 1.0\}$. As shown in Figure \ref{fig:grad norm}, the gradient norm decreases to zero in all cases, indicating convergence. Figure \ref{fig:reward} shows that small $\beta$ values improve and stabilize the ImageReward score, while larger $\beta$ offers limited gains. This aligns with Figure \ref{fig:kl}, where KL divergence remains high for $\beta = 0.01$, but stays significantly lower for $\beta \in \{0.1, 1.0\}$. Figure \ref{fig:example} also shows that smaller $\beta$ produces images with nearly four wolves, whereas larger $\beta$ leads to fewer. These results underscore the importance of the KL coefficient in effective fine-tuning.

\begin{figure}[h]
    \centering
    \begin{subfigure}[H]{0.32\linewidth}
        \centering
        \includegraphics[width=\linewidth]{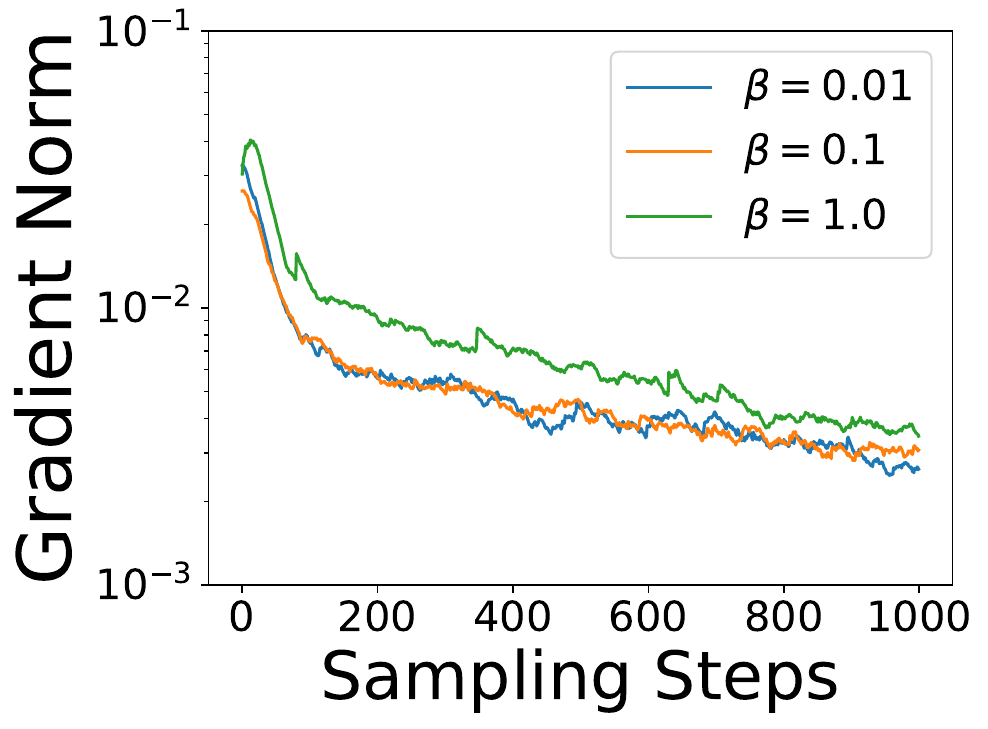}
        \caption{Gradient norm}
        \label{fig:grad norm}
    \end{subfigure}
    \hfill
    \begin{subfigure}[H]{0.32\linewidth}
        \centering
        \includegraphics[width=\linewidth]{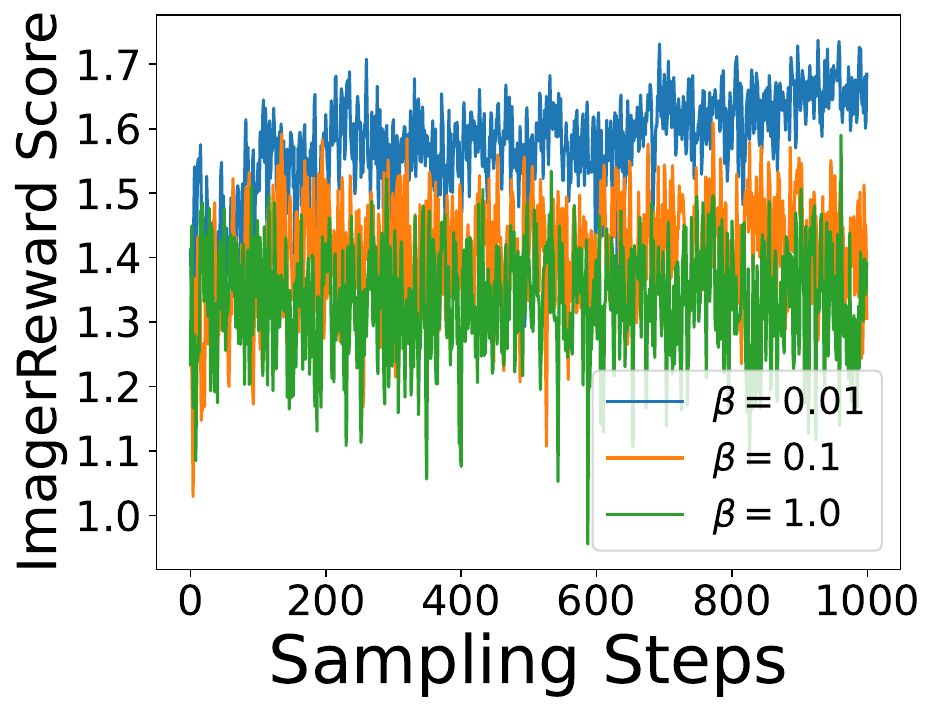}
        \caption{ImageReward score}
        \label{fig:reward}
    \end{subfigure}
    \hfill
    \begin{subfigure}[H]{0.32\linewidth}
        \centering
        \includegraphics[width=\linewidth]{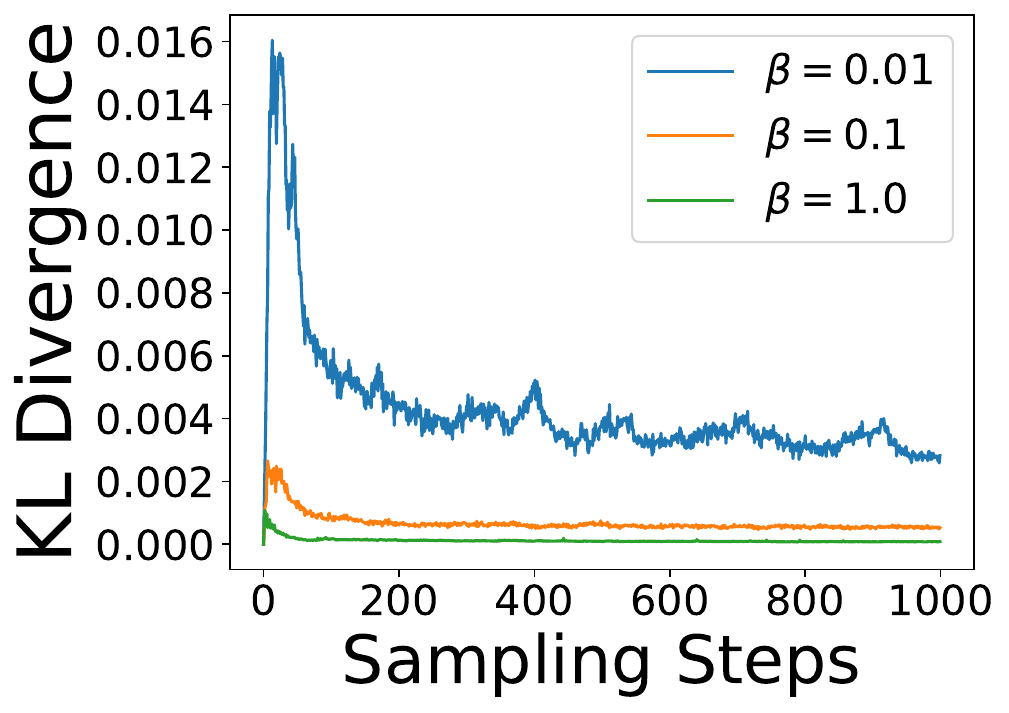}
        \caption{KL divergence}
        \label{fig:kl}
    \end{subfigure}
    \caption{(a) Gradient norm (logarithmic scale) of PI-FT during training. The curves are smoothed using the exponential moving average (EMA). A linear convergence rate is observed. (b) ImageReward score of PI-FT during training. Smaller KL regularization coefficient $\beta$ leads to higher ImageReward score. (c) KL divergence of PI-FT during training. Larger KL regularization coefficient $ \beta $ leads to smaller KL divergence and faster convergence.}
\end{figure}

\begin{figure}[h]
    \centering
    \includegraphics[width=\linewidth]{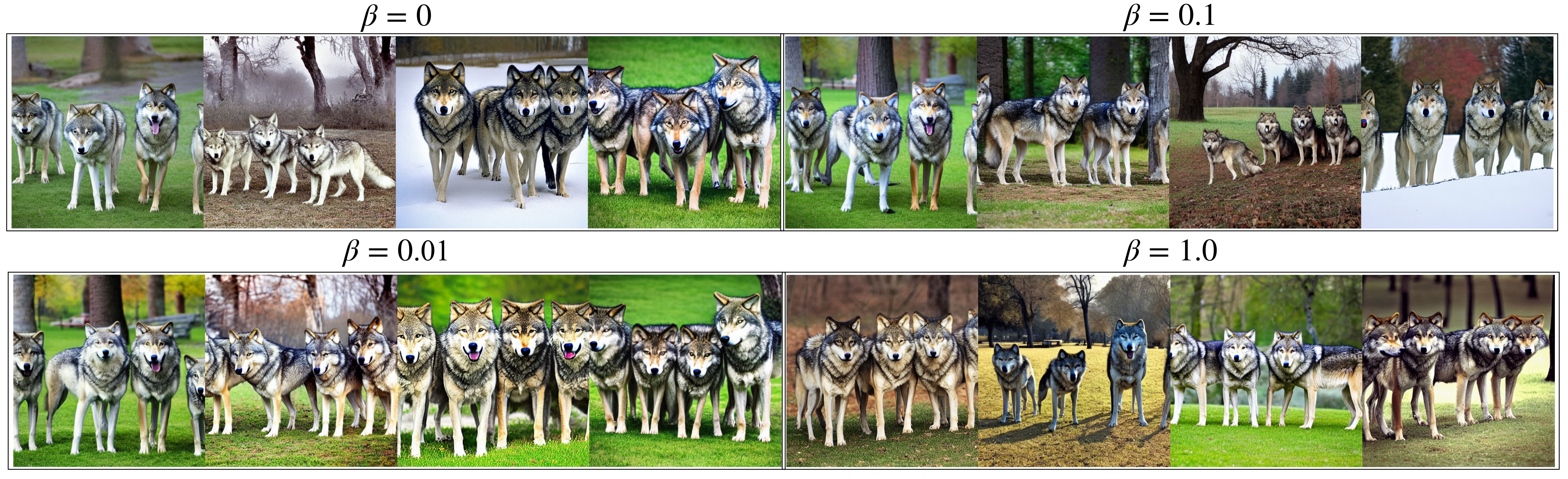}
    \caption{Randomly generated samples from PI-FT model  with different KL regularization coefficients. Images from a single text prompt: ``Four wolves in the park''. The model with smaller $ \beta > 0$ generates more accurate number of wolves.}
    \label{fig:example}
\end{figure}
}

{\color{black}\section{Conclusion}
We introduce a stochastic control framework for fine-tuning diffusion models, integrating linear dynamics with KL regularization. Our approach establishes the well-posedness and regularity of the control problem and proposes a policy iteration algorithm (PI-FT) that guarantees global linear convergence. Unlike prior work that assumes regularity throughout training, we prove that PI-FT inherently maintains these properties. Additionally, we extend our framework to parametric settings, broadening its applicability. This work advances the theoretical understanding of fine-tuning diffusion models and provides a foundation for developing more effective fine-tuning algorithms. Our algorithmic design and theoretical findings are also supported by thorough numerical experiments.}

\paragraph{Acknowledgment.} R.X. is partially supported by the NSF CAREER award DMS-2339240 and a JP
Morgan Faculty Research Award. The work of M.R. is partially supported by a gift from Google Research and a gift from Meta.

\bibliography{reference.bib}

\begin{thebibliography}{79}
\providecommand{\natexlab}[1]{#1}
\providecommand{\url}[1]{\texttt{#1}}
\expandafter\ifx\csname urlstyle\endcsname\relax
  \providecommand{\doi}[1]{doi: #1}\else
  \providecommand{\doi}{doi: \begingroup \urlstyle{rm}\Url}\fi

\bibitem[Agarwal et~al.(2021)Agarwal, Kakade, Lee, and Mahajan]{Agarwal2021OnTT}
Alekh Agarwal, Sham~M. Kakade, Jason Lee, and Gaurav Mahajan.
\newblock On the theory of policy gradient methods: Optimality, approximation, and distribution shift.
\newblock \emph{Journal of Machine Learning Research}, 22:\penalty0 98:1--98:76, 2021.

\bibitem[Alzubaidi et~al.(2023)Alzubaidi, Bai, Al-Sabaawi, Santamar{\'\i}a, Albahri, Al-dabbagh, Fadhel, Manoufali, Zhang, Al-Timemy, et~al.]{alzubaidi2023survey}
Laith Alzubaidi, Jinshuai Bai, Aiman Al-Sabaawi, Jose Santamar{\'\i}a, Ahmed~Shihab Albahri, Bashar Sami~Nayyef Al-dabbagh, Mohammed~A Fadhel, Mohamed Manoufali, Jinglan Zhang, Ali~H Al-Timemy, et~al.
\newblock A survey on deep learning tools dealing with data scarcity: definitions, challenges, solutions, tips, and applications.
\newblock \emph{Journal of Big Data}, 10\penalty0 (1):\penalty0 46, 2023.

\bibitem[Bansal et~al.(2022)Bansal, Sharma, and Kathuria]{bansal2022systematic}
Ms~Aayushi Bansal, Dr~Rewa Sharma, and Dr~Mamta Kathuria.
\newblock A systematic review on data scarcity problem in deep learning: solution and applications.
\newblock \emph{ACM Computing Surveys (Csur)}, 54\penalty0 (10s):\penalty0 1--29, 2022.

\bibitem[Berner et~al.(2024)Berner, Richter, and Ullrich]{berner2022optimal}
Julius Berner, Lorenz Richter, and Karen Ullrich.
\newblock An optimal control perspective on diffusion-based generative modeling.
\newblock 2024.
\newblock ISSN 2835-8856.
\newblock URL \url{https://openreview.net/forum?id=oYIjw37pTP}.

\bibitem[Betker et~al.(2023)Betker, Goh, Jing, Brooks, Wang, Li, Ouyang, Zhuang, Lee, Guo, et~al.]{betker2023improving}
James Betker, Gabriel Goh, Li~Jing, Tim Brooks, Jianfeng Wang, Linjie Li, Long Ouyang, Juntang Zhuang, Joyce Lee, Yufei Guo, et~al.
\newblock Improving image generation with better captions.
\newblock \emph{Computer Science. https://cdn. openai. com/papers/dall-e-3. pdf}, 2\penalty0 (3):\penalty0 8, 2023.

\bibitem[Bhandari and Russo(2021)]{bhandari2021linear}
Jalaj Bhandari and Daniel Russo.
\newblock {On the linear convergence of policy gradient methods for finite MDPs}.
\newblock In \emph{International Conference on Artificial Intelligence and Statistics}, pages 2386--2394. PMLR, 2021.

\bibitem[Bhandari and Russo(2024)]{bhandari2024global}
Jalaj Bhandari and Daniel Russo.
\newblock Global optimality guarantees for policy gradient methods.
\newblock \emph{Operations Research}, 2024.

\bibitem[Black et~al.(2024)Black, Janner, Du, Kostrikov, and Levine]{black2023training}
Kevin Black, Michael Janner, Yilun Du, Ilya Kostrikov, and Sergey Levine.
\newblock Training diffusion models with reinforcement learning.
\newblock In \emph{The Twelfth International Conference on Learning Representations}, 2024.
\newblock URL \url{https://openreview.net/forum?id=YCWjhGrJFD}.

\bibitem[Bu et~al.(2019)Bu, Mesbahi, Fazel, and Mesbahi]{bu2019lqr}
Jingjing Bu, Afshin Mesbahi, Maryam Fazel, and Mehran Mesbahi.
\newblock {LQR} through the lens of first order methods: Discrete-time case.
\newblock \emph{arXiv preprint arXiv:1907.08921}, 2019.

\bibitem[Cen et~al.(2021)Cen, Cheng, Chen, Wei, and Chi]{Cen2021FastGC}
Shicong Cen, Chen Cheng, Yuxin Chen, Yuting Wei, and Yuejie Chi.
\newblock Fast global convergence of natural policy gradient methods with entropy regularization.
\newblock \emph{Operations Research}, 2021.

\bibitem[Dai et~al.(2023)Dai, Hou, Ma, Tsai, Wang, Wang, Zhang, Vandenhende, Wang, Dubey, et~al.]{dai2023emu}
Xiaoliang Dai, Ji~Hou, Chih-Yao Ma, Sam Tsai, Jialiang Wang, Rui Wang, Peizhao Zhang, Simon Vandenhende, Xiaofang Wang, Abhimanyu Dubey, et~al.
\newblock Emu: Enhancing image generation models using photogenic needles in a haystack.
\newblock \emph{arXiv preprint arXiv:2309.15807}, 2023.

\bibitem[Dhariwal and Nichol(2021)]{dhariwal2021diffusion}
Prafulla Dhariwal and Alexander Nichol.
\newblock {Diffusion models beat GANs on image synthesis}.
\newblock \emph{Advances in neural information processing systems}, 34:\penalty0 8780--8794, 2021.

\bibitem[Ding et~al.(2020)Ding, Zhang, Başar, and Jovanovi{\'c}]{Ding2020NaturalPG}
Dongsheng Ding, Kaiqing Zhang, Tamer Başar, and Mihailo~R. Jovanovi{\'c}.
\newblock Natural policy gradient primal-dual method for constrained{ Markov} decision processes.
\newblock In \emph{NeurIPS}, 2020.

\bibitem[Doersch(2016)]{doersch2016tutorial}
Carl Doersch.
\newblock Tutorial on variational autoencoders.
\newblock \emph{arXiv preprint arXiv:1606.05908}, 2016.

\bibitem[Domingo-Enrich et~al.(2025)Domingo-Enrich, Drozdzal, Karrer, and Chen]{domingo2024adjoint}
Carles Domingo-Enrich, Michal Drozdzal, Brian Karrer, and Ricky~TQ Chen.
\newblock Adjoint matching: Fine-tuning flow and diffusion generative models with memoryless stochastic optimal control.
\newblock In \emph{International Conference on Learning Representations}, 2025.
\newblock URL \url{https://openreview.net/forum?id=xQBRrtQM8u}.

\bibitem[Durgadevi et~al.(2021)]{durgadevi2021generative}
M~Durgadevi et~al.
\newblock {Generative Adversarial Network (GAN): A general review on different variants of GAN and applications}.
\newblock In \emph{2021 6th International Conference on Communication and Electronics Systems (ICCES)}, pages 1--8. IEEE, 2021.

\bibitem[Fan and Lee(2023)]{fan2023optimizing}
Ying Fan and Kangwook Lee.
\newblock {Optimizing DDPM Sampling with Shortcut Fine-Tuning}.
\newblock In \emph{International Conference on Machine Learning}, pages 9623--9639. PMLR, 2023.

\bibitem[Fan et~al.(2024)Fan, Watkins, Du, Liu, Ryu, Boutilier, Abbeel, Ghavamzadeh, Lee, and Lee]{fan2024reinforcement}
Ying Fan, Olivia Watkins, Yuqing Du, Hao Liu, Moonkyung Ryu, Craig Boutilier, Pieter Abbeel, Mohammad Ghavamzadeh, Kangwook Lee, and Kimin Lee.
\newblock Reinforcement learning for fine-tuning text-to-image diffusion models.
\newblock \emph{Advances in Neural Information Processing Systems}, 36, 2024.

\bibitem[Fatkhullin et~al.(2023)Fatkhullin, Barakat, Kireeva, and He]{fatkhullin2023stochastic}
Ilyas Fatkhullin, Anas Barakat, Anastasia Kireeva, and Niao He.
\newblock Stochastic policy gradient methods: Improved sample complexity for fisher-non-degenerate policies.
\newblock In \emph{International Conference on Machine Learning}, pages 9827--9869. PMLR, 2023.

\bibitem[Fazel et~al.(2018)Fazel, Ge, Kakade, and Mesbahi]{fazel2018global}
Maryam Fazel, Rong Ge, Sham Kakade, and Mehran Mesbahi.
\newblock Global convergence of policy gradient methods for the linear quadratic regulator.
\newblock In \emph{International Conference on Machine Learning}, pages 1467--1476. PMLR, 2018.

\bibitem[Fetaya et~al.(2020)Fetaya, Jacobsen, Grathwohl, and Zemel]{fetaya2019understanding}
Ethan Fetaya, Joern-Henrik Jacobsen, Will Grathwohl, and Richard Zemel.
\newblock Understanding the limitations of conditional generative models.
\newblock In \emph{International Conference on Learning Representations}, 2020.
\newblock URL \url{https://openreview.net/forum?id=r1lPleBFvH}.

\bibitem[Folland(1999)]{folland1999real}
Gerald~B Folland.
\newblock \emph{Real analysis: modern techniques and their applications}, volume~40.
\newblock John Wiley \& Sons, 1999.

\bibitem[Fu et~al.(2021)Fu, Yang, and Wang]{Fu2021SingleTimescaleAP}
Zuyue Fu, Zhuoran Yang, and Zhaoran Wang.
\newblock Single-timescale actor-critic provably finds globally optimal policy.
\newblock In \emph{International Conference on Learning Representations}, 2021.
\newblock URL \url{https://openreview.net/forum?id=pqZV_srUVmK}.

\bibitem[Gangwal et~al.(2024)Gangwal, Ansari, Ahmad, Azad, and Sulaiman]{gangwal2024current}
Amit Gangwal, Azim Ansari, Iqrar Ahmad, Abul~Kalam Azad, and Wan Mohd Azizi~Wan Sulaiman.
\newblock Current strategies to address data scarcity in artificial intelligence-based drug discovery: A comprehensive review.
\newblock \emph{Computers in Biology and Medicine}, 179:\penalty0 108734, 2024.

\bibitem[Gao et~al.(2023)Gao, Schulman, and Hilton]{gao2023scaling}
Leo Gao, John Schulman, and Jacob Hilton.
\newblock Scaling laws for reward model overoptimization.
\newblock In \emph{International Conference on Machine Learning}, pages 10835--10866. PMLR, 2023.

\bibitem[Gao et~al.(2024)Gao, Zha, and Zhou]{gao2024reward}
Xuefeng Gao, Jiale Zha, and Xun~Yu Zhou.
\newblock Reward-directed score-based diffusion models via q-learning.
\newblock \emph{arXiv preprint arXiv:2409.04832}, 2024.

\bibitem[Guo et~al.(2023)Guo, Li, and Xu]{guo2023fast}
Xin Guo, Xinyu Li, and Renyuan Xu.
\newblock Fast policy learning for linear quadratic regulator with entropy regularization.
\newblock \emph{arXiv preprint arXiv:2311.14168}, 2023.

\bibitem[Hambly et~al.(2021)Hambly, Xu, and Yang]{hambly2021policy}
Ben Hambly, Renyuan Xu, and Huining Yang.
\newblock Policy gradient methods for the noisy linear quadratic regulator over a finite horizon.
\newblock \emph{SIAM Journal on Control and Optimization}, 59\penalty0 (5):\penalty0 3359--3391, 2021.

\bibitem[Han et~al.(2023)Han, Razaviyayn, and Xu]{han2023policy}
Yinbin Han, Meisam Razaviyayn, and Renyuan Xu.
\newblock Policy gradient converges to the globally optimal policy for nearly linear-quadratic regulators.
\newblock \emph{arXiv preprint arXiv:2303.08431}, 2023.

\bibitem[Han et~al.(2024)Han, Razaviyayn, and Xu]{han2024neural}
Yinbin Han, Meisam Razaviyayn, and Renyuan Xu.
\newblock Neural network-based score estimation in diffusion models: Optimization and generalization.
\newblock In \emph{International Conference on Learning Representations}, 2024.
\newblock URL \url{https://openreview.net/forum?id=h8GeqOxtd4}.

\bibitem[Ho and Salimans(2021)]{ho2022classifier}
Jonathan Ho and Tim Salimans.
\newblock Classifier-free diffusion guidance.
\newblock In \emph{NeurIPS 2021 Workshop on Deep Generative Models and Downstream Applications}, 2021.
\newblock URL \url{https://openreview.net/forum?id=qw8AKxfYbI}.

\bibitem[Ho et~al.(2020)Ho, Jain, and Abbeel]{Ho20}
Jonathan Ho, Ajay Jain, and Pieter Abbeel.
\newblock Denoising diffusion probabilistic models.
\newblock In \emph{Advances in Neural Information Processing Systems}, volume~33, pages 6840--6851, 2020.

\bibitem[Hu et~al.(2022)Hu, yelong shen, Wallis, Allen-Zhu, Li, Wang, Wang, and Chen]{hu2021lora}
Edward~J Hu, yelong shen, Phillip Wallis, Zeyuan Allen-Zhu, Yuanzhi Li, Shean Wang, Lu~Wang, and Weizhu Chen.
\newblock Lo{RA}: Low-rank adaptation of large language models.
\newblock In \emph{International Conference on Learning Representations}, 2022.
\newblock URL \url{https://openreview.net/forum?id=nZeVKeeFYf9}.

\bibitem[Hyv{\"a}rinen and Dayan(2005)]{hyvarinen2005estimation}
Aapo Hyv{\"a}rinen and Peter Dayan.
\newblock Estimation of non-normalized statistical models by score matching.
\newblock \emph{Journal of Machine Learning Research}, 6\penalty0 (4), 2005.

\bibitem[Jacot et~al.(2018)Jacot, Gabriel, and Hongler]{jacot2018neural}
Arthur Jacot, Franck Gabriel, and Cl{\'e}ment Hongler.
\newblock Neural tangent kernel: Convergence and generalization in neural networks.
\newblock \emph{Advances in neural information processing systems}, 31, 2018.

\bibitem[Jin et~al.(2020)Jin, Yang, Wang, and Jordan]{jin2020provably}
Chi Jin, Zhuoran Yang, Zhaoran Wang, and Michael~I Jordan.
\newblock Provably efficient reinforcement learning with linear function approximation.
\newblock In \emph{Conference on learning theory}, pages 2137--2143. PMLR, 2020.

\bibitem[Lee et~al.(2023)Lee, Liu, Ryu, Watkins, Du, Boutilier, Abbeel, Ghavamzadeh, and Gu]{lee2023aligning}
Kimin Lee, Hao Liu, Moonkyung Ryu, Olivia Watkins, Yuqing Du, Craig Boutilier, Pieter Abbeel, Mohammad Ghavamzadeh, and Shixiang~Shane Gu.
\newblock Aligning text-to-image models using human feedback.
\newblock \emph{arXiv preprint arXiv:2302.12192}, 2023.

\bibitem[Li et~al.(2024)Li, Wei, Chen, and Chi]{li2023towards}
Gen Li, Yuting Wei, Yuxin Chen, and Yuejie Chi.
\newblock Towards non-asymptotic convergence for diffusion-based generative models.
\newblock In \emph{The Twelfth International Conference on Learning Representations}, 2024.
\newblock URL \url{https://openreview.net/forum?id=4VGEeER6W9}.

\bibitem[Liu et~al.(2019)Liu, Cai, Yang, and Wang]{Liu2019NeuralPR}
Boyi Liu, Qi~Cai, Zhuoran Yang, and Zhaoran Wang.
\newblock Neural trust region/proximal policy optimization attains globally optimal policy.
\newblock \emph{Advances in Neural Information Processing Systems}, 32, 2019.

\bibitem[Liu et~al.(2020)Liu, Zhang, Başar, and Yin]{Liu2020AnIA}
Yanli Liu, Kaiqing Zhang, Tamer Başar, and Wotao Yin.
\newblock An improved analysis of (variance-reduced) policy gradient and natural policy gradient methods.
\newblock In \emph{NeurIPS}, 2020.

\bibitem[Ma et~al.(2024)Ma, Wang, and Zhang]{ma2024convergence}
Jin Ma, Gaozhan Wang, and Jianfeng Zhang.
\newblock On convergence analysis of policy iteration algorithms for entropy-regularized stochastic control problems.
\newblock \emph{arXiv preprint arXiv:2406.10959}, 2024.

\bibitem[Malik et~al.(2019)Malik, Pananjady, Bhatia, Khamaru, Bartlett, and Wainwright]{malik2019derivative}
Dhruv Malik, Ashwin Pananjady, Kush Bhatia, Koulik Khamaru, Peter~L. Bartlett, and Martin~J. Wainwright.
\newblock Derivative-free methods for policy optimization: Guarantees for linear quadratic systems.
\newblock \emph{Journal of Machine Learning Research}, 21:\penalty0 21:1--21:51, 2019.

\bibitem[Mohammadi et~al.(2019)Mohammadi, Zare, Soltanolkotabi, and Jovanovi{\'c}]{mohammadi2019global}
Hesameddin Mohammadi, Armin Zare, Mahdi Soltanolkotabi, and Mihailo~R Jovanovi{\'c}.
\newblock Global exponential convergence of gradient methods over the nonconvex landscape of the linear quadratic regulator.
\newblock In \emph{2019 IEEE 58th Conference on Decision and Control (CDC)}, pages 7474--7479. IEEE, 2019.

\bibitem[Mondal and Aggarwal(2024)]{mondal2024improved}
Washim~U Mondal and Vaneet Aggarwal.
\newblock Improved sample complexity analysis of natural policy gradient algorithm with general parameterization for infinite horizon discounted reward markov decision processes.
\newblock In \emph{International Conference on Artificial Intelligence and Statistics}, pages 3097--3105. PMLR, 2024.

\bibitem[Mou and Zhu(2024)]{mou2024bellman}
Wenlong Mou and Yuhua Zhu.
\newblock {On Bellman equations for continuous-time policy evaluation I: discretization and approximation}.
\newblock \emph{arXiv preprint arXiv:2407.05966}, 2024.

\bibitem[OpenAI(2024)]{openai2024sora}
OpenAI.
\newblock Sora: Creating video from text.
\newblock \url{https://openai.com/sora}, 2024.

\bibitem[Podell et~al.(2023)Podell, English, Lacey, Blattmann, Dockhorn, M{\"u}ller, Penna, and Rombach]{podell2023sdxl}
Dustin Podell, Zion English, Kyle Lacey, Andreas Blattmann, Tim Dockhorn, Jonas M{\"u}ller, Joe Penna, and Robin Rombach.
\newblock {SDXL: Improving latent diffusion models for high-resolution image synthesis}.
\newblock \emph{arXiv preprint arXiv:2307.01952}, 2023.

\bibitem[Rahimi and Recht(2007)]{rahimi2007random}
Ali Rahimi and Benjamin Recht.
\newblock Random features for large-scale kernel machines.
\newblock \emph{Advances in neural information processing systems}, 20, 2007.

\bibitem[Ramesh et~al.(2022)Ramesh, Dhariwal, Nichol, Chu, and Chen]{ramesh2022hierarchical}
Aditya Ramesh, Prafulla Dhariwal, Alex Nichol, Casey Chu, and Mark Chen.
\newblock Hierarchical text-conditional image generation with clip latents.
\newblock \emph{arXiv preprint arXiv:2204.06125}, 1\penalty0 (2):\penalty0 3, 2022.

\bibitem[Rombach et~al.(2022)Rombach, Blattmann, Lorenz, Esser, and Ommer]{rombach2022high}
Robin Rombach, Andreas Blattmann, Dominik Lorenz, Patrick Esser, and Bj{\"o}rn Ommer.
\newblock High-resolution image synthesis with latent diffusion models.
\newblock In \emph{Proceedings of the IEEE/CVF conference on computer vision and pattern recognition}, pages 10684--10695, 2022.

\bibitem[Ruiz et~al.(2023)Ruiz, Li, Jampani, Pritch, Rubinstein, and Aberman]{ruiz2023dreambooth}
Nataniel Ruiz, Yuanzhen Li, Varun Jampani, Yael Pritch, Michael Rubinstein, and Kfir Aberman.
\newblock Dreambooth: Fine tuning text-to-image diffusion models for subject-driven generation.
\newblock In \emph{Proceedings of the IEEE/CVF conference on computer vision and pattern recognition}, pages 22500--22510, 2023.

\bibitem[Ryu(2023)]{ryu2023low}
Simo Ryu.
\newblock Low-rank adaptation for fast text-to-image diffusion fine-tuning.
\newblock 2023.
\newblock URL \url{https://github. com/cloneofsimo/lora}.

\bibitem[Sohl-Dickstein et~al.(2015)Sohl-Dickstein, Weiss, Maheswaranathan, and Ganguli]{sohl2015deep}
Jascha Sohl-Dickstein, Eric Weiss, Niru Maheswaranathan, and Surya Ganguli.
\newblock Deep unsupervised learning using nonequilibrium thermodynamics.
\newblock In \emph{International conference on machine learning}, pages 2256--2265. PMLR, 2015.

\bibitem[Song and Ermon(2019)]{song2019generative}
Yang Song and Stefano Ermon.
\newblock Generative modeling by estimating gradients of the data distribution.
\newblock \emph{Advances in neural information processing systems}, 32, 2019.

\bibitem[Song et~al.(2021)Song, Sohl-Dickstein, Kingma, Kumar, Ermon, and Poole]{song2020score}
Yang Song, Jascha Sohl-Dickstein, Diederik~P Kingma, Abhishek Kumar, Stefano Ermon, and Ben Poole.
\newblock Score-based generative modeling through stochastic differential equations.
\newblock In \emph{International Conference on Learning Representations}, 2021.
\newblock URL \url{https://openreview.net/forum?id=PxTIG12RRHS}.

\bibitem[Steinwart and Christmann(2008)]{steinwart2008support}
Ingo Steinwart and Andreas Christmann.
\newblock \emph{Support vector machines}.
\newblock Springer Science \& Business Media, 2008.

\bibitem[Szpruch et~al.(2021)Szpruch, Treetanthiploet, and Zhang]{szpruch2021exploration}
Lukasz Szpruch, Tanut Treetanthiploet, and Yufei Zhang.
\newblock Exploration-exploitation trade-off for continuous-time episodic reinforcement learning with linear-convex models.
\newblock \emph{arXiv preprint arXiv:2112.10264}, 2021.

\bibitem[Szpruch et~al.(2024)Szpruch, Treetanthiploet, and Zhang]{szpruch2022optimal}
Lukasz Szpruch, Tanut Treetanthiploet, and Yufei Zhang.
\newblock Optimal scheduling of entropy regularizer for continuous-time linear-quadratic reinforcement learning.
\newblock \emph{SIAM Journal on Control and Optimization}, 62\penalty0 (1):\penalty0 135--166, 2024.

\bibitem[Tang(2024)]{tang2024fine}
Wenpin Tang.
\newblock Fine-tuning of diffusion models via stochastic control: entropy regularization and beyond.
\newblock \emph{arXiv preprint arXiv:2403.06279}, 2024.

\bibitem[Uehara et~al.(2024{\natexlab{a}})Uehara, Zhao, Biancalani, and Levine]{uehara2024understanding}
Masatoshi Uehara, Yulai Zhao, Tommaso Biancalani, and Sergey Levine.
\newblock Understanding reinforcement learning-based fine-tuning of diffusion models: A tutorial and review.
\newblock \emph{arXiv preprint arXiv:2407.13734}, 2024{\natexlab{a}}.

\bibitem[Uehara et~al.(2024{\natexlab{b}})Uehara, Zhao, Black, Hajiramezanali, Scalia, Diamant, Tseng, Levine, and Biancalani]{uehara2024feedback}
Masatoshi Uehara, Yulai Zhao, Kevin Black, Ehsan Hajiramezanali, Gabriele Scalia, Nathaniel~Lee Diamant, Alex~M Tseng, Sergey Levine, and Tommaso Biancalani.
\newblock Feedback efficient online fine-tuning of diffusion models.
\newblock In \emph{Forty-first International Conference on Machine Learning}, 2024{\natexlab{b}}.
\newblock URL \url{https://openreview.net/forum?id=dtVlc9ybTm}.

\bibitem[Uehara et~al.(2024{\natexlab{c}})Uehara, Zhao, Hajiramezanali, Scalia, Eraslan, Lal, Levine, and Biancalani]{uehara2024bridging}
Masatoshi Uehara, Yulai Zhao, Ehsan Hajiramezanali, Gabriele Scalia, G{\"o}kcen Eraslan, Avantika Lal, Sergey Levine, and Tommaso Biancalani.
\newblock Bridging model-based optimization and generative modeling via conservative fine-tuning of diffusion models.
\newblock \emph{arXiv preprint arXiv:2405.19673}, 2024{\natexlab{c}}.

\bibitem[Wallace et~al.(2024)Wallace, Dang, Rafailov, Zhou, Lou, Purushwalkam, Ermon, Xiong, Joty, and Naik]{wallace2024diffusion}
Bram Wallace, Meihua Dang, Rafael Rafailov, Linqi Zhou, Aaron Lou, Senthil Purushwalkam, Stefano Ermon, Caiming Xiong, Shafiq Joty, and Nikhil Naik.
\newblock Diffusion model alignment using direct preference optimization.
\newblock In \emph{Proceedings of the IEEE/CVF Conference on Computer Vision and Pattern Recognition}, pages 8228--8238, 2024.

\bibitem[Wang et~al.(2020{\natexlab{a}})Wang, Zariphopoulou, and Zhou]{wang2020reinforcement}
Haoran Wang, Thaleia Zariphopoulou, and Xunyu Zhou.
\newblock Reinforcement learning in continuous time and space: A stochastic control approach.
\newblock \emph{The Journal of Machine Learning Research}, 21\penalty0 (1):\penalty0 8145--8178, 2020{\natexlab{a}}.

\bibitem[Wang et~al.(2020{\natexlab{b}})Wang, Cai, Yang, and Wang]{Wang2020NeuralPG}
Lingxiao Wang, Qi~Cai, Zhuoran Yang, and Zhaoran Wang.
\newblock Neural policy gradient methods: Global optimality and rates of convergence.
\newblock In \emph{International Conference on Learning Representations}, 2020{\natexlab{b}}.
\newblock URL \url{https://openreview.net/forum?id=BJgQfkSYDS}.

\bibitem[Watson et~al.(2023)Watson, Juergens, Bennett, Trippe, Yim, Eisenach, Ahern, Borst, Ragotte, Milles, et~al.]{watson2023novo}
Joseph~L Watson, David Juergens, Nathaniel~R Bennett, Brian~L Trippe, Jason Yim, Helen~E Eisenach, Woody Ahern, Andrew~J Borst, Robert~J Ragotte, Lukas~F Milles, et~al.
\newblock De novo design of protein structure and function with rfdiffusion.
\newblock \emph{Nature}, 620\penalty0 (7976):\penalty0 1089--1100, 2023.

\bibitem[Xiao(2022)]{xiao2022convergence}
Lin Xiao.
\newblock On the convergence rates of policy gradient methods.
\newblock \emph{Journal of Machine Learning Research}, 23\penalty0 (282):\penalty0 1--36, 2022.

\bibitem[Xu et~al.(2023)Xu, Liu, Wu, Tong, Li, Ding, Tang, and Dong]{xu2023imagereward}
Jiazheng Xu, Xiao Liu, Yuchen Wu, Yuxuan Tong, Qinkai Li, Ming Ding, Jie Tang, and Yuxiao Dong.
\newblock Imagereward: Learning and evaluating human preferences for text-to-image generation.
\newblock \emph{Advances in Neural Information Processing Systems}, 36:\penalty0 15903--15935, 2023.

\bibitem[Xu et~al.(2021)Xu, Yang, Wang, and Liang]{Xu2021DoublyRO}
Tengyu Xu, Zhuoran Yang, Zhaoran Wang, and Yingbin Liang.
\newblock Doubly robust off-policy actor-critic: Convergence and optimality.
\newblock In \emph{ICML}, 2021.

\bibitem[Zhan et~al.(2023)Zhan, Cen, Huang, Chen, Lee, and Chi]{zhan2023policy}
Wenhao Zhan, Shicong Cen, Baihe Huang, Yuxin Chen, Jason~D Lee, and Yuejie Chi.
\newblock Policy mirror descent for regularized reinforcement learning: A generalized framework with linear convergence.
\newblock \emph{SIAM Journal on Optimization}, 33\penalty0 (2):\penalty0 1061--1091, 2023.

\bibitem[Zhang(2001)]{zhang2001some}
Jianfeng Zhang.
\newblock \emph{Some fine properties of backward stochastic differential equations}.
\newblock Purdue University, 2001.

\bibitem[Zhang(2017)]{zhang2017backward}
Jianfeng Zhang.
\newblock \emph{Backward stochastic differential equations}.
\newblock Springer, 2017.

\bibitem[Zhang et~al.(2020)Zhang, Koppel, Bedi, Szepesvari, and Wang]{zhang2020variational}
Junyu Zhang, Alec Koppel, Amrit~Singh Bedi, Csaba Szepesvari, and Mengdi Wang.
\newblock Variational policy gradient method for reinforcement learning with general utilities.
\newblock \emph{Advances in Neural Information Processing Systems}, 33:\penalty0 4572--4583, 2020.

\bibitem[Zhang et~al.(2021)Zhang, Yang, and Wang]{zhang2021provably}
Yufeng Zhang, Zhuoran Yang, and Zhaoran Wang.
\newblock Provably efficient actor-critic for risk-sensitive and robust adversarial {RL}: A linear-quadratic case.
\newblock In \emph{International Conference on Artificial Intelligence and Statistics}, pages 2764--2772. PMLR, 2021.

\bibitem[Zhao et~al.(2024{\natexlab{a}})Zhao, Chen, Zhang, Yao, and Tang]{zhao2024scores}
Hanyang Zhao, Haoxian Chen, Ji~Zhang, David~D Yao, and Wenpin Tang.
\newblock Scores as actions: a framework of fine-tuning diffusion models by continuous-time reinforcement learning.
\newblock \emph{arXiv preprint arXiv:2409.08400}, 2024{\natexlab{a}}.

\bibitem[Zhao et~al.(2024{\natexlab{b}})Zhao, Uehara, Scalia, Biancalani, Levine, and Hajiramezanali]{zhao2024adding}
Yulai Zhao, Masatoshi Uehara, Gabriele Scalia, Tommaso Biancalani, Sergey Levine, and Ehsan Hajiramezanali.
\newblock Adding conditional control to diffusion models with reinforcement learning.
\newblock \emph{arXiv preprint arXiv:2406.12120}, 2024{\natexlab{b}}.

\bibitem[Zhou and Lu(2023{\natexlab{a}})]{zhou2023policy}
Mo~Zhou and Jianfeng Lu.
\newblock A policy gradient framework for stochastic optimal control problems with global convergence guarantee.
\newblock \emph{arXiv preprint arXiv:2302.05816}, 2023{\natexlab{a}}.

\bibitem[Zhou and Lu(2023{\natexlab{b}})]{zhou2023single}
Mo~Zhou and Jianfeng Lu.
\newblock Single timescale actor-critic method to solve the linear quadratic regulator with convergence guarantees.
\newblock \emph{Journal of Machine Learning Research}, 24\penalty0 (222):\penalty0 1--34, 2023{\natexlab{b}}.

\bibitem[Zhou and Lu(2024)]{zhou2024solving}
Mo~Zhou and Jianfeng Lu.
\newblock Solving time-continuous stochastic optimal control problems: Algorithm design and convergence analysis of actor-critic flow.
\newblock \emph{arXiv preprint arXiv:2402.17208}, 2024.

\end{thebibliography}
\bibliographystyle{plainnat}

\end{document}